\documentclass{article}
\PassOptionsToPackage{numbers,compress}{natbib}

\usepackage[main, final]{neurips_2026}

\usepackage[utf8]{inputenc} 
\usepackage[T1]{fontenc}    
\usepackage{hyperref}       
\usepackage{url}            
\usepackage{booktabs}       
\usepackage{amsfonts}       
\usepackage{nicefrac}       
\usepackage{microtype}      
\usepackage{xcolor}         
\usepackage{graphicx}
\usepackage{multirow}
\usepackage{threeparttable}
\usepackage{amsmath}
\usepackage{amsthm}
\usepackage{amssymb}

\newtheorem{proposition}{Proposition}

\theoremstyle{definition}

\theoremstyle{remark}

\usepackage{algorithm}
\usepackage{algpseudocode}

\usepackage{bbm}

\usepackage{tabularx}
\usepackage{array}

\usepackage{float}
\newcolumntype{Y}{>{\raggedright\arraybackslash}X}
\title{PRISM: Prior Relational Information for Self-supervised Modeling to Enhance Solubility OOD Generalization}

\author{%
  Xinyi Chen\textsuperscript{1} \\
  {\small\texttt{chenxinyi@sjtu.edu.cn}}
  \And
  Jiahuan Pang\textsuperscript{1} \\
  {\small\texttt{pangjiahuan@sjtu.edu.cn}}
  \And
  Titus Chima\textsuperscript{1} \\
  {\small\texttt{titusschima@sjtu.edu.cn}}
  \AND
  Paul Weng\textsuperscript{3,4,*} \\
  {\small\texttt{paul.weng@dukekunshan.edu.cn}}
  \And
  Wendong Wang\textsuperscript{1,2,*} \\
  {\small\texttt{wendong.wang@sjtu.edu.cn}}
}

\begin{document}

\maketitle

\begingroup
\renewcommand{\thefootnote}{}
\footnotetext{%
  \textsuperscript{1}Global College, Shanghai Jiao Tong University, Shanghai, China.
  \textsuperscript{2}National Key Laboratory of Advanced Micro and Nano Manufacture Technology, Shanghai Jiao Tong University, Shanghai, China.
  \textsuperscript{3}Digital Innovation Research Center, Duke Kunshan University, Kunshan, Jiangsu, China.
  \textsuperscript{4}Zu Chongzhi Center, Duke Kunshan University, Kunshan, Jiangsu, China.
  
  \textsuperscript{*}Corresponding authors: \texttt{paul.weng@dukekunshan.edu.cn} and \texttt{wendong.wang@sjtu.edu.cn}.
}
\endgroup
\begin{abstract}
Solubility depends on intermolecular interactions between solutes and solvents, and is central to drug discovery, chemical and materials process engineering.
Its prediction requires out-of-distribution (OOD) generalization to unseen solutes, solvents, and solvent mixtures, where supervised methods suffer from data scarcity, and existing {self-supervised methods} do not explicitly learn chemically meaningful relations.
Learning chemically meaningful relational information without labels can alleviate data scarcity and improve OOD generalization, but remains challenging.
We introduce Prior Relational Information for Self-supervised Modeling (PRISM) and apply it to solubility OOD generalization.
PRISM is a label-free pretraining framework that converts expert-defined chemical priors into relational teachers.
For solubility prediction, it trains four encoders to preserve the population-level relational information from priors capturing hydrophobic, electrostatic, polarizability, and substructural effects, using Kullback-Leibler (KL) divergence to match {distance-based neighborhood distributions} during pretraining.
{Across scaffold, solute, and solvent OOD settings}, PRISM matches or surpasses supervised and self-supervised baselines.
{Under mixed-solvent OOD}, freezing the four PRISM encoders and training a {lightweight} fusion-and-composition head outperforms {strong baselines} and {substantially reduces run-to-run variation}.
Ablations and control studies confirm that gains stem from chemically meaningful relational structures baked into the encoders by PRISM.
We expect this strategy of turning relational priors into self-supervision signals to extend beyond solubility prediction to domains where chemically or biologically meaningful relational priors are available.
\end{abstract}

\section{Introduction}
\label{sec:introduction}

Solubility measures how much of a solute dissolves in a specified pure or mixed solvent under defined temperature and pressure, and it is a critical property across drug discovery, formulation, and pharmaceutical process development~\cite{DaVia2024GSKSolubility}.
In practical applications, models often encounter unseen scaffolds, unseen solutes and solvents, and unseen binary solvent pairs in mixed-solvent systems, requiring solubility prediction beyond the training distribution.
These shifts directly affect the solute--solvent and solvent--solvent interactions governing solubility~\citep{yeyunchen2025relational,deng_systematic_2023}, making them natural out-of-distribution (OOD) scenarios for evaluating chemically meaningful molecular representations.

Recent solubility prediction methods model solvent context through recommendation, interaction-graph, hypergraph, or multicomponent architectures~\citep{wollschlager2024solvmate,ramani2024molmerger,du2025mmhnn,jung2025multicomponent}.
Broader molecular relational learning methods explicitly model drug--drug or solute--solvent interactions~\citep{lee2023cgib,lee2023cmrl,yeyunchen2025relational}, but usually require {supervised pairs} or task-specific data.
{Molecular self-supervised objectives based on augmentation-defined pairs, masking, or reconstruction~\citep{molclr,hu2023your,SMICLR,Chemformer,vincent2011connection}
do not specify scientifically grounded population-level neighborhoods: {which molecules should be close or far under a chemically meaningful relation.} Classical manifold and molecular metric-learning methods can preserve predefined relations, but leave open how to transfer scientifically grounded population-level neighborhood geometry into molecular encoders without labels~\citep{ceriotti2011sketchmap,coupry2022metric}.}
These studies leave a research gap: label-free pretraining remains weakly aligned with the relational structure needed for solubility OOD generalization.

For solubility, a chemically grounded route can close this gap: by the empirical ``like dissolves like'' rule, solute--solvent compatibility is {reflected} by descriptor- or fingerprint-based physicochemical similarity~\citep{todeschini2009molecular,rogers2010extended,vsa}.
We propose \textbf{PRISM} (Prior Relational Information for Self-supervised Modeling), which turns {such similarity} into a label-free relational teacher.
PRISM {uses these prior-defined relations to shape the geometry of learned molecular embeddings}.

{PRISM has two design choices: the \emph{prior source}, which specifies the expert-defined relation, and the \emph{alignment design}, which transfers that relation into embedding space.}
We use van der Waals surface area (VSA) descriptors~\citep{vsa} and Morgan fingerprints~\citep{rogers2010extended} as complementary physicochemical and substructural priors, and adopt {distance-based} distribution matching with prior-specific encoders fused only downstream (late fusion) as the alignment design among several candidates.

{Our main results are OOD generalization of solubility prediction.
PRISM outperforms supervised and SSL-pretrained baselines in cross-dataset generalization from BigSolDB to Leeds and SolProp for predicting the solubility of diverse solutes in pure solvents.} With frozen prior-specific encoders, trainable fusion, and a composition-weighted head, it matches or improves end-to-end mixed-solvent baselines under solvent-pair-disjoint splits with better stability.
Label energy (a form of graph Dirichlet energy) diagnostics and ablation controls further show that these gains come from prior--task compatibility, relational pretraining, and late fusion of prior-specific embeddings, rather than from descriptor inputs, reconstruction, architectural capacity, or coordinate-level imitation.

Our contributions are the following:
\begin{enumerate}
    \item We propose \textbf{PRISM}, a label-free molecular pretraining framework that {bakes chemically meaningful prior knowledge into molecular representations through prior-induced relational modeling}, providing a practical route to label-free relational representation learning.
    
    \item We build a PRISM-based solubility prediction model using VSA descriptors and Morgan fingerprints as complementary physicochemical and substructural priors. The model includes two stages: the pretraining of four prior-specific encoders and the weighted-sum late fusion of prior-specific embeddings during downstream tasks.
    
    \item We show that PRISM improves solubility-centered OOD generalization from scaffold splits to pure- and mixed-solvent settings, {matching or surpassing strong baselines and improving stability under mixed-solvent OOD. Ablations and label-energy diagnostics support relational pretraining, prior complementarity, and prior--task compatibility.}
\end{enumerate}

\section{Methodology}
\label{sec:prism}
PRISM treats each chemical prior as a \emph{relational teacher}, defining molecular closeness under a specific physicochemical criterion.
{We adopt D-MPNN as a strong 2D graph backbone~\cite{yang2019dmpnn}; the framework itself remains encoder-agnostic.} We first discuss \emph{which} prior relations to choose in Section \ref{sec:prior_spaces}, and then describe \emph{how} they are aligned in the embedding space and fused in Section \ref{sec:prism_family} and \ref{sec:prism_full}.

\subsection{Solubility-grounded prior spaces}
\label{sec:prior_spaces}

The chemistry of solubility guides which prior relations PRISM uses.
Solubility depends on solute--solvent interactions and molecular structure.
Following the empirical ``like dissolves like'' principle, solubility is favored when solute and solvent share similar interaction components, including hydrophobic, electrostatic, and dispersion-related effects, or similar substructures capturing interactions not accounted for by those effects.
PRISM therefore uses four complementary priors: SlogP\_VSA for hydrophobic/lipophilic effects, SMR\_VSA for polarizability/dispersion-relevant effects, PEOE\_VSA for electrostatic effects, and Morgan fingerprints for substructure and functional-group patterns.
Detailed descriptions of these prior spaces and other prior spaces are provided in Appendix~\ref{app:priors}.

\subsection{Prior-induced relational alignment}
\label{sec:prism_family}

Let $\mathcal{D}_{\mathrm{pre}}=\{x_i\}_{i=1}^{N}$ be an unlabeled set of molecular graphs and let $\mathcal{S}=\{s_k\}_{k=1}^{K}$ be a set of prior spaces.
For each prior space $s_k\in\mathcal{S}$, a fixed prior function maps molecule $x_i$ to $\pi_i^{s_k}$, and a prior-space comparator $d^{s_k}$ measures pairwise relations.
A prior-specific encoder $G^{s_k}_{\theta_k}$ with parameters $\theta_k$ maps the same molecule to an embedding $r_i^{s_k}$.
PRISM aligns the relation induced by $\{r_i^{s_k}\}_{i=1}^{N}$ with the relation induced by $\{\pi_i^{s_k}\}_{i=1}^{N}$ in a minibatch $\mathcal{B}$. The alignment procedure is illustrated in Figure~\ref{fig:prism_overview}a.

\begin{figure*}[h]
  \centering
  \includegraphics[width=0.95\textwidth]{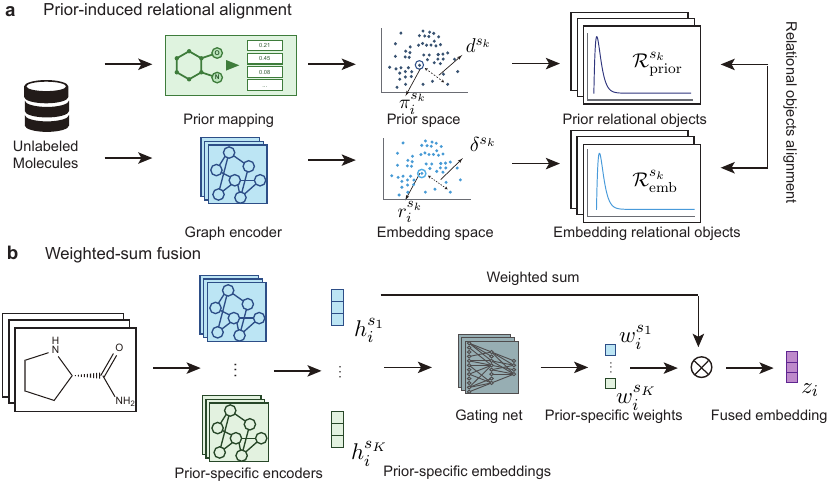}
    \caption{
    \textbf{PRISM overview.}
    \textbf{(a)} Prior-guided alignment pretrains an encoder per prior space.
    \textbf{(b)} Weighted fusion produces a fused molecular embedding.
    }
  \label{fig:prism_overview}
\end{figure*}

Concretely, two relational constructors $\mathcal{A}_{prior}^{s_k}$ and $\mathcal{A}_{emb}^{s_k}$ map pairwise comparisons in $\mathcal{B}$ to two relational objects $\mathcal{R}_{\mathrm{prior}}^{s_k}$ and $\mathcal{R}_{\mathrm{emb}}^{s_k}$:
\begin{equation}
    \mathcal{R}_{\mathrm{prior}}^{s_k}
    =
    \mathcal{A}_{prior}^{s_k}
    \!\left(\{d^{s_k}(\pi_i^{s_k},\pi_j^{s_k})\}_{i,j\in\mathcal{B}}\right),
    \qquad
    \mathcal{R}_{\mathrm{emb}}^{s_k}
    =
    \mathcal{A}_{emb}^{s_k}
    \!\left(\{\delta^{s_k}(r_i^{s_k},r_j^{s_k})\}_{i,j\in\mathcal{B}}\right).
\label{eq:prism_relational_objects}
\end{equation}
The objective of PRISM is to align the two relational objects, $\mathcal{R}_{\mathrm{prior}}^{s_k}$ and $\mathcal{R}_{\mathrm{emb}}^{s_k}$, using a suitable discrepancy metric as the loss function: 
\begin{equation}
    \mathcal{L}_{\mathrm{PRISM}}^{s_k}
    =
    \mathcal{M}^{s_k}
    \!\left(
    \mathcal{R}_{\mathrm{prior}}^{s_k},
    \mathcal{R}_{\mathrm{emb}}^{s_k}
    \right).
\label{eq:prism_objective}
\end{equation}

\begin{table}[h]
\centering
\small
\caption{
\textbf{Relational objects and alignment objectives in PRISM with related instances.}
}
\label{tab:alignment_family}
\setlength{\tabcolsep}{4pt}
\renewcommand{\arraystretch}{1.12}
\resizebox{0.99\textwidth}{!}{
\begin{tabular}{llll}
\toprule
Relational object 
& Objective 
& Discrepancy metric 
& Related instances \\
\midrule
Pairwise distance matrix
& Pointwise metric matching
& Distance regression
& Talibart et al.~\citep{talibart2026chemistry} \\

Neighbor ordering
& Rank preservation
& Rank / ListMLE loss
& PH-TDL~\cite{luo2023improving} \\

Top-$n$ prior-space kNN graph
& Neighborhood-graph alignment
& InfoNCE over graph positives
& Moleco~\cite{park-etal-2024-moleco} \\

Pair-wise distance-based distribution
& Distribution matching
& KL divergence
& SNE-style \citep{hinton2002sne} / GraphMSL~\cite{xu2024graphmultisimilaritylearningmolecular} \\
\bottomrule
\end{tabular}
}
\vspace{-2mm}
\end{table}

Table~\ref{tab:alignment_family} classifies alignment objectives by the relational objects they preserve. As a concrete example, pairwise distances can be converted into anchor-wise stochastic neighborhood distributions.
For each prior space \(s_k\), anchor molecule \(x_i\), and candidate set
\(\mathcal{B}_i=\mathcal{B}\setminus\{x_i\}\),
\(R_i^{s_k}\) denotes the prior-space probability distribution of  \(x_i\), while \(P_i^{s_k}\) denotes the corresponding embedding-space probability distribution. The relational objects are sets of anchor-wise distributions
\[
\mathcal{R}_{\mathrm{prior}}^{s_k}
=
\{R_i^{s_k}\}_{x_i\in\mathcal{B}},
\qquad
\mathcal{R}_{\mathrm{emb}}^{s_k}
=
\{P_i^{s_k}\}_{x_i\in\mathcal{B}}.
\]
The probabilities in these two distributions are defined by
\begin{equation}
\resizebox{0.99\textwidth}{!}{$
\displaystyle
R_i^{s_k}(x_j)
=
\frac{\exp\!\left(-\beta_k d^{s_k}(\pi_i^{s_k},\pi_j^{s_k})\right)}
{\sum_{x_l\in\mathcal{B}_i}\exp\!\left(-\beta_k d^{s_k}(\pi_i^{s_k},\pi_l^{s_k})\right)},
\qquad
P_i^{s_k}(x_j)
=
\frac{\exp\!\left(-\beta_\tau\delta(r_i^{s_k},r_j^{s_k})\right)}
{\sum_{x_l\in\mathcal{B}_i}\exp\!\left(-\beta_\tau\delta(r_i^{s_k},r_l^{s_k})\right)} .
$}
\label{eq:soft_neighborhoods}
\end{equation}
Here \(d^{s_k}\) is the prior-space distance,
\(\delta(r_i^{s_k},r_j^{s_k})=1-\cos(r_i^{s_k},r_j^{s_k})\), and \(\beta_k,\beta_\tau\) control locality.
The loss function aggregates anchor-wise KL divergence over the minibatch $\mathcal{B}$:
\begin{equation}
\mathcal{L}_{\mathrm{KL}}^{s_k}
=
\frac{1}{|\mathcal{B}|}
\sum_{x_i\in\mathcal{B}}
\mathrm{KL}\!\left(R_i^{s_k}\,\middle\|\,P_i^{s_k}\right).
\label{eq:align_loss}
\end{equation}
Because Eq.~\ref{eq:soft_neighborhoods} normalizes each anchor separately, these neighborhood distributions are conditional and generally asymmetric, retaining anchor-specific {neighborhood geometry}.
Section~\ref{sec:alignment_selection} evaluates the four alignment objectives and explains why distribution matching is used in PRISM.
Appendix~\ref{app:loss} provides additional theoretical analysis on these alignment objectives {and compares the conditional formulation with a symmetric SNE-style alternative}.

\subsection{Weighted-sum fusion}
\label{sec:prism_full}

The second stage of PRISM is weighted-sum fusion (Figure~\ref{fig:prism_overview}b). 
PRISM uses $K$ prior-specific encoders, each producing
$h_i^{(s_k)}=q^{s_k}(r_i^{s_k})\in\mathbb{R}^{D}$, with branch embeddings stacked as
$H_i=[h_i^{(s_1)};\ldots;h_i^{(s_K)}]\in\mathbb{R}^{K\times D}$.
PRISM fuses these branch embeddings by weighted summation:
\begin{equation}
    z_i=\sum_{s_k\in\mathcal{S}}w_i^{s_k}h_i^{s_k}.
\label{eq:fusion}
\end{equation}
We consider two fusion variants.
The first uses task-specific weights, $w_i^{s_k}=w^{s_k}$, providing a task-level summary of prior importance.
The second uses molecule-specific weights $w_i=\mathrm{softmax}(g(H_i))$ from a small gating network $g$, allowing PRISM to emphasize different prior-induced relations across molecules as solubility-governing interactions vary.
Implementation details and fusion analyses are provided in Appendix~\ref{app:impl:fusion} and Appendix~\ref{app:interp:faithfulness}.

\label{sec:task_heads}

\begin{figure*}[h]
  \centering
  \includegraphics[width=0.95\textwidth]{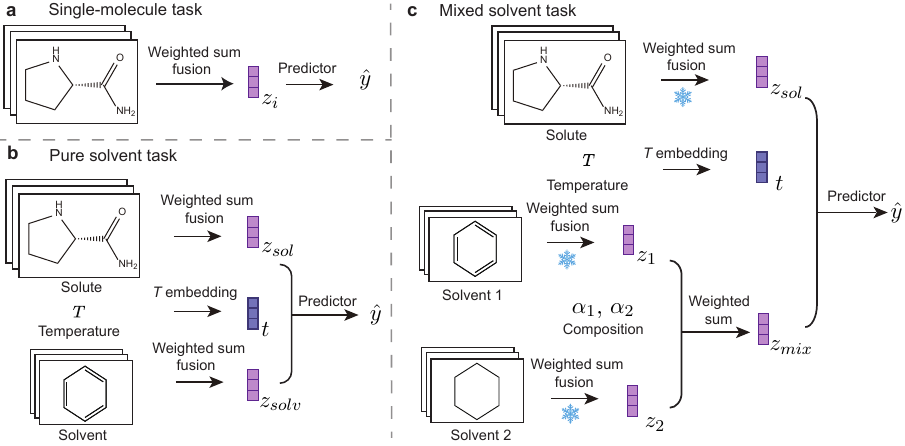}
  \caption{
  \textbf{Task-specific prediction heads.}
    \textbf{(a)} Single-molecule prediction.
    \textbf{(b)} Pure-solvent prediction combines solute, solvent, and temperature embeddings.
    \textbf{(c)} Mixed-solvent prediction freezes encoders, trains fusion, and compositionally pools solvent embeddings.
  }
  \label{fig:task_heads}
\end{figure*}

\subsection{Task-specific prediction heads}
Figure~\ref{fig:task_heads} summarizes the downstream task settings.
For single-molecule tasks, the model predicts directly from the fused molecular embedding,
$\hat y_i=f_{\mathrm{single}}(z_i)$.
For pure-solvent tasks, the head combines solute embedding $z_{\mathrm{sol}}$, solvent embedding
$z_{\mathrm{solv}}$, and temperature embedding $t=h_T(T)$, where $h_T$ is an MLP, as
$\hat y=f_{\mathrm{pure}}([z_{\mathrm{sol}};z_{\mathrm{solv}};t])$.

For mixed-solvent prediction, we freeze the original PRISM-pretrained encoders, initialize the fusion module from the pure-solvent task, and train it together with a lightweight mixed-solvent head $f_{\mathrm{mix}}$.
Given two solvents $x_1,x_2$ and their mole fractions $\alpha_1,\alpha_2$, where $\alpha_1+\alpha_2=1$, we form
\begin{equation}
    z_{\mathrm{mix}}=\alpha_1z_1+\alpha_2z_2,
    \qquad
    \hat y=f_{\mathrm{mix}}([z_{\mathrm{sol}};z_{\mathrm{mix}};t]).
\label{eq:mix_head_simple}
\end{equation}
Because a real solvent mixture may have nonlinear solvent--solvent interactions, the linear pooling restricts head capacity.
This deliberate low-capacity design probes whether PRISM-pretrained solvent embeddings contain relational information about solvent--solvent interactions.

\subsection{Label energy for prior--task compatibility}
\label{sec:label_energy}

PRISM relies on prior spaces that induce task-relevant molecular neighborhoods.
We therefore introduce \textbf{label energy} as a practical diagnostic of prior--task compatibility: if a prior is relevant to a downstream task, {molecules with similar labels should be nearby in its prior space}.
Using $n$ nearest neighbors for each molecule (we use $n=10$ in all experiments), we define
\begin{equation}
\mathcal{E}^{s_k}_{n}(y)
=
\frac{1}{|Q|}
\sum_{i\in Q}
\frac{1}{n}
\sum_{j\in\mathcal{N}^{s_k}_{n}(x_i)}
(y_i-y_j)^2 .
\label{eq:label_energy_main}
\end{equation}
Here $Q$ denotes the labeled molecules available for diagnosis, and
$\mathcal{N}^{s_k}_{n}(x_i)$ is the set of $n$ nearest neighbors of $x_i$ in prior space $s_k$.
Lower energy indicates {molecules with} smoother labels over the prior-induced neighborhood and therefore stronger prior--task compatibility.

Because the raw energy depends on label scale and class balance, we normalize it against a Gaussian random prior:
\begin{equation}
\rho^{s_k}
=
\frac{\mathcal{E}^{s_k}_n(y)}
{\mathcal{E}^{(\mathrm{random})}_n(y)}.
\label{eq:label_energy_ratio}
\end{equation}
Thus, $\rho^{s_k}<1$ indicates that prior $s_k$ induces smoother label neighborhoods than random geometry.
{Appendix~F.3 further shows that label energy bounds the error of a local predictor in the prior space, providing a theoretical justification for this diagnostic.}
When prior choice is less obvious, candidate priors should be scientifically relevant, inexpensive, and complementary, with label energy on training/validation labels used to filter task-incompatible choices (Appendix~\ref{app:extra_ablation:dynamic_selection}).

\section{Experiments}
\label{sec:experiments}

\subsection{Experimental setup}
\label{sec:setup}

In this work, we evaluate PRISM under four chemically significant OOD scenarios:
\begin{enumerate}
    \item \textbf{Structural OOD}: MoleculeNet scaffold splits test unseen molecular scaffolds.
    \item \textbf{Solute OOD}: pure-solvent transfer tests unseen solutes on Leeds and SolProp.
    \item \textbf{Solvent OOD}: {pure-solvent transfer evaluates the subset of SolProp whose solvents are absent from the processed BigSolDB reference set.}
    \item \textbf{Solvent-pair OOD}: mixed-solvent transfer tests unseen solvent pairs from reference \citep{bao2024binarysolubility}.
\end{enumerate}

We report RMSE for regression tasks.
For pure and mixed solubility tasks, we also report $\%|\Delta \log (S)|\leq 1 = N^{-1}\sum_i \mathbf{1}(|\widehat{\log (S_i)}-\log (S_i)|\leq 1)\times 100$, the fraction of predictions within one log-solubility unit, a practically meaningful error range~\citep{attia_data-driven_2025} that complements RMSE. Split construction and leakage checks are detailed in Appendix~\ref{app:datasets:leakage}.
{
Because the evaluated tasks require single-molecule, solute--solvent, and multi-molecule inputs, we use capability- and protocol-matched baselines for each setting; Appendix~\ref{app:baseline_protocols} summarizes their selection and training protocols.
}

\subsection{Selecting the default alignment objective for PRISM}
\label{sec:alignment_selection}

\begin{table}[h]
\centering
\small
\setlength{\tabcolsep}{5pt}
\renewcommand{\arraystretch}{1.15}
\caption{
\textbf{Alignment objective selection on ESOL.}
{
C1--C3 denote fixed early, intermediate, and late {checkpoints} from one uninterrupted pretraining run for each objective.
}}
\label{tab:loss_ablation}
\begin{tabular}{lcccccc}
\toprule
\multirow{2}{*}{Objective} 
& \multicolumn{3}{c}{Alignment metric$\downarrow$} 
& \multicolumn{3}{c}{ESOL RMSE $\downarrow$} \\
\cmidrule(lr){2-4} \cmidrule(lr){5-7}
& C1 & C2 & C3 & C1 & C2 & C3 \\
\midrule
$\ell_2$ distance
& 0.968 & 0.933 & \textbf{0.851}
& 0.956 & \textbf{0.865} & 0.987 \\

Rank
& 0.180 & 0.109 & \textbf{0.025}
& 1.023 & 1.027 & \textbf{1.022} \\

Neighborhood-graph alignment 
& 2.53 & 2.02 & \textbf{1.34}
& 2.011 & 1.241 & \textbf{1.083} \\

Distribution matching
& 0.813 & 0.736 & \textbf{0.723}
& 0.844 & 0.776 & \textbf{0.772} \\
\bottomrule
\end{tabular}
\vspace{-2mm}
\end{table}

{
Table~\ref{tab:alignment_family} defines four types of relational alignment objectives, and Table~\ref{tab:loss_ablation} compares one representative from each type on ESOL under the same SlogP\_VSA prior, D-MPNN backbone, pretraining corpus, and downstream protocol.
C1--C3 trace how each objective evolves during pretraining; final checkpoints are selected using held-out unlabeled alignment loss.
}

The objectives exhibit different behaviors under this matched comparison.
$\ell_2$ {is overly rigid}: alignment improves across C1--C3, but downstream RMSE first decreases and then increases, suggesting that exact distance matching can suppress graph-derived information not represented by the prior.
Rank {is too coarse: it preserves relative ordering while discarding graded similarity strength~\citep{van2012stochastic}, and its downstream RMSE changes little across checkpoints.}
Neighborhood-graph alignment and distribution matching both preserve neighborhood structure, while distribution matching additionally retains graded similarity rather than reducing relations to hard positive/negative assignments, and achieves the best late-checkpoint RMSE.
{
We therefore select distribution matching as PRISM's default alignment objective and keep this objective fixed in all subsequent experiments.
}
\subsection{MoleculeNet structural OOD}
\label{sec:molnet_results}

\begin{table*}[h]
\centering
\small
\setlength{\tabcolsep}{5pt}
\caption{
\textbf{MoleculeNet structural OOD on solubility-related tasks under scaffold splits.}
RMSE ($\downarrow$). Results averaged over three seeds with independent splits.
}
\label{tab:molnet_core}
\begin{tabular}{lccc}
\toprule
Model 
& ESOL $\downarrow$ 
& FreeSolv $\downarrow$ 
& Lipo $\downarrow$ \\
\midrule
XGBoost
& $0.925\pm0.00$ 
& $3.398\pm0.00$ 
& $0.741\pm0.00$ \\
D-MPNN\citep{yang2019dmpnn}
& $1.050\pm0.01$ 
& $\mathbf{1.673\pm0.08}$ 
& $0.683\pm0.02$ \\
MolCLR\citep{molclr}
& $1.113\pm0.02$ 
& $2.301\pm0.25$ 
& $0.789\pm0.01$ \\
MoAMa\citep{MoAMa}
& $1.394\pm0.12$ 
& $2.700\pm0.11$ 
& $0.772\pm0.03$ \\
\midrule
\textbf{PRISM}
& $\mathbf{0.772\pm0.002}$ 
& $1.770\pm0.09$ 
& $\mathbf{0.538\pm0.02}$ \\
\bottomrule
\end{tabular}
\vspace{-2mm}
\end{table*}

In addition to ESOL, we evaluate PRISM on FreeSolv and Lipo to assess structural OOD for additional solubility-related properties. PRISM achieves the best results on ESOL and Lipo. FreeSolv is a boundary case: its hydration free energy depends on 3D geometry and electrostatic structure that 2D encoders cannot fully capture. Replacing the 2D backbone with a 3D-aware encoder substantially {improves performance on FreeSolv} (Appendix~\ref{app:extra_ablation:schnet}), {supporting the value of 3D-aware representations for this task}. The complete results on the remaining MoleculeNet benchmarks (BBBP, BACE, HIV) are in the Appendix~\ref{app:extra_results:molnet_full}, suggesting that PRISM extends to broader properties.

\subsection{Pure-solvent external transfer}

\label{sec:pure_solvent_results}
\begin{table}[t]
  \centering
  \small
  \setlength{\tabcolsep}{8pt}
  \caption{
  \textbf{Pure-solvent external transfer from BigSolDB to Leeds and SolProp.}
  {Published baseline values are from prior reported protocols; XGBoost and PRISM are  mean $\pm$ std over three seeds.}
  }
  \label{tab:pure_full}
  \begin{tabular}{lcccc}
    \toprule
    Model 
    & \multicolumn{2}{c}{RMSE $\downarrow$} 
    & \multicolumn{2}{c}{$\%|\Delta \log S|\leq 1$ $\uparrow$} \\
    \cmidrule(lr){2-3}\cmidrule(lr){4-5}
    & Leeds & SolProp & Leeds & SolProp \\
    \midrule
    Vermeire \citep{attia_data-driven_2025} 
    & 2.16 & 1.43 & 41.2 & 66.9 \\
    MolFormer \citep{attia_data-driven_2025} 
    & 1.07 & 0.98 & 64.1 & 70.7 \\
    ChemBERTa \citep{attia_data-driven_2025} 
    & 1.26 & 1.01 & 57.0 & 72.4 \\
    Chemprop \citep{attia_data-driven_2025} 
    & 0.99 & 0.83 & 70.9 & 76.1 \\
    Fastprop\citep{attia_data-driven_2025} 
    & 0.95 & 0.83 & 73.6 & 78.1 \\
    {XGBoost}
    & {$1.00\pm0.003$}
    & {$0.78\pm0.008$}
    & {$69.4\pm0.2$}
    & {$78.3\pm1.2$} \\
    \midrule
    \textbf{PRISM}
    & $\mathbf{0.91\pm0.002}$ 
    & $\mathbf{0.77\pm0.030}$ 
    & $\mathbf{75.7\pm0.9}$ 
    & $\mathbf{82.3\pm0.8}$ \\
    \bottomrule
  \end{tabular}
  \vspace{-2mm}
\end{table}
Practical solubility prediction goes beyond single-molecule aqueous solubility and requires modeling solute--solvent relations across diverse solvents.
Therefore, we test generalizations of solvent OOD and solute OOD using cross-dataset transfer: PRISM is pretrained on ZINC, fine-tuned on BigSolDB, and evaluated on Leeds and SolProp.  The result is shown in Table~\ref{tab:pure_full}.

{PRISM achieves the best results on both external datasets, improving RMSE and $\%|\Delta \log S|\leq 1$ over supervised and SSL-pretrained baselines.
Its advantage holds over XGBoost with direct access to the same prior features, suggesting that the gain does not arise from the priors alone.
It further outperforms Fastprop, a strong solubility model with thermodynamic constraints, without using the same loss.
These results are particularly notable given the pretraining scale: MolFormer uses 1.1B molecules and ChemBERTa 77M, whereas PRISM is pretrained on only 250K molecules.
Together with the temperature-trend analysis in Appendix~\ref{app:temperature_trends}, these results suggest that prior-induced relational pretraining extracts compact, transfer-relevant structure for solubility prediction.}

\subsection{{Mixed-solvent OOD generalization and representation reuse}}

\begin{table}[H]
\centering
\caption{
\textbf{Mixed-solvent solubility prediction under random and solvent-pair-disjoint splits.}
{
All results are mean $\pm$ std over five independently resampled splits under the corresponding protocol.
}}
\label{tab:mixedsolv_main}
\setlength{\tabcolsep}{3pt}
\renewcommand{\arraystretch}{1.02}
\scriptsize

\resizebox{0.98\linewidth}{!}{
\begin{tabular}{lcccccc}
\toprule
& \multicolumn{3}{c}{Random split}
& \multicolumn{3}{c}{Solvent-pair-disjoint OOD} \\
\cmidrule(lr){2-4}\cmidrule(lr){5-7}
Model
& RMSE $\downarrow$
& $R^2$ $\uparrow$
& $\%|\Delta \log S|\leq1$ $\uparrow$
& RMSE $\downarrow$
& $R^2$ $\uparrow$
& $\%|\Delta \log S|\leq1$ $\uparrow$ \\
\midrule

\multicolumn{7}{l}{\textit{Direct-prior control}} \\
{XGBoost}
& {$0.055\pm0.011$}
& {$\mathbf{0.998\pm0.001}$}
& {$\mathbf{100\pm0.0}$}
& {$0.459\pm0.022$}
& {$0.796\pm0.086$}
& {$96.0\pm0.8$} \\

\midrule
\multicolumn{7}{l}{\textit{End-to-end D-MPNN aggregation controls}} \\
D-MPNN + Concat
& $0.103\pm0.003$
& {$0.995\pm0.000$}
& $99.9\pm0.0$
& {$0.587\pm0.093$}
& {$0.703\pm0.072$}
& {$91.6\pm3.8$} \\

D-MPNN + Weighted Sum
& $0.101\pm0.005$
& {$0.995\pm0.001$}
& $\mathbf{100\pm0.0}$
& {$0.498\pm0.161$}
& {$0.766\pm0.134$}
& {$93.8\pm4.5$} \\

D-MPNN + HyperGraph
& $0.089\pm0.003$
& {$0.996\pm0.003$}
& $\mathbf{100\pm0.0}$
& {$0.457\pm0.211$}
& {$0.785\pm0.142$}
& {$95.2\pm6.0$} \\

D-MPNN + Cross Attention
& $0.119\pm0.005$
& {$0.993\pm0.001$}
& $\mathbf{100\pm0.0}$
& {$0.483\pm0.186$}
& {$0.761\pm0.116$}
& {$94.2\pm6.8$} \\

\midrule
\multicolumn{7}{l}{\textit{Specialized interaction models}} \\
{Triple-CIGIN\citep{pathak2020chemically}}
& {$\mathbf{0.052\pm0.004}$}
& {$\mathbf{0.998\pm0.001}$}
& {$\mathbf{100\pm0.0}$}
& {$0.674\pm0.123$}
& {$0.639\pm0.104$}
& {$88.1\pm5.0$} \\

{Triple-MMGNN\citep{du2024mmgnn}}
& {$0.065\pm0.001$}
& {$\mathbf{0.998\pm0.001}$}
& {$\mathbf{100\pm0.0}$}
& {$0.639\pm0.190$}
& {$0.661\pm0.178$}
& {$89.6\pm5.4$} \\

\midrule
\multicolumn{7}{l}{\textit{Frozen SSL + composition-weighted pooling}} \\
{ChemBERTa\citep{ChemBERTa}}
& {$0.109\pm0.015$}
& {$0.995\pm0.002$}
& {$99.9\pm0.1$}
& {$0.503\pm0.106$}
& {$0.787\pm0.053$}
& {$93.0\pm4.9$} \\

{MolCLR-GCN\citep{molclr}}
& {$0.135\pm0.020$}
& {$0.990\pm0.004$}
& {$99.8\pm0.2$}
& {$0.592\pm0.146$}
& {$0.709\pm0.140$}
& {$88.5\pm8.7$} \\

{GROVER\citep{rong2020self}}
& {$0.119\pm0.009$}
& {$0.993\pm0.001$}
& {$99.9\pm0.1$}
& {$0.737\pm0.327$}
& {$0.523\pm0.305$}
& {$90.0\pm9.5$} \\

{MolCLR-GIN\citep{molclr}}
& {$0.115\pm0.008$}
& {$0.993\pm0.001$}
& {$99.9\pm0.0$}
& {$0.797\pm0.448$}
& {$0.478\pm0.370$}
& {$85.6\pm13.3$} \\

\midrule
\textbf{PRISM}
& {$0.077\pm0.006$}
& {$\mathbf{0.998\pm0.001}$}
& $\mathbf{100.0\pm0.0}$
& {$\mathbf{0.438\pm0.030}$}
& {$\mathbf{0.846\pm0.031}$}
& $\mathbf{96.6\pm0.9}$ \\
\bottomrule
\end{tabular}
}
\vspace{-2mm}
\end{table}

Beyond pure-solvent transfer, practical solubility prediction often involves mixed solvents, where composition-dependent solvent--solvent interactions also matter.
We evaluate random and solvent-pair-disjoint splits using frozen PRISM encoders, fusion initialized from the pure-solvent task, and the composition-weighted pooling head in Eq.~\eqref{eq:mix_head_simple} (Figure~\ref{fig:task_heads}c), testing whether mixture-relevant information can be recovered through lightweight downstream adaptation.

{
We compare against direct-prior XGBoost, same-backbone D-MPNN controls, specialized interaction models, and frozen SSL baselines.
Under the random split, most methods are near saturation, with differences largely within {measurement uncertainty}.
Under solvent-pair-disjoint OOD, PRISM achieves the best mean RMSE, $R^2$, and $\%|\Delta\log S|\leq1$, with substantially lower split-to-split variation than most baselines.
The comparison with XGBoost shows that direct access to the priors alone is insufficient, while the D-MPNN and interaction-model controls show that model capacity or explicit pairwise interaction modeling does not recover the same transferability.
Together with the frozen-SSL baselines, these results suggest that separating complementary prior-defined relations during pretraining yields representations that are more reusable under unseen solvent combinations.
Additional ablations support solvent aggregation and fusion methods(Appendix~\ref{app:mixed_aggregation_fusion}).
}

{
We further evaluate PRISM on the independent Cenci et al.~dataset~\citep{cenci2024solubility}, covering pure, binary, and ternary solvent systems.
PRISM improves overall held-out prediction over UNIFAC--PLS, while combining PRISM with UNIFAC gives the strongest overall and ternary results (Appendix~\ref{app:cenci_mixture}).
}

\section{Ablation Study}
\label{sec:ablation}
{
This section examines the two representation-learning choices introduced in Section~\ref{sec:introduction}: \emph{how} prior relations should be encoded and \emph{which} priors are useful for a downstream task.
We first use controlled variants to isolate the roles of relational pretraining and prior-specific late fusion, distinguishing them from direct prior input, reconstruction, increased model capacity, and alternative multi-prior fusion strategies.
We then use label energy to diagnose prior--task alignment and characterize when a scientific prior is likely to provide useful supervision.
Together, these analyses examine both the \emph{alignment design} and the \emph{prior source} underlying PRISM.}

\subsection{Solubility gains require PRISM pretraining}
\label{sec:solubility_ablation}

We first test whether PRISM's solubility gains come from prior-induced relational pretraining rather than model capacity or direct descriptor use.
Table~\ref{tab:solv_pretrain_ablation} compares PRISM with five controls: D-MPNN w.\ prior directly concatenates prior descriptors; D-MPNN (reconstruction) uses the same encoder architecture but replaces PRISM's relational alignment objective with pointwise descriptor reconstruction during pretraining; D-MPNN (all spaces) uses one encoder trained with all prior-alignment losses; 4$\times$D-MPNN (scratch) keeps the multi-branch architecture without pretraining; and PRISM-Distilled compresses PRISM's four pretrained encoders into a single encoder by aligning it with the multi-prior representation of PRISM under the same KL objective (Appendix~\ref{app:distillation}).

\begin{table*}[h]
\centering
\caption{
\textbf{Effect of PRISM pretraining on solubility transfer.}
Controls use comparable downstream heads. Mixed OOD denotes the solvent-pair-disjoint split.
}
\label{tab:solv_pretrain_ablation}
\setlength{\tabcolsep}{3.5pt}
\renewcommand{\arraystretch}{1.05}
\scriptsize
\resizebox{0.98\textwidth}{!}{
\begin{tabular}{lcccccc}
\toprule
& \multicolumn{2}{c}{Leeds}
& \multicolumn{2}{c}{SolProp}
& \multicolumn{2}{c}{Mixed OOD} \\
\cmidrule(lr){2-3}\cmidrule(lr){4-5}\cmidrule(lr){6-7}
Model
& RMSE $\downarrow$ & $\%|\Delta \log S|\leq 1$ $\uparrow$
& RMSE $\downarrow$ & $\%|\Delta \log S|\leq 1$ $\uparrow$
& RMSE $\downarrow$ & $\%|\Delta \log S|\leq 1$ $\uparrow$ \\
\midrule
D-MPNN w. prior
& $0.990\!\pm\!0.013$ & $72.3\!\pm\!1.2$
& $0.820\!\pm\!0.035$ & $80.2\!\pm\!0.9$
& $0.466\!\pm\!0.078$ & $95.3\!\pm\!1.7$ \\
D-MPNN (reconstruction)
& $1.076\!\pm\!0.033$ & $62.1\!\pm\!2.1$
& $0.956\!\pm\!0.082$ & $69.0\!\pm\!4.7$
& $0.618\!\pm\!0.133$ & $87.3\!\pm\!2.1$ \\
D-MPNN (all spaces)
& $0.930\!\pm\!0.004$ & $73.9\!\pm\!1.6$
& $0.860\!\pm\!0.017$ & $77.2\!\pm\!0.8$
& $0.580\!\pm\!0.107$ & $92.5\!\pm\!2.1$ \\
4$\times$D-MPNN (scratch)
& $1.001\!\pm\!0.020$ & $71.6\!\pm\!0.7$
& $0.901\!\pm\!0.038$ & $74.7\!\pm\!3.4$
& $0.472\!\pm\!0.123$ & $92.1\!\pm\!3.6$ \\
PRISM-Distilled
& $0.950\!\pm\!0.014$ & $73.9\!\pm\!0.6$
& $0.970\!\pm\!0.024$ & $77.2\!\pm\!1.1$
& $0.457\!\pm\!0.110$ & $95.6\!\pm\!3.8$ \\
\midrule
\textbf{PRISM}
& $\mathbf{0.910\!\pm\!0.002}$ & $\mathbf{75.7\!\pm\!0.9}$
& $\mathbf{0.770\!\pm\!0.030}$ & $\mathbf{82.3\!\pm\!0.8}$
& $\mathbf{0.440\!\pm\!0.010}$ & $\mathbf{96.6\!\pm\!0.9}$ \\
\bottomrule
\end{tabular}
}
\vspace{-2mm}
\end{table*}

These controls support two conclusions about why PRISM works.

First, the selected physicochemical priors are informative, but are most effective as relational supervision.
D-MPNN w.\ prior is competitive, showing that the priors carry task-relevant signal; yet descriptor reconstruction is among the weakest variants, and 4$\times$D-MPNN (scratch) underperforms despite matching PRISM's branch capacity.
Thus, PRISM's gain is not {explained by} descriptor access, pointwise prior prediction, or architecture alone, but {by} transferring prior-induced relational information into the embedding space.
{This distinguishes the information contained in the prior coordinates from the relational structure they induce over molecules.}

Second, multi-prior structure matters.
D-MPNN (all spaces) and PRISM-Distilled are competitive in some settings, showing that prior-induced information can be partially recovered by fused or compressed designs, but neither matches full PRISM.
D-MPNN (all spaces) corresponds to GraphMSL-style early fusion~\cite{xu2024graphmultisimilaritylearningmolecular}: supervising one shared embedding distribution with all priors is equivalent to first fusing the prior-induced target distributions and then matching this fused target by KL (Appendix~\ref{app:loss}).
Its weaker performance shows that learning prior-specific geometries before fusion preserves more {task-relevant} information than fusing prior targets before learning.
PRISM-Distilled probes the opposite direction: relational distillation compresses the multi-branch teacher into a single encoder, preserving useful signal on relation-centered solubility tasks but losing ground {relative to the full model} in Appendix~\ref{app:distillation}.
Together, these results support PRISM's late-fusion design: per-prior geometries are {partly compromised} by early fusion and only partly recovered by post-hoc compression, so PRISM keeps them separate during pretraining and combines them only downstream.
{The comparison therefore isolates not only whether multiple priors are useful, but also when their distinct relational geometries should be combined.}

\subsection{Label energy diagnoses prior--task alignment}
\label{sec:label_energy_results}
\begin{table}[h]
\centering
\caption{
\textbf{Prior--task alignment and Gaussian-prior control on solubility-related tasks.}
Lower $\rho_s$ indicates smoother prior-induced label neighborhoods.
All downstream columns report RMSE ($\downarrow$).
}
\label{tab:energy_gaussian_solubility}
\setlength{\tabcolsep}{4pt}
\small
\resizebox{0.98\linewidth}{!}{
\begin{tabular}{lcccccc}
\toprule
Dataset 
& $\rho_{\mathrm{VSA}}$ 
& $\rho_{\mathrm{FP}}$ 
& Gaussian prior trained PRISM
& 4$\times$D-MPNN scratch
& D-MPNN
& PRISM \\
\midrule
ESOL     
& 0.107 
& 0.150 
& $1.24\pm0.025$ 
& $0.848\pm0.025$
& $1.050\pm0.008$ 
& $\mathbf{0.772\pm0.002}$ \\
FreeSolv 
& 0.072 
& 0.090 
& $2.79\pm0.058$ 
& $2.222\pm0.054$
& $\mathbf{1.673\pm0.082}$ 
& $1.770\pm0.093$ \\
Lipo     
& 0.305 
& 0.315 
& $1.07\pm0.014$ 
& $0.642\pm0.014$
& $0.683\pm0.016$ 
& $\mathbf{0.538\pm0.018}$ \\
\bottomrule
\end{tabular}
}
\vspace{-2mm}
\end{table}

We next test whether the selected priors are meaningful for solubility-related tasks.
Table~\ref{tab:energy_gaussian_solubility} shows that the default VSA and fingerprint priors induce smoother-than-random label neighborhoods on ESOL, FreeSolv, and Lipo (\(\rho_s<1\)), and PRISM improves over scratch across all three.
This supports the initial prior choices: priors related to intermolecular interaction provide chemically meaningful relational supervision to encoders for solubility prediction.

The Gaussian-prior control separates prior quality from pretraining.
Replacing chemical priors with an uninformative random prior degrades performance, showing that the gains are not merely due to an extra pretraining stage.
{
Nevertheless, performance does not collapse, indicating that the model can still retain or recover graph-derived information beyond the imposed prior.
This further supports the flexibility of distribution matching, which shapes prior-consistent relations without forcing the embedding to discard complementary structural information.
Thus, relational pretraining is not sufficient by itself: its benefit depends on the quality of the relational structure being transferred.
}

FreeSolv remains a boundary case for PRISM with a 2D backbone: although its labels are smooth under the selected priors, PRISM is slightly weaker than supervised D-MPNN, suggesting that the 2D representation misses task-relevant factors such as molecular conformation.
With a 3D-aware SchNet backbone, PRISM improves substantially on FreeSolv (Appendix~\ref{app:extra_ablation:schnet}), showing that {a task-aligned prior is most useful when the encoder can also represent the relevant physical factors.}

Label energy can also guide prior selection when the relevant task factors are less obvious (Appendix~\ref{app:extra_ablation:dynamic_selection}).
{
First, the dynamic-selection experiments show that label energy can identify task-compatible priors: when several priors are similarly aligned, the selected set retains their performance, while poorly aligned priors can be filtered out.
Second, these results show that more priors are not necessarily better.
On CYP2D6, for example, using all candidate priors yields a ROC-AUC of $42.08\pm9.35$, whereas selecting the label-energy-compatible Pharm2D prior increases it to $73.31\pm1.02$.
Thus, complementary priors are useful when their relational geometries are task-aligned, but adding incompatible views can substantially degrade representation learning.
}

\section{Prior-Grounded Interpretability}
\label{sec:prior_grounded_interpretability}

Finally, we examine whether PRISM's prior spaces provide meaningful interpretation on solubility predictions.
Figure~\ref{fig:pure_interpret} gives two examples: a temperature-dependent ranking study and molecule-wise prior-space contribution study.

\subsection{Temperature-dependent ranking beyond structural similarity}
\label{sec:pure_interpret:ranking}

Figure~3a compares experimental and predicted $\log S$--$T$ trends for two solutes in a fixed solvent, testing whether the solute ordering is preserved across temperature. Both PRISM and the baseline model Fastprop successfully predict the temperature trend.
{This demonstrates that PRISM can capture the observed temperature-dependent trends without explicit
temperature-gradient supervision.} Although {pyrimidine-4-carboxamide and isatoic anhydride are almost equidistant from the solvent in the fingerprint, PEOE, and SMR prior spaces}, their {distances in the SlogP prior space} differ significantly, {making SlogP the only clearly distinguishing space among the four priors in this comparison}. {While the baseline reverses the experimental solute ordering, suggesting difficulty in resolving these
subtle chemical differences, PRISM effectively captures those differences through prior-induced relational pretraining, which allows PRISM to preserve the correct solute ranking.}

\begin{figure}[h]
  \centering
  \includegraphics[width=\columnwidth]{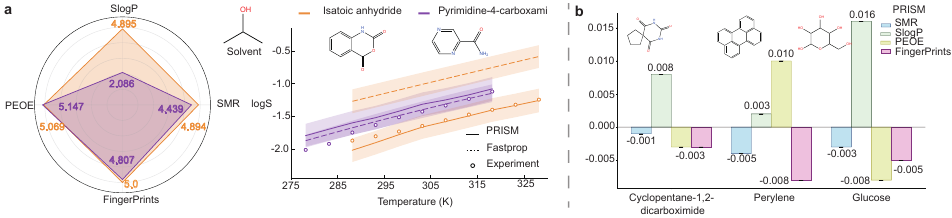}
  \caption{
    \textbf{Interpretability diagnostics in pure-solvent solubility.}
    (a) Prior-space profiles and temperature-dependent predictions for two representative solutes in a fixed solvent.
    PRISM preserves the correct solute ranking, whereas Fastprop reverses it.
    (b) Molecule-wise space attribution on ESOL in water, shown as fusion-weight deviation from a uniform baseline.
    }
  \label{fig:pure_interpret}
\end{figure}

\subsection{Molecule-wise prior-space attribution in water}
\label{sec:pure_interpret:attribution}

Figure~\ref{fig:pure_interpret}b reports molecule-wise space attribution on ESOL in water, measured as adaptive fusion-weight deviation from a uniform baseline.
Because each branch is aligned with a distinct prior space, these deviations indicate which chemical views PRISM emphasizes for each solute.

For cyclopentane-1,2-dicarboximide and glucose, the positive SlogP-related contribution suggests that hydrophobicity modulates how polar functionality translates into aqueous solubility.
Glucose has multiple hydrogen-bond donors and acceptors and is highly water-compatible, whereas cyclopentane-1,2-dicarboximide reflects a stronger trade-off between polar hydration and hydrophobic structure.

For perylene, a nonpolar \(\pi\)-rich aromatic solute, attribution shifts toward PEOE-related information.
We interpret this as a partial-charge-based electrostatic signal.
This may reflect interactions between water-derived protons, or hydrated protons, and the extended aromatic \(\pi\) system, including cation--\(\pi\)-like effects~\citep{dougherty_cation_pi_2025}.
Because the partial charges of conjugated aromatic \(\pi\) systems are not explicitly calculated in PEOE, this interaction can be attributed to the collective effect of $sp^2$-connected atoms.

Overall, these diagnostics suggest that PRISM's adaptive weights show prior-space contributions, rather than causal decompositions of individual predictions.
Perturbation-based faithfulness tests in Appendix~\ref{app:interp:faithfulness} further support this interpretation.

\section{Conclusion}
\label{sec:conclusion}
We introduced \textbf{PRISM}, a label-free framework that uses expert-defined chemical priors as relational teachers, preserving prior-induced neighborhood geometry through soft distribution matching across prior-specific branches fused only downstream. Empirically, PRISM improves solubility MoleculeNet results, pure-solvent external transfer, and mixed-solvent OOD stability with frozen encoders and a lightweight head. Ablations show that these gains require both relational pretraining and prior-specific late fusion: descriptor input, reconstruction, scratch multi-branch, and shared early-fusion variants all fall short, while PRISM-Distilled compresses only part of the branch-specific signal. A label-energy diagnostic further turns prior selection into a measurable property, separating useful priors from uninformative ones before pretraining. Low-cost chemical priors thus provide structured relational supervision for molecular OOD generalization, and we expect the principle of PRISM to extend to other domains where chemically or biologically meaningful priors are available.

\section*{Code Availability}
The code for PRISM is available at
\url{https://github.com/cxy0206/PRISM}.

\section*{Acknowledgements and Disclosures}

\textbf{Funding.}
This work was supported by the Shanghai Jiao Tong University - Global College Research Start-up Fund and Young Top-notch Talent Support Program.

\textbf{Competing interests.}
The authors declare no competing interests.

\bibliographystyle{unsrtnat}
\bibliography{chemA}

\clearpage
\appendix

\section{Notation and Organization of the Appendix}
\label{app:notation}

Table~\ref{tab:notation} summarizes the notation used throughout the paper and appendix.

\begin{table}[h]
\centering
\caption{Summary of notation. Symbols match the main text.}
\label{tab:notation}
\begin{tabular}{ll}
\toprule
Symbol & Meaning \\
\midrule
$x$ & Molecular graph input \\
$y$ & Downstream label \\
$\hat y$ & Model prediction \\
$\mathcal{D}_{\mathrm{pre}}=\{x_i\}_{i=1}^{N}$ & Unlabeled molecular graph dataset for pretraining \\
$\mathcal{B}$ & Mini-batch of molecules \\
$\mathcal{S}=\{s_1,\ldots,s_K\}$ & Set of expert-defined prior spaces \\
$\pi_i^{s_k}$ & Prior vector of molecule $i$ in prior space $s_k$ \\
$d^{s_k}(\cdot,\cdot)$ & Prior-space distance in prior space $s_k$ \\
$\delta^{s_k}(\cdot,\cdot)$ & Embedding-space distance \\
$\mathcal{A}_{\mathrm{prior}}^{s_k}$ & Relational constructor in prior space $s_k$ \\
$\mathcal{A}_{\mathrm{emb}}^{s_k}$ & Relational constructor in embedding space for branch $s_k$ \\
$\mathcal{M}^{s_k}$ & Discrepancy between prior- and embedding-induced relational objects \\
$\mathcal{R}_{\mathrm{prior}}^{s_k}$ & Relational object induced by prior space $s_k$ \\
$\mathcal{R}_{\mathrm{emb}}^{s_k}$ & Relational object induced by learned embeddings for branch $s_k$ \\
$G_{\theta_k}^{s_k}$ & Encoder branch for prior space $s_k$ \\
$r_i^{s_k}$ & Branch embedding produced by $G_{\theta_k}^{s_k}$ \\
$R_i^{s_k}$ & Prior-induced soft neighborhood distribution for anchor $i$ \\
$P_i^{s_k}$ & Learned soft neighborhood distribution for anchor $i$ \\
$\beta_k$ & Locality scale for prior space $s_k$ \\
$\beta_\tau$ & Embedding-space inverse temperature \\
\midrule
$q_{s_k}$ & Per-branch projection head for prior space $s_k$ \\
$h_i^{s_k}$ & Aligned branch embedding $q_{s_k}(r_i^{s_k})\in\mathbb{R}^{D}$ \\
$H_i$ & Stack of all branch embeddings $\in\mathbb{R}^{K\times D}$ \\
$g$ & Gating network for fusion \\
$w_i^{s_k}$ & Fusion weight for branch $s_k$ on input $i$ \\
$w_i$ & Fusion weight sequence on input $i$ \\
$z_i$ & Fused molecular embedding of molecule $i$ \\
\midrule
$\mathcal{N}^{s_k}_{n}(x_i)$ & $n$ nearest neighbors of $x_i$ in prior space $s_k$ \\
$\mathcal{E}^{s_k}_{n}(y)$ & Label energy on prior-space $n$NN graph \\
$\rho^{s_k}$ & Normalized label-energy ratio for prior space $s_k$ \\
$Q$ & The set of labeled molecules used for diagnosis \\
\midrule
$f_{\mathrm{single}}$ & Single-molecule prediction head \\
$f_{\mathrm{pure}}$ & Pure-solvent solubility prediction head \\
$f_{\mathrm{mix}}$ & Mixed-solvent solubility prediction head \\
$x_{\mathrm{sol}}$ & Solute molecule \\
$x_{\mathrm{solv}}$ & Pure solvent molecule \\
$x_1,x_2$ & Two solvents in a binary solvent mixture \\
$z_{\mathrm{sol}}$ & Fused solute embedding \\
$z_{\mathrm{solv}}$ & Fused pure-solvent embedding \\
$z_1,z_2$ & Fused embeddings of the two mixture solvents \\
$\alpha_1,\alpha_2$ & Mole fractions of the two mixture solvents \\
$z_{\mathrm{mix}}$ & Composition-weighted solvent-mixture embedding \\
$T$ & Temperature \\
$h_T$ & Temperature embedding network \\
$t=h_T(T)$ & Temperature embedding \\
\bottomrule
\end{tabular}
\end{table}

\clearpage
\section{Related Work}
\label{sec:related}

We discuss three lines of work most relevant to PRISM.

\textbf{Molecular self-supervised learning.}
Molecular SSL commonly defines pretraining signals through contrastive learning, generative pretraining, multi-view learning, or latent representation learning~\cite{cfree_neurips_2025}.
Contrastive and multi-view methods construct reference neighborhoods from augmentations, alternative molecular views, or paired representations~\cite{molclr,hu2023your,SMICLR,moon_3d_2023}, while generative and denoising methods use masking, reconstruction, or corruption kernels~\cite{vincent2011connection,Chemformer,SubgDiff}.
Recent molecular SSL further moves beyond single negative/positive-pair construction by using multi-similarity signals, fingerprint-based similarity, or persistent homology to provide richer neighborhood, ranking, or topological supervision~\cite{xu2024graphmultisimilaritylearningmolecular,park-etal-2024-moleco,luo2023improving}.
{ReLMole~\citep{10.1021/acs.jcim.2c00798} further uses atom-level and functional-group-level graph similarities to guide contrastive molecular pretraining.}
Large-scale molecular foundation models scale these pretext mechanisms and often incorporate 3D conformational information~\cite{unimol2_neurips_2024,mole_natcomm_2024}.
PRISM belongs to this relational view of SSL but changes the source and organization of the relational target: reference neighborhoods are induced by expert-defined chemical priors, and heterogeneous prior geometries are preserved in separate branches before downstream fusion.
Thus, PRISM turns chemically meaningful prior spaces into label-free relational teachers, rather than relying only on augmentation, reconstruction, or fused similarity targets.

\textbf{Molecular relational learning and solubility modeling.}
Molecular relational learning explicitly models relations among molecules, including drug--drug interactions, solute--solvent interactions, molecular complexes, and higher-order molecular sets~\cite{lee2023cgib,lee2023cmrl,zhong2024metddi,ye2024hypergraph}.
Solubility-specific models further encode solute--solvent or mixture interactions through solvent recommendation, merged molecular graphs, hypergraphs, or multicomponent architectures~\cite{DaVia2024GSKSolubility,wollschlager2024solvmate,ramani2024molmerger,du2025mmhnn,jung2025multicomponent}.
Recent binary- and multicomponent-solvent benchmarks also highlight the need to evaluate prediction under unseen solvent environments~\cite{bao2024binarysolubility,rajaonson2025chemixhub}.
These methods are strong supervised relation predictors; PRISM is complementary, pretraining reusable molecular encoders from unlabeled molecules through prior-induced relational supervision.

\textbf{Descriptor-based chemical spaces.}
Molecular descriptors and fingerprints provide interpretable chemical coordinate systems for physicochemical surfaces, substructures, pharmacophores, and 3D conformation~\cite{hansen2007hansen,vsa,bajusz2015tanimoto}.
They are commonly used as inputs, auxiliary features, reconstruction targets, or visualization tools~\cite{Defining-and-Exploring,Exploring-and-mapping}.
PRISM uses them differently: descriptors and fingerprints are not appended to downstream predictors or reconstructed pointwise; they define population-level stochastic neighborhoods for pretraining.
Because different priors induce different neighborhoods over the same molecules~\cite{Diverse,Similar}, PRISM keeps prior-specific branches and fuses them only downstream, enabling chemically structured relational supervision with interpretable prior-level contributions.

\clearpage
\section{Datasets, Splits, and Evaluation Protocols}
\label{app:datasets}

This section details the datasets, split protocols, leakage controls, and evaluation metrics used throughout the paper.
Unless otherwise noted, molecules are identified by canonical SMILES strings.

\subsection{Dataset Overview}
\label{app:datasets:overview}

We evaluate PRISM on three families of downstream tasks.
First, we use standard MoleculeNet benchmarks to assess molecular representation quality under scaffold-disjoint molecular property prediction.
Second, we evaluate pure-solvent solubility prediction, where each example contains a solute, a solvent, temperature, and a solubility label.
Third, we evaluate mixed-solvent solubility prediction, where each example contains a solute, two solvents, mixture composition, temperature, and a solubility label.

Table~\ref{tab:dataset_summary} summarizes the datasets used in this work.
The unlabeled ZINC250k subset is used only for PRISM pretraining.
All labeled downstream datasets are used only for supervised fine-tuning and evaluation.

\begin{table}[h]
\centering
\small
\setlength{\tabcolsep}{5pt}
\caption{Summary of datasets used in this work.}
\label{tab:dataset_summary}
\resizebox{0.99\textwidth}{!}{%
\begin{tabular}{lcccl}
\toprule
Dataset & Task & \#Samples & Split & Metric \\
\midrule
ZINC250k & pretraining & $\sim$250k & --- & --- \\
\midrule
BBBP & classification & 2,039 & scaffold 8:1:1 & ROC-AUC $\uparrow$ \\
BACE & classification & 1,513 & scaffold 8:1:1 & ROC-AUC $\uparrow$ \\
HIV & classification & 41,127 & scaffold 8:1:1 & ROC-AUC $\uparrow$ \\
\midrule
ESOL & regression & 1,128 & scaffold 8:1:1 & RMSE $\downarrow$ \\
FreeSolv & regression & 642 & scaffold 8:1:1 & RMSE $\downarrow$ \\
Lipo & regression & 4,200 & scaffold 8:1:1 & RMSE $\downarrow$ \\
\midrule
CYP2D6 & classification & 13,130 & scaffold 8:1:1 & ROC-AUC $\uparrow$ \\
KOR & classification & 7,963 & scaffold 8:1:1 & ROC-AUC $\uparrow$ \\
\midrule
BigSolDB\citep{attia_data-driven_2025} & pure-solvent solubility & 41,724 & solute-group 95/5 & RMSE / $\%|\Delta\log S|\leq1$ \\
SolProp\citep{Vermeire} & pure-solvent solubility & 5,230 & external test & RMSE / $\%|\Delta\log S|\leq1$ \\
Leeds\citep{Leeds} & pure-solvent solubility & 1,294 & external test & RMSE / $\%|\Delta\log S|\leq1$ \\
\midrule
Mixed-solvent~\cite{bao2024binarysolubility} & binary-mixture solubility & 26,730 & random 8:1:1 / pair-disjoint & RMSE / $\%|\Delta\log S|\leq1$ \\
\bottomrule
\end{tabular}%
}
\end{table}

\subsection{MoleculeNet Property Benchmarks}
\label{app:datasets:molnet}

We use MoleculeNet benchmarks to evaluate structural generalization under scaffold-disjoint splits.
The main text focuses on three regression tasks that are most directly related to solubility and solvent-relevant physicochemical behavior: ESOL, FreeSolv, and Lipophilicity.
Additional classification benchmarks, including BBBP, BACE, and HIV, are reported in the appendix to provide a broader view of representation quality beyond solubility-centered regression.

Each sample is represented by a canonical SMILES string, and the target corresponds to an experimentally measured or curated molecular property.

\paragraph{Regression tasks emphasized in the main text.}
\textbf{ESOL} predicts aqueous solubility from molecular structure.
\textbf{FreeSolv} predicts hydration free energy of small molecules in water.
\textbf{Lipophilicity} predicts experimental lipophilicity, reflecting hydrophobicity and partitioning behavior.
These tasks are used as the primary MoleculeNet structural-OOD evaluation because they are closest to the solubility-centered scope of this work.

\paragraph{Additional classification tasks.}
\textbf{BBBP} is a binary classification task indicating whether a compound can penetrate the blood--brain barrier.
\textbf{BACE} is a binary classification task for inhibition activity against the $\beta$-secretase 1 target.
\textbf{HIV} is a binary classification task indicating whether a compound shows activity in HIV replication inhibition assays.
We include these tasks in the extended appendix results to assess whether PRISM also provides useful structural-OOD representations outside solubility-related regression.

\paragraph{Split protocol.}
For all MoleculeNet benchmarks, we use scaffold splits with an 8:1:1 train/validation/test ratio.
This protocol evaluates generalization across scaffold-disjoint chemical series rather than random interpolation among closely related molecules.
We report mean $\pm$ standard deviation across three random seeds.
For classification tasks, we report ROC-AUC; for regression tasks, we report RMSE.

\paragraph{Note on MoleculeNet evaluation.}
We acknowledge known limitations of MoleculeNet evaluation, including scaffold-split inconsistency, benchmark leakage, and seed sensitivity~\cite{ash2025jcim}.
Accordingly, MoleculeNet is used as supplementary structural-OOD evidence.
The main empirical support for PRISM comes from solubility-centered external transfer and mixed-solvent solvent-pair-disjoint evaluation.

\subsection{Additional Bioactivity Benchmarks for Prior Generality}
\label{app:datasets:bioactivity}

In addition to the six MoleculeNet benchmarks used in the main comparison, we include two additional bioactivity datasets, \textbf{CYP2D6} \citep{Veith2009CYPSelectivity} and \textbf{KOR}, in the prior-generality analyses.
These datasets are not used as primary benchmark tasks in the main performance table.
They are used to test whether PRISM remains useful when the prior space is changed, corrupted, or selected using the label-energy diagnostic.

\paragraph{CYP2D6.}
CYP2D6 is a small-molecule bioactivity classification task related to cytochrome P450 2D6, with 13,130 samples after canonicalization and deduplication.
It provides a pharmacologically relevant setting where functional-group and pharmacophore information can be more directly relevant than generic physicochemical descriptors.

\paragraph{KOR.}
KOR is a small-molecule bioactivity classification task associated with the kappa opioid receptor, with 7,963 samples.
We include it as an additional receptor-level activity task to evaluate whether pharmacophore-informed priors provide useful relational supervision beyond standard MoleculeNet tasks.

\paragraph{Usage in this work.}
CYP2D6 and KOR are used only in the supplementary prior-generality experiments, including Gaussian random prior degradation and dynamic prior selection with Pharm2D.
For these datasets, we follow the same supervised fine-tuning protocol as the MoleculeNet classification tasks and report ROC-AUC, averaged over three seeds.

\subsection{Pure-Solvent Solubility Datasets}
\label{app:datasets:pure_solvent}

Pure-solvent solubility prediction evaluates whether a model can predict solute solubility across different solvents and temperatures.
Each record is $\langle x_{\mathrm{sol}}, x_{\mathrm{solv}}, T, \log S\rangle$, where $T$ is temperature in units of K and $\log S$ is the experimental solubility label.

\paragraph{BigSolDB.}
We use BigSolDB as the training and validation source.
BigSolDB contains experimental solubility measurements across diverse organic solvents and temperatures.
To evaluate solute extrapolation, we group all records by canonical solute SMILES and assign each solute group entirely to either training or validation.
We use 95\% of the filtered BigSolDB data for training and 5\% for validation/model selection.

\paragraph{SolProp and Leeds.}
We use SolProp and Leeds as external test sets.
They are not used for training, validation, or hyperparameter selection.
Before training on BigSolDB, we remove all BigSolDB records whose canonical solute SMILES appear in either SolProp or Leeds.
This overlap-removal step ensures that external evaluation measures extrapolate to solutes not seen during training.

\subsection{Mixed-Solvent Solubility Dataset}
\label{app:datasets:mixed_solvent}

For mixed-solvent solubility prediction, we use the binary-solvent drug solubility dataset introduced by Bao et al.~\cite{bao2024binarysolubility}, which is also included as the drug-solubility task in CheMixHub~\cite{rajaonson2025chemixhub}.
This dataset was curated from the literature to support the prediction of small-molecule solubility in binary solvent mixtures under varying temperatures. 

Each record is
\begin{equation}
    \langle x_{\mathrm{sol}}, x_1, x_2, \alpha_1, \alpha_2, T, \log S \rangle,
\end{equation}
where $x_{\mathrm{sol}}$ is the solute, $x_1$ and $x_2$ are the two solvents, $\alpha_1$ and $\alpha_2$ are their mole fractions ($\alpha_1+\alpha_2=1$), $T$ is temperature, and $\log S$ is solubility.
The original dataset standardizes solubility values to $\log S$, which denotes the base-10 logarithm of $S$ with $S$ measured in g/100g, and converts reported solvent ratios into consistent mixture-composition representations. 
Following the CheMixHub curation convention, we use mole fractions as composition variables in all mixed-solvent experiments.

After cleaning, deduplication, and retaining samples with valid molecular identifiers and mole-fraction fields, our working mixed-solvent dataset contains 26,730 experimental samples.
It covers 125 unique solutes, 44 unique solvents, 110 unique binary solvent pairs, and 377 unique solute--binary-solvent systems.
The temperature range is 252.55--368.15 K, the mole-fraction range is 0.0--1.0, and the $\log S$ range is $-8.85$ to $-0.02$.

Because the dataset includes endpoint compositions ($\alpha=0$ or $1$), we additionally check exact endpoint duplicates defined by solute, effective pure solvent, temperature, and label. No exact endpoint record is shared across train/validation/test splits under the solvent-pair-disjoint protocol.

\paragraph{Why this dataset.}
Mixed-solvent solubility is a natural testbed for PRISM because the target depends not only on the solute and each solvent individually, but also on solvent--solvent and solute--environment interactions.
The dataset therefore evaluates whether prior-induced molecular representations remain useful in a multi-molecular setting where composition and temperature modulate the effective solvent environment.

\paragraph{Dataset statistics.}
Table~\ref{tab:mixed_stats} summarizes the processed mixed-solvent dataset.

\begin{table}[h]
\centering
\caption{Statistics of the processed mixed-solvent solubility dataset.}
\label{tab:mixed_stats}
\begin{tabular}{ll}
\toprule
\textbf{Metric} & \textbf{Value} \\
\midrule
Total samples & 26,730 \\
Unique solutes & 125 \\
Unique solvents & 44 \\
Unique binary solvent pairs & 110 \\
Unique solute--binary-solvent systems & 377 \\
Temperature range (K) & 252.55--368.15 \\
Mole fraction range & 0.0--1.0 \\
Solubility range, $\log S$ & $-8.85$--$-0.02$ \\
Target unit & $\log(S)$, $S$ in g/100g \\
\bottomrule
\end{tabular}
\end{table}

\subsection{Splits and Leakage Control}
\label{app:datasets:leakage}

\paragraph{MoleculeNet scaffold splits.}
For MoleculeNet benchmarks, we use scaffold-disjoint 8:1:1 train/validation/test splits.
This split evaluates generalization across molecular scaffolds and is used for all MoleculeNet classification and regression experiments.

\paragraph{Pure-solvent solute-level split.}
For BigSolDB, all records associated with the same canonical solute SMILES are assigned to the same split.
This prevents the same solute from appearing in both training and validation under different solvent or temperature conditions.
For external evaluation, we additionally remove from BigSolDB all solutes that appear in SolProp or Leeds.

\paragraph{Mixed-solvent random split.}
For the mixed-solvent dataset, we first consider a sample-level random split with an 8:1:1 train/validation/test ratio.
This split follows the same random-split principle as the default random benchmark setting in CheMixHub, but uses the 8:1:1 ratio adopted in our experiments.
All samples are randomly shuffled with a fixed seed and then assigned to training, validation, and test sets.
This split is not solvent-pair-disjoint, so the same binary solvent pair may appear across training and test sets.
We report this split as an in-distribution benchmark rather than as a strict OOD evaluation.

\paragraph{Mixed-solvent solvent-pair canonicalization.}
A binary solvent mixture should be invariant to the ordering of the two solvents.
To prevent representing the same physical mixture as either $(A,B)$ or $(B,A)$, we canonicalize each solvent pair by sorting the two canonical solvent SMILES:
\begin{equation}
    \mathrm{pair}(s_1,s_2)
    =
    (\min(s_1,s_2), \max(s_1,s_2)).
\end{equation}
All pair-level split keys are computed using this canonicalized representation.

\paragraph{Mixed-solvent solvent-pair-disjoint split.}
We further evaluate a stricter solvent-pair-disjoint split, designed to test generalization to unseen binary solvent environments.
We collect the set of unique canonical solvent-pair keys $\mathcal{P}$, randomly shuffle them with a fixed seed, and split them into train/validation/test subsets.
A sample is assigned to a split only through its canonical solvent-pair key:
\begin{equation}
    (x_{\mathrm{sol}},s_1,s_2,\alpha_1,\alpha_2,T,y)
    \in \mathcal{D}_{\mathrm{train}}
    \quad
    \Longleftrightarrow
    \quad
    \mathrm{pair}(s_1,s_2) \in \mathcal{P}_{\mathrm{train}},
\end{equation}
and analogously for validation and test.
This protocol guarantees that no binary solvent pair appears in more than one split.
Because different solvent pairs have different numbers of measurements, the final sample-level proportions may deviate from the nominal pair-level proportions.
Across repeated runs, the solvent-pair split is resampled with different random seeds.

\paragraph{Interpretation of the two mixed-solvent splits.}
The random split measures whether the model can learn the overall solubility landscape under the same data distribution used for training.
The solvent-pair-disjoint split measures environment-OOD generalization, where the model must predict solubility for binary solvent pairs not observed during training.
We treat the random split as the standard in-distribution benchmark and the solvent-pair-disjoint split as the main OOD stress test.

\paragraph{Relation to CheMixHub splitting protocols.}
CheMixHub defines four benchmark splitting strategies for chemical-mixture property prediction: random cross-validation, mixture-size splits, leave-molecules-out splits, and temperature splits.
Its default random benchmark uses 5-fold random cross-validation with a 70/10/20 ratio.
Our mixed-solvent random split follows the same sample-level random-split principle but uses the 8:1:1 ratio adopted throughout our experiments.
Our solvent-pair-disjoint split differs from leave-molecules-out: leave-molecules-out removes individual chemical components from training, whereas our split removes entire binary solvent environments.
We do not use mixture-size splits because the drug-solubility dataset is restricted to solute--binary-solvent systems and does not provide meaningful variation in mixture size.

\subsection{OOD Split Visualization and Diagnostics}
\label{app:ood_visualization}
We provide additional diagnostics to clarify the distribution shifts induced by the evaluation splits.

\paragraph{Structural OOD under scaffold splits.}
Figure~\ref{fig:ood_freesolv_scaffold} visualizes the FreeSolv scaffold split in Morgan-fingerprint PCA space and compares the corresponding hydration-free-energy distributions.
The scaffold split induces visible structural shift: validation and test molecules occupy several fingerprint-space regions that are sparse in the training set.
At the same time, FreeSolv also exhibits non-negligible label-distribution differences across splits, which is expected given its small sample size.
This supports our treatment of FreeSolv as a boundary case: its evaluation reflects scaffold-level structural generalization, but performance is also affected by limited data, label-range imbalance, and hydration-specific physics that may require 3D information.

\begin{figure}[h]
    \centering
    \includegraphics[width=0.95\linewidth]{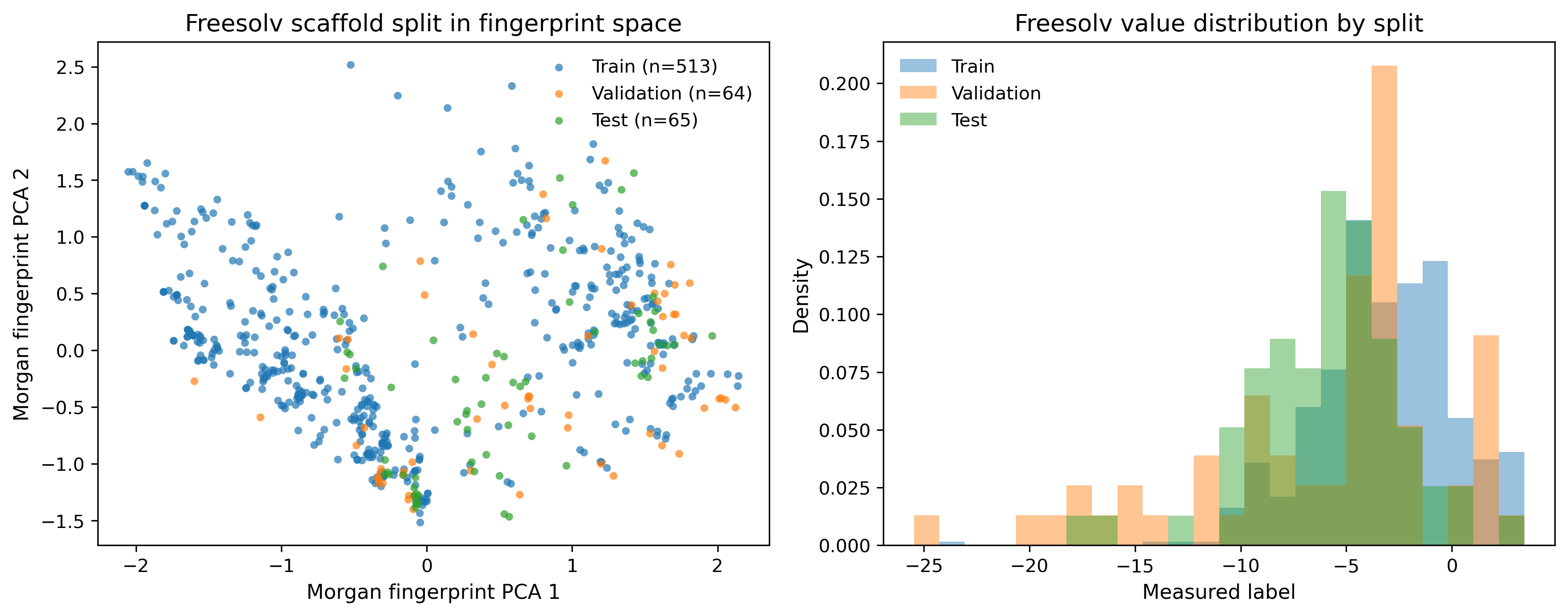}
    \caption{
    \textbf{FreeSolv scaffold split visualization.}
    }
    \label{fig:ood_freesolv_scaffold}
\end{figure}

\paragraph{Solute and solvent OOD in pure-solvent transfer.}
For pure-solvent solubility, the external-transfer setting from BigSolDB to SolProp and Leeds follows the dataset analysis in Attia et al.~\citep{attia_data-driven_2025}.
Their study shows that these external test sets contain solute and solvent shifts relative to the BigSolDB training source, making them suitable for evaluating extrapolation beyond seen solute--solvent combinations.
In our protocol, all BigSolDB records whose solute appears in SolProp or Leeds are removed before training, so every external-test solute is unseen during supervised fine-tuning.
SolProp contains both seen- and unseen-solvent cases.

\paragraph{Unseen solvent-pair OOD in mixed-solvent prediction.}
Figure~\ref{fig:ood_pair_diagnostics} provides diagnostics for the solvent-pair-disjoint mixed-solvent split.
By construction, each binary solvent pair is canonicalized by sorting the two canonical solvent SMILES and is assigned exclusively to train, validation, or test.
We explicitly verify that the canonical pair-key intersections are empty:
$\mathcal{P}_{\mathrm{train}}\cap\mathcal{P}_{\mathrm{val}}=\emptyset$,
$\mathcal{P}_{\mathrm{train}}\cap\mathcal{P}_{\mathrm{test}}=\emptyset$, and
$\mathcal{P}_{\mathrm{val}}\cap\mathcal{P}_{\mathrm{test}}=\emptyset$.
Thus, every validation or test example contains a binary solvent environment absent from training.

The marginal distributions in Figure~\ref{fig:ood_pair_diagnostics} show that log-solubility, temperature, and solvent composition remain broadly overlapping across train, validation, and test sets.
Therefore, the solvent-pair-disjoint split is not primarily a temperature-extrapolation or composition-extrapolation benchmark.
Rather, it evaluates whether a model can generalize to unseen binary solvent environments while remaining within broadly comparable label, temperature, and composition support.

\begin{figure}[t]
    \centering
    \includegraphics[width=0.95\linewidth]{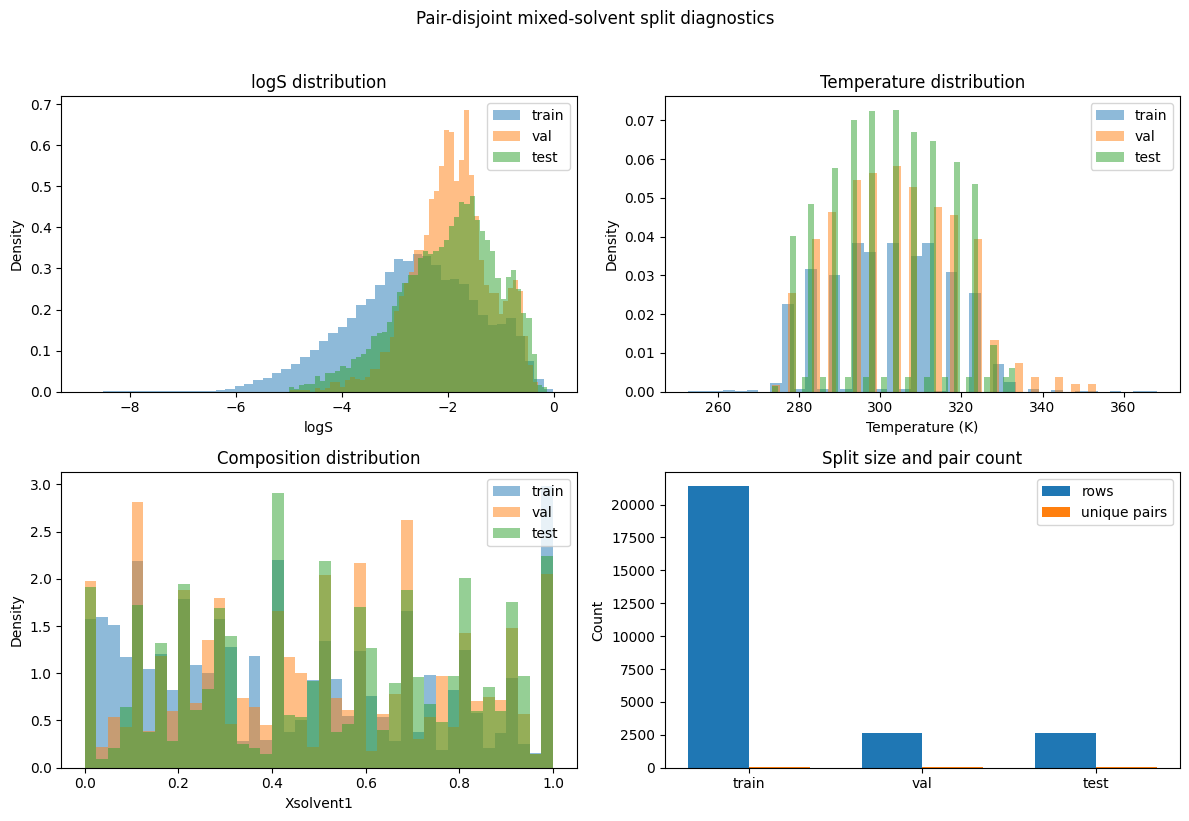}
    \caption{
    \textbf{Diagnostics for the solvent-pair-disjoint mixed-solvent split.}
    }
    \label{fig:ood_pair_diagnostics}
\end{figure}

\subsection{Evaluation Metrics and Repeated-Run Protocol}
\label{app:datasets:metrics}

For classification tasks, we report ROC-AUC.
For regression tasks, we report RMSE.
For solubility prediction, we additionally report
\begin{equation}
    \%|\Delta \log S| \leq 1
    =
    \frac{1}{N_\mathrm{test}}
    \sum_{i=1}^{N_\mathrm{test}}
    \mathbf{1}\{|\hat y_i - y_i|\leq 1\}\times 100\%.
\end{equation}

Unless otherwise specified, we report mean $\pm$ standard deviation across three random seeds.
For split-dependent evaluations, the split is regenerated under the corresponding constraint for each seed.
For the mixed-solvent random split, each seed corresponds to a new sample-level random 8:1:1 split; for the solvent-pair-disjoint split, each seed corresponds to a new split over canonical solvent-pair keys.
Thus, the reported standard deviations capture run-to-run variation including both initialization and split resampling.

\paragraph{Limitations of three-seed evaluation.}
We acknowledge that three seeds is the minimum statistical protocol and may underestimate true variance, particularly on small datasets such as FreeSolv.
This is a known limitation of standard MoleculeNet protocols~\cite{ash2025jcim}.
We mitigate this by reporting full standard deviations and by emphasizing the larger and more stable solubility benchmarks for the main empirical claims.

\subsection{Baseline Selection and Comparison Protocols}
\label{app:baseline_protocols}

Baseline sets differ across experiments because the evaluated tasks require different input structures and support different native model capabilities.
We select baselines according to five criteria: input compatibility or transparent adaptability, protocol compatibility, established competitiveness, learning-paradigm coverage, and the ability to test a specific alternative explanation for PRISM's gains.
We distinguish controlled comparisons from broader capability-matched and domain-level comparisons.
Table~\ref{tab:baseline_protocols} summarizes the resulting comparison groups and protocols.

\begin{table*}[t]
\centering
\scriptsize
\setlength{\tabcolsep}{3pt}
\renewcommand{\arraystretch}{1.05}
\caption{
\textbf{Baseline selection and comparison protocols.}
``Reported'' denotes values taken from prior studies under compatible published protocols, while ``Reproduced'' denotes baselines evaluated in our experiments.
}
\label{tab:baseline_protocols}

\begin{tabularx}{\textwidth}{
>{\raggedright\arraybackslash}p{0.145\textwidth}
>{\raggedright\arraybackslash}p{0.135\textwidth}
>{\raggedright\arraybackslash}p{0.155\textwidth}
>{\raggedright\arraybackslash}X
>{\raggedright\arraybackslash}p{0.235\textwidth}
>{\centering\arraybackslash}p{0.065\textwidth}
}
\toprule
Setting
& Comparison type
& Baseline group
& Selection rationale
& Protocol
& Source \\
\midrule

\textbf{Single-molecule prediction}
& Supervised control
& Supervised GNNs
& Tests whether relational pretraining improves over direct supervised learning from molecular graphs.
& One molecular graph as input; the full model is trained on the task-specific split.
& Reported \\

\textbf{Single-molecule prediction}
& Generic molecular SSL
& 2D and geometry-aware pretrained models
& Compares scientific relational supervision with established molecular pretraining paradigms.
& Respective pretrained encoders with task-specific downstream training.
& Reported \\

\textbf{Single-molecule prediction}
& Direct-prior control
& Traditional ML with fingerprints and VSA priors
& Tests whether gains can be explained by directly supplying the informative prior features.
& Concatenated raw prior features; no relational pretraining.
& Reproduced \\

\midrule

\textbf{External pure-solvent transfer}
& Protocol-matched supervised
& Fastprop, Chemprop, Vermeire
& Compares PRISM with established supervised solute--solvent predictors under external transfer.
& Trained on cleaned BigSolDB and tested directly on Leeds and SolProp.
& Reported \\

\textbf{External pure-solvent transfer}
& Large-scale pretrained
& MolFormer, ChemBERTa
& Tests whether generic large-scale pretraining alone explains transfer performance.
& Pretrained molecular encoders are fine-tuned on BigSolDB and evaluated unchanged on Leeds and SolProp.
& Reported \\

\textbf{External pure-solvent transfer}
& Direct-prior control
& XGBoost
& Tests whether relational pretraining adds transfer value beyond direct access to the same priors.
& Raw prior features for solute and solvent together with temperature; trained on BigSolDB and tested on Leeds and SolProp.
& Reproduced \\

\midrule

\textbf{Mixed-solvent pair-disjoint OOD}
& Controlled architecture
& Same-backbone D-MPNN aggregation controls
& Isolates relational pretraining from backbone capacity and downstream aggregation.
& Same D-MPNN backbone and five pair-disjoint splits; models are trained end-to-end without molecular pretraining.
& Reproduced \\

\textbf{Mixed-solvent pair-disjoint OOD}
& Capability-matched interaction
& Triple-CIGIN, Triple-MMGNN
& Tests whether specialized end-to-end interaction models can learn transferable solvent-pair relations.
& Full models are trained and evaluated on the same five pair-disjoint splits.
& Reproduced \\

\textbf{Mixed-solvent pair-disjoint OOD}
& Lightweight frozen SSL
& ChemBERTa, MolCLR, GROVER
& Tests whether PRISM's advantage is specific to relational representations rather than frozen molecular pretraining in general.
& Encoders are frozen; matched downstream pooling and prediction modules are trained on the same five splits.
& Reproduced \\

\textbf{Mixed-solvent pair-disjoint OOD}
& Direct-prior control
& XGBoost and prior-input MLP
& Tests whether flexible supervised models can match PRISM through direct access to the prior features.
& Raw prior features for the solute and two solvents together with composition and temperature; evaluated on the same five pair-disjoint splits.
& Reproduced \\

\midrule

\textbf{Pure/binary/ternary mixture evaluation}
& Domain-level comparison
& Cenci-style PLS and UNIFAC-based models
& Tests whether PRISM adds value beyond descriptor-based and thermodynamic approaches and whether their information is complementary.
& Dataset-specific compatible protocols; PRISM is used only in representation-based or combined-model variants.
& Reproduced \\

\bottomrule
\end{tabularx}

\vspace{-1mm}
\end{table*}

\paragraph{Source note.}
The GROVER result in the single-molecule comparison is taken from Fang et al.~\citep{fang_knowledge_2023}, rather than directly from the original GROVER paper.

\clearpage
\section{Chemical Prior Spaces}
\label{app:priors}

This section describes the chemical prior spaces used in PRISM.
The default PRISM configuration uses three VSA descriptor families and one Morgan fingerprint prior.
We additionally evaluate a Pharm2D pharmacophore prior and a 3D shape prior in supplementary generality experiments; these auxiliary priors are not part of the default four-branch model unless explicitly stated.

\subsection{Overview of Prior Spaces}
\label{app:priors:overview}

PRISM uses expert-defined chemical spaces as relational teachers.
Each prior space defines a chemically meaningful notion of molecular similarity, and the encoder is trained to preserve the corresponding neighborhood structure in the learned embedding space.
The goal is not to copy descriptor values directly, but to preserve the relative organization induced by a chosen chemical view.

The prior spaces used in this work fall into five categories:
\textbf{(i) VSA descriptor priors}, including \texttt{SMR\_VSA}, \texttt{SlogP\_VSA}, and \texttt{PEOE\_VSA};
\textbf{(ii) Morgan fingerprint priors}, encoding local substructure similarity;
\textbf{(iii) Pharm2D priors}, encoding pharmacophore-level arrangements of chemical features;
\textbf{(iv) 3D shape priors}, encoding global molecular shape and spatial mass distribution;
\textbf{(v) Gaussian random priors}, used only as a negative control.

The first two categories form the default PRISM prior set.
Pharm2D and 3D priors are used to evaluate whether PRISM is prior-agnostic and whether task-dependent prior selection can identify useful non-default priors.
The Gaussian random prior is used to test degradation under uninformative priors.

\subsection{VSA Descriptor Families}
\label{app:priors:vsa}

We use three van der Waals surface area (VSA) descriptor families: \texttt{SMR\_VSA}, \texttt{SlogP\_VSA}, and \texttt{PEOE\_VSA}.
For a molecule with atoms $a=1,\dots,N$, each atom is assigned an approximate surface-area contribution $A_a$ and an atom-level property value $q_a$.
Atoms are then assigned to predefined bins according to $q_a$, and the descriptor coordinate for bin $k$ is
\begin{equation}
    v_k
    =
    \sum_{a:\, q_a \in [b_k,b_{k+1}]} A_a .
\end{equation}
Each coordinate measures how much molecular surface area is contributed by atoms whose property value falls into a particular range.

\paragraph{\texttt{SlogP\_VSA}: hydrophobic surface partitioning.}
\texttt{SlogP\_VSA} bins atoms according to their Crippen logP fragment contributions and sums the surface area within each bin.
It provides a surface-weighted description of hydrophobic and lipophilic contributions.

\paragraph{\texttt{SMR\_VSA}: polarizability-related surface partitioning.}
\texttt{SMR\_VSA} bins atoms according to their Crippen molar refractivity contributions.
Molar refractivity is closely related to molecular size and electronic polarizability and is associated with dispersion-related interactions.

\paragraph{\texttt{PEOE\_VSA}: electrostatic surface partitioning.}
\texttt{PEOE\_VSA} bins atoms according to partial charges estimated by the partial equalization of orbital electronegativity procedure.
It summarizes how molecular surface area is distributed across negatively charged, neutral, and positively charged regions.
We therefore interpret \texttt{PEOE\_VSA} as an electrostatic surface prior rather than as a dispersion or polarizability prior.

This distinction is important for interpretation:
\texttt{SMR\_VSA} is associated with polarizability/dispersion-relevant surface information,
\texttt{SlogP\_VSA} with hydrophobicity/lipophilicity-related surface information, and
\texttt{PEOE\_VSA} with partial-charge/electrostatic surface information.

\subsection{Morgan FP as Structural Priors}
\label{app:priors:fp}

To complement the continuous physicochemical VSA priors, we use Morgan FP as a structural prior.
Morgan FP encode local molecular substructures by iteratively expanding atom-centered neighborhoods and hashing the resulting identifiers into a fixed-length bit vector.
Unless otherwise stated, we use radius $r=2$ and 2048 bits.

Unlike VSA descriptors, which summarize surface-weighted physicochemical properties, Morgan FP emphasize graph topology and discrete substructure motifs.
They capture information about ring systems, substituent patterns, local functional groups, and other structural fragments.
Two molecules can be structurally similar under fingerprints while differing in hydrophobic, polarizability, or electrostatic surface distributions, making the fingerprint prior a useful complement to the VSA priors.

\subsection{Pharm2D Pharmacophore Priors}
\label{app:priors:pharm2d}

We additionally evaluate a Pharm2D prior in the generality experiments.
Pharm2D fingerprints encode pharmacophore-level patterns, such as arrangements of hydrogen-bond donors, hydrogen-bond acceptors, aromatic features, hydrophobic features, and ionizable groups.
Compared with Morgan FP, which emphasize local graph substructures, Pharm2D captures a more interaction-oriented view of molecular similarity.

This prior is especially relevant for bioactivity tasks where ligand--target recognition depends on pharmacophore compatibility rather than only on local substructure similarity.
We therefore use Pharm2D in the dynamic prior-selection experiments on HIV, CYP2D6, and KOR (Appendix~\ref{app:extra_ablation:dynamic_selection}).
The purpose is not to claim that Pharm2D is universally better than the default priors, but to test whether PRISM can benefit from a task-relevant alternative when the label-energy diagnostic indicates stronger prior--label alignment.

\subsection{3D Shape Priors}
\label{app:priors:3d}

We also evaluate a compact 3D shape prior in the prior-agnostic and encoder-agnostic experiments.
The 3D prior uses standard RDKit shape descriptors covering complementary aspects of molecular geometry:
\textbf{global shape:} normalized principal moments \texttt{NPR1} and \texttt{NPR2};
\textbf{anisotropy and compactness:} \texttt{Asphericity}, \texttt{Eccentricity}, and \texttt{SpherocityIndex};
\textbf{spatial extent:} \texttt{RadiusOfGyration};
\textbf{principal-axis mass distribution:} \texttt{PMI1}, \texttt{PMI2}, and \texttt{PMI3}.

These descriptors provide a low-dimensional summary of molecular shape, elongation, compactness, and mass distribution.
They are used to test whether PRISM can align learned representations to a geometry-derived prior space.
For the 3D prior, conformers are generated from SMILES using RDKit; molecules for which conformer generation fails are excluded from the corresponding 3D-prior experiment but retained in all 2D-prior experiments.
The compatibility of PRISM with 3D-aware molecular encoders such as SchNet is discussed in Appendix~\ref{app:extra_ablation:schnet}.

\subsection{Distance Metrics and Standardization}
\label{app:priors:distances}

Each prior space is paired with a distance metric matching its representation type.

\paragraph{Continuous descriptor spaces.}
For VSA descriptors and 3D shape descriptors, we apply coordinate-wise standardization within each mini-batch before computing Euclidean distances.
For a mini-batch $B$ and a descriptor vector $\pi^{(s)}(x)$, we compute
\begin{equation}
    \widetilde{\pi}^{(s)}(x)
    =
    \frac{
    \pi^{(s)}(x) - \mu^{(s)}_B
    }{
    \sigma^{(s)}_B + \epsilon
    },
\end{equation}
where $\mu^{(s)}_B$ and $\sigma^{(s)}_B$ are the coordinate-wise mean and standard deviation over the current mini-batch, and $\epsilon$ is a small numerical constant.
The prior distance is then
\begin{equation}
    d_s(x_i,x_j)
    =
    \left\|
    \widetilde{\pi}^{(s)}(x_i)
    -
    \widetilde{\pi}^{(s)}(x_j)
    \right\|_2 .
\end{equation}
This batch-wise standardization prevents descriptor coordinates with large numerical ranges from dominating the neighborhood construction, while preserving relative variation within each descriptor coordinate in the current batch.

\paragraph{Binary or sparse fingerprint spaces.}
For Morgan FP and Pharm2D fingerprints, we use the Tanimoto similarity:
\begin{equation}
    d_s(x_i,x_j)
    =
    1 -
    \mathrm{Tanimoto}
    \left(
    \pi^{(s)}(x_i),
    \pi^{(s)}(x_j)
    \right).
\end{equation}

\paragraph{Learned embedding space.}
For the embedding-space comparator $\delta$, we use cosine similarity after $\ell_2$ normalization, as in the main text.

These choices follow standard practice for representation-consistent comparisons~\cite{bajusz2015tanimoto} and separate the metric used to define prior-space neighborhoods from the metric used to compare learned embeddings.

\subsection{Default and Auxiliary Prior Configurations}
\label{app:priors:configs}

The default PRISM configuration uses four prior branches:
\begin{equation}
    \{\texttt{SMR\_VSA},\ \texttt{SlogP\_VSA},\ \texttt{PEOE\_VSA},\ \texttt{Morgan FP}\}.
\end{equation}
This configuration is used for the main MoleculeNet, pure-solvent, and mixed-solvent experiments unless otherwise specified.

The auxiliary prior configurations are used only in supplementary experiments:
\textbf{Pharm2D prior:} dynamic prior-selection experiments on HIV, CYP2D6, and KOR (Appendix~\ref{app:extra_ablation:dynamic_selection});
\textbf{3D shape prior:} compatibility experiments with 3D backbones and 3D priors (Appendix~\ref{app:extra_ablation:schnet});
\textbf{Gaussian random prior:} negative control for uninformative prior geometry (Appendix~\ref{app:extra_ablation:gaussian}).

This separation avoids conflating two claims.
The default four-prior PRISM tests whether chemically interpretable 2D priors improve molecular representation learning.
The auxiliary prior experiments test whether PRISM is flexible with respect to the choice of prior space.

\subsection{Prior Quality Diagnostics}
\label{app:priors:quality}

Not every prior space is expected to be useful for every downstream task.
We use label energy as a diagnostic for prior--label alignment: a prior has low label energy when nearby molecules in that prior space tend to have similar labels.
The formal definition of raw label energy, the normalized label-energy ratio, and the prediction-theoretic interpretation are given in Appendix~\ref{app:theory:energy} (Proposition~\ref{prop:energy_knn_risk}).
We use label energy in two ways: (i) to compare chemically meaningful priors against a Gaussian random prior, and (ii) to select among candidate priors for bioactivity tasks (Appendix~\ref{app:extra_ablation:dynamic_selection}).

\clearpage
\section{PRISM Implementation Details}
\label{app:impl}

This section provides implementation details for PRISM pretraining and PRISM downstream adaptation.
We describe molecular featurization, mini-batch stochastic-neighborhood construction, locality scaling, the pretraining algorithm, fusion modes, task-specific prediction heads, hyperparameter selection, and computational cost.
Unless otherwise stated, the default PRISM configuration uses four prior-specialized branches corresponding to \texttt{SMR\_VSA}, \texttt{SlogP\_VSA}, \texttt{PEOE\_VSA}, and Morgan FP.

\subsection{D-MPNN Backbone and Molecular Featurization}
\label{app:impl:backbone}

All molecular graphs are constructed from canonical SMILES using RDKit.
Given an input SMILES string, we parse it into an RDKit \texttt{Mol} object and discard molecules that fail sanitization.
We keep RDKit's default perception of aromaticity, valence, and bond order.
Unless otherwise specified, we do not perform additional tautomer enumeration or salt stripping.
The same canonical SMILES strings are used for deduplication, split construction, and leakage control.

We use a directed message passing neural network (D-MPNN)~\cite{yang2019dmpnn} as the default molecular encoder, with hidden dimension 512.
For multi-prior PRISM, each prior space has its own encoder branch $G_{\theta_s}$.
This design allows each branch to specialize in preserving the relational geometry induced by one prior space.

We use the standard 2D atom and bond features adopted by Chemprop~\cite{yang2019dmpnn}: atomic number, total degree, formal charge, chiral tag, total number of attached hydrogens, hybridization, aromaticity, and scaled atomic mass for atoms; bond type, conjugation, ring membership, and stereo for bonds.
The atom feature dimension is 133, and the bond feature dimension is 14.
Full feature definitions are listed in Table~\ref{tab:featurization_summary}.

\begin{table}[h]
\centering
\caption{Atom and bond features used in our D-MPNN featurization. All features are 2D RDKit-derived.}
\label{tab:featurization_summary}
\setlength{\tabcolsep}{6pt}
\small
\begin{tabular}{lllc}
\toprule
\textbf{Level} & \textbf{Feature group} & \textbf{RDKit field} & \textbf{Dim.} \\
\midrule
Atom & Atomic number & \texttt{GetAtomicNum()} shifted by $-1$ & $101$ \\
Atom & Total degree & \texttt{GetTotalDegree()} & $7$ \\
Atom & Formal charge & \texttt{GetFormalCharge()} & $6$ \\
Atom & Chiral tag & \texttt{GetChiralTag()} & $5$ \\
Atom & Total \#H & \texttt{GetTotalNumHs()} & $6$ \\
Atom & Hybridization & \texttt{GetHybridization()} & $6$ \\
Atom & Aromaticity & \texttt{GetIsAromatic()} & $1$ \\
Atom & Scaled atomic mass & \texttt{GetMass()} $\times 0.01$ & $1$ \\
\midrule
Bond & Padding indicator & bond is None & $1$ \\
Bond & Bond type & single/double/triple/aromatic & $4$ \\
Bond & Conjugation & \texttt{GetIsConjugated()} & $1$ \\
Bond & Ring membership & \texttt{IsInRing()} & $1$ \\
Bond & Stereo & \texttt{GetStereo()} one-hot + UNK & $7$ \\
\bottomrule
\end{tabular}
\end{table}

\subsection{Mini-Batch Stochastic Neighborhood Construction}
\label{app:impl:batch}

Because PRISM uses a mini-batch approximation, batch size can affect neighborhood quality.
We tested batch-size sensitivity in a representative \texttt{SlogP\_VSA} setup.
Increasing batch size from 256 to 512 reduces the held-out KL loss after pretraining from 0.0102 to 0.0086, while downstream performance remains essentially unchanged.
We use $B=512$ as the default because it provides a slightly better neighborhood approximation while remaining feasible on a single RTX 3090, with peak memory around 75\%.
We do not claim that 512 is globally optimal; it is an effective, practical choice.

\subsection{Locality Scaling $\beta_s$ and Embedding Temperature $\tau$}
\label{app:impl:scales}

Different prior spaces have different distance scales.
For example, standardized VSA descriptors are continuous vectors, whereas Morgan FP are sparse binary vectors compared by Tanimoto distance.
To make neighborhood sharpness comparable across prior spaces, we scale the locality parameter $\beta_s$ using the median off-diagonal distance within the mini-batch:
\begin{equation}
    \widetilde{d}_s
    =
    \mathrm{median}
    \left\{
        d^{(s)}_{ij}:
        i\neq j,\;
        i,j\in\mathcal{B}
    \right\},
    \qquad
    \beta_s = \frac{1}{\gamma_s\widetilde{d}_s + \epsilon},
\label{eq:impl_beta_scale}
\end{equation}
where $\epsilon=10^{-8}$ is used for numerical stability.
The default setting is $\gamma_s = 0.25$.
This makes $\gamma_s$ a dimensionless locality control: smaller $\gamma_s$ produces a sharper reference neighborhood; larger $\gamma_s$ produces a softer one.
The embedding temperature is $\tau = \frac{1}{\beta_\tau}$.
Per-task hyperparameter values are reported in Section~\ref{app:impl:hyper}.

\subsection{Pretraining Algorithm}
\label{app:impl:pretrain}

Algorithm~\ref{alg:prism_pretrain} summarizes PRISM pretraining for one prior space.
Each prior branch is pretrained independently using its own prior-induced stochastic neighborhoods.
This avoids forcing a single encoder to simultaneously match heterogeneous prior geometries; the architectural justification for separate branches is given in Appendix~\ref{app:extra_ablation:shared_encoder}.

\begin{algorithm}[h]
\caption{PRISM pretraining for one prior space $s$.}
\label{alg:prism_pretrain}
\begin{algorithmic}[1]
  \State Sample mini-batch $\mathcal{B}$ with $|\mathcal{B}|=B$ molecules.
  \State Compute prior-space distances $d^{(s)}_{ij}$ for $i,j\in\mathcal{B}$.
  \State Compute $\widetilde{d}_s \leftarrow \mathrm{median}\{d^{(s)}_{ij}: i\neq j\}$ and set $\beta_s \leftarrow 1 / (\gamma_s\widetilde{d}_s + \epsilon)$.
  \State Compute reference neighborhoods $R_i^{(s)}$ for each anchor $i\in\mathcal{B}$ using Eq.~\eqref{eq:soft_neighborhoods}.
  \State Encode molecules with $G_{\theta_s}$, apply projection head and $\ell_2$ normalization.
  \State Compute learner neighborhoods $P_i^{(s)}$ for each anchor using Eq.~\eqref{eq:soft_neighborhoods}.
  \State Compute $\mathcal{L}^{(s)}_{\mathrm{align}}$ using Eq.~\eqref{eq:align_loss}.
  \State Update $G_{\theta_s}$ and the projection head by AdamW.
\end{algorithmic}
\end{algorithm}

Table~\ref{tab:hp_pretrain} reports the default PRISM pretraining hyperparameters.
Each prior branch is pretrained independently on the same ZINC250k corpus.
Unless otherwise specified, the same optimizer and schedule are used for all four default prior spaces.

\begin{table}[h]
\centering
\small
\caption{Reference hyperparameters for PRISM pretraining.}
\label{tab:hp_pretrain}
\setlength{\tabcolsep}{6pt}
\begin{tabular}{lc}
\toprule
Parameter & Value \\
\midrule
Pretraining corpus & ZINC250k \\
Prior branches & SMR\_VSA, SlogP\_VSA, PEOE\_VSA, Morgan FP \\
Backbone & D-MPNN \\
Encoder hidden dimension & 512 \\
Projection head & 2-layer MLP, 512--64, ReLU \\
Batch size & 512 \\
Optimizer & AdamW \\
Learning rate & $10^{-4}$ \\
Weight decay & $10^{-5}$ \\
Epochs & 100 \\
Embedding temperature $\tau$ & task-selected; default grid in Tables~\ref{tab:tau_gamma_cls}--\ref{tab:tau_gamma_pair} \\
Locality scale $\gamma_s$ & task-selected; default grid in Tables~\ref{tab:tau_gamma_cls}--\ref{tab:tau_gamma_pair} \\
Local scaling constant $\epsilon$ & $10^{-8}$ \\
Checkpoint selection & Lowest validation KL loss \\
\bottomrule
\end{tabular}
\end{table}

\subsection{Projection Heads and Fusion Modes}
\label{app:impl:fusion}

During downstream adaptation, PRISM combines representations from multiple prior-specialized branches.
For molecule $x_i$, each branch produces $r_i^{(s)} = G_{\theta_s}(x_i)$.
Before fusion, each branch embedding is passed through a two-layer projection module:
\begin{equation}
    h_i^{(s)}
    =
    W_2\,\sigma(W_1 r_i^{(s)} + b_1) + b_2
    \in \mathbb{R}^{D},
\end{equation}
where $\sigma$ is ReLU and both the input and output dimensions are $D = 512$.
The branch embeddings are stacked as
\begin{equation}
    H_i=[h_i^{(s_1)};\ldots;h_i^{(s_K)}]\in\mathbb{R}^{K\times D}.
\end{equation}
The aligned branch embeddings are then combined by weighted fusion to produce $z_i$.

We use the term \textbf{weighted fusion} throughout the paper.
This module is not a sparse mixture-of-experts routing: all branches remain active, and the learned weights represent soft preferences over prior spaces.

\paragraph{Task-specific weighted fusion.}
For tasks with global prior-space weights, PRISM learns a sample-independent logit vector
$\ell\in\mathbb{R}^{K}$ by logistic regression.
The branch weights are
\begin{equation}
    w^{(s_k)}
    =
    \frac{\exp(\ell_k)}
    {\sum_{m=1}^{K}\exp(\ell_m)} ,
\end{equation}
and the fused representation is
\begin{equation}
    z_i
    =
    \sum_{k=1}^{K}
    w^{(s_k)} h_i^{(s_k)} .
\end{equation}
This task-specific variant is used for classification and solute--solvent prediction, where a stable task-level combination of prior spaces is empirically more robust.

\paragraph{Input-adaptive weighted fusion.}
For single-molecule regression tasks, PRISM also supports molecule-specific fusion weights.
For each branch embedding $h_i^{(s_k)}$, a small gating MLP produces a branch logit:
\begin{equation}
    a_i^{(s_k)}
    =
    v_2^\top
    \sigma
    \left(
    W_1 h_i^{(s_k)} + b_1
    \right)
    + b_2 ,
\end{equation}
where the gating MLP has dimensions
$D \rightarrow \max(128,D/2) \rightarrow 1$.
We also add a learnable branch-specific bias $b^{(s_k)}$.
The input-adaptive fusion weights are
\begin{equation}
    w_i^{(s_k)}
    =
    \frac{
    \exp\left((a_i^{(s_k)}+b^{(s_k)})/\tau_f\right)
    }{
    \sum_{m=1}^{K}
    \exp\left((a_i^{(s_m)}+b^{(s_m)})/\tau_f\right)
    },
\end{equation}
where $\tau_f = 1$ is a learnable softmax temperature clamped to be positive in implementation.
The fused representation is then
\begin{equation}
    z_i
    =
    \sum_{k=1}^{K}
    w_i^{(s_k)} h_i^{(s_k)} .
\end{equation}
This input-adaptive variant is useful when the dominant physicochemical factor varies across molecules, as in ESOL, FreeSolv, and Lipophilicity.
A comparison of task-specific and input-adaptive fusion is provided in Appendix~\ref{app:extra_ablation:fusion}.

\paragraph{Regularization and branch-level representations.}
After weighted summation, we apply dropout and, when enabled, layer normalization to the fused representation.
The branch embedding stack $H_i\in\mathbb{R}^{K\times D}$ is retained for diagnostic analyses and ablation variants.
Each row of $H_i$ corresponds to one prior-aligned branch, preserving heterogeneous chemical perspectives before they are collapsed by weighted fusion.
The main mixed-solvent head uses only the fused embeddings $z_i$, keeping the evaluation focused on whether frozen PRISM-pretrained encoders can be compositionally reused through lightweight downstream fusion.

\paragraph{Interpretation of fusion weights.}
The default prior spaces are chemically related and partially redundant.
Therefore, fusion weights should be interpreted as soft preference reweighting rather than hard causal assignments.
A larger weight indicates that the corresponding branch is more useful for a task or instance, but it does not imply that other branches are irrelevant.
We provide faithfulness tests for the learned fusion weights in Appendix~\ref{app:interp:faithfulness}.

\subsection{Single-Molecule and Pure-Solvent Heads}
\label{app:impl:heads_single_pure}

\paragraph{Single-molecule prediction.}
For single-molecule tasks, the model first computes the fused molecular representation $z_i$ using either task-specific or input-adaptive weighted fusion.
The prediction is $\hat y_i = f_{\mathrm{single}}(z_i)$, where $f_{\mathrm{single}}$ is an MLP head.
For classification, the head outputs logits and is trained with binary cross-entropy.
For regression, the head outputs a scalar and is trained with mean squared error.

\paragraph{Pure-solvent solubility prediction.}
For pure-solvent solubility, each sample is $(x_{\mathrm{sol}}, x_{\mathrm{solv}}, T, y)$.
The solute and solvent are encoded by the same PRISM encoder:
\begin{equation}
    z_{\mathrm{sol}} = E_\Theta(x_{\mathrm{sol}}),
    \qquad
    z_{\mathrm{solv}} = E_\Theta(x_{\mathrm{solv}}),
    \qquad
    t = h_T(T) \in \mathbb{R}^{d_T},
\end{equation}
where $h_T$ maps standardized scalar temperature to an embedding.
Temperature is standardized using training-set statistics and mapped by a two-layer MLP([1 $\to$ $\frac{d_T}{2}$ $\to$ $d_T$ ]):
\begin{equation}
    h_T(T)=W_2\,\sigma(W_1\,\widetilde{T}+b_1)+b_2 .
\end{equation}
The prediction is $\hat y = f_{\mathrm{pure}}([z_{\mathrm{sol}}; z_{\mathrm{solv}}; t])$ as in the main text.

In our implementation, $d_T = 32$ and $f_{\mathrm{pure}}$ is an MLP with hidden sizes $(1024, 128, 64)$ and dropout 0.2.
We train with mean squared error using Adam.
Early stopping uses patience 50, and the learning-rate scheduler uses patience 30 with decay factor 0.5.

\subsection{Mixed-Solvent Head: Implementation Details}
\label{app:impl:heads_mixed}

For mixed-solvent prediction, we reload and freeze the original PRISM-pretrained encoders rather than the pure-solvent-fine-tuned branches.
Only the weighted-fusion module is initialized from the pure-solvent checkpoint and then trained together with the lightweight mixed-solvent head.
This initialization is used as a warm start for prior-space weighting, while the frozen encoders ensure that no molecular message-passing parameters are updated using mixed-solvent supervision.

Given a solute embedding $z_{\mathrm{sol}}$, two solvent embeddings $z_1,z_2$, mole fractions $\alpha_1,\alpha_2$ with $\alpha_1+\alpha_2=1$, and temperature embedding $t=h_T(T)$, the main mixed-solvent head constructs a composition-weighted solvent representation:
\begin{equation}
    z_{\mathrm{mix}}
    =
    \alpha_1 z_1 + \alpha_2 z_2,
    \qquad
    \hat y
    =
    f_{\mathrm{mix}}
    \left(
    [z_{\mathrm{sol}};z_{\mathrm{mix}};t]
    \right).
\end{equation}
The weighted-fusion module, temperature embedding $h_T$, and prediction head $f_{\mathrm{mix}}$ are trained in the default mixed-solvent protocol.
The PRISM-pretrained encoders remain frozen.

\paragraph{Architecture details.}
Temperature is standardized using training-set statistics and mapped by a two-layer MLP:
\begin{equation}
    h_T(T)=W_2\,\sigma(W_1\,\widetilde{T}+b_1)+b_2 .
\end{equation}
The temperature embedding dimension is $d_T=32$.
The prediction head $f_{\mathrm{mix}}$ is an MLP with two hidden layers of size 512 and dropout 0.1.

\paragraph{Training.}
We train with mean squared error using Adam, learning rate $10^{-4}$, batch size 64, and up to 100 epochs.
The best checkpoint is selected by validation MSE.
We report RMSE and $\%|\Delta\log S|\leq1$ on the test set.

\paragraph{Interpretation.}
This linear composition operation is not intended to fully model all non-ideal solvent--solvent interactions.
Instead, it provides a deliberately lightweight test of representation transfer: if PRISM has transferred chemically meaningful relational structure into the frozen encoder-branch representations, then lightweight fusion and composition-weighted pooling should be sufficient to support generalization to unseen solvent-pair environments.

\subsection{Three-Component Interaction Baselines for Mixed Solvents}
\label{app:triple_interaction_baselines}

CIGIN~\citep{pathak2020chemically} and MMGNN~\citep{du2024mmgnn} were originally developed for interactions between two molecular components.
To construct capability-matched interaction baselines for binary mixed-solvent prediction, we extend their respective interaction principles to a three-component system consisting of a solute $G_s$ and two solvents $G_1$ and $G_2$.
Both models are trained end-to-end on mixed-solvent labels without molecular pretraining.
The solvent composition $\alpha$ and standardized temperature $\widetilde{T}$ are separately mapped to environment embeddings,
\begin{equation}
e_{\mathrm{env}}
=
\left[
f_{\alpha}(\alpha);
f_T(\widetilde{T})
\right],
\end{equation}
which are supplied to the final prediction head.

\paragraph{Triple-CIGIN.}
CIGIN models intermolecular dependence by constructing interactions between atom-level representations of two molecules.
For the three-component setting, we generalize this idea to a joint interaction tensor over the solute and both solvents.
Separate MPNN encoders first produce atom representations
$H_s=\{h_i^s\}$, $H_1=\{h_j^1\}$, and $H_2=\{h_k^2\}$, with the same solvent encoder shared between the two solvent components.
After learned linear projections,
\begin{equation}
q_i^s=W_s h_i^s,\qquad
q_j^1=W_1 h_j^1,\qquad
q_k^2=W_2 h_k^2,
\end{equation}
we define the three-way interaction tensor
\begin{equation}
A_{ijk}
=
\frac{1}{\sqrt{d}}
\sum_{\ell=1}^{d}
q_{i\ell}^{s}q_{j\ell}^{1}q_{k\ell}^{2},
\end{equation}
where $i$, $j$, and $k$ index atoms in the solute, solvent~1, and solvent~2, respectively.

For each component, $A$ is normalized over atoms belonging to the other two components.
For the solute, for example,
\begin{equation}
P^{s}_{ijk}
=
\frac{\exp(A_{ijk})}
{\sum_{j',k'}\exp(A_{ij'k'})},
\qquad
m_i^s
=
\sum_{j,k}
P^{s}_{ijk}
\left(q_j^1\odot q_k^2\right).
\end{equation}
Analogous operations yield $m_j^1$ and $m_k^2$ for the two solvents.
The resulting interaction-enriched atom representations are
\begin{equation}
\widetilde h_i^s=[h_i^s;m_i^s],\qquad
\widetilde h_j^1=[h_j^1;m_j^1],\qquad
\widetilde h_k^2=[h_k^2;m_k^2].
\end{equation}
Each component is independently pooled with Set2Set to obtain
$r_s$, $r_1$, and $r_2$, and the final prediction is
\begin{equation}
\widehat y
=
f_{\mathrm{CIGIN}}
\left(
[r_s;r_1;r_2;e_{\mathrm{env}}]
\right).
\end{equation}

This construction extends the interaction-map principle of CIGIN to a genuine three-component interaction: the representation of each molecule is conditioned jointly on the other two components rather than on independent pairwise predictions.

\paragraph{Triple-MMGNN.}
MMGNN represents intermolecular interactions through a merged graph containing both covalent molecular structure and candidate intermolecular relations.
We retain this merged-graph and interaction-selection principle while using the same MPNN family as the other end-to-end mixed-solvent controls.
For a solute and two solvents, we construct
\begin{equation}
G_{\mathrm{merge}}^{(3)}
=
G_s\cup G_1\cup G_2
\cup E_{s1}^{\mathrm{inter}}
\cup E_{s2}^{\mathrm{inter}}
\cup E_{12}^{\mathrm{inter}},
\end{equation}
where the three intermolecular edge sets connect atoms belonging to different molecular components.
A learned component embedding is appended to each atom feature to identify membership in the solute, solvent~1, or solvent~2.

An initial merged-graph MPNN produces node states $h_v^{(0)}$.
The model then learns a soft node mask
\begin{equation}
m_v
=
\sigma\!\left(
f_{\mathrm{node}}(h_v^{(0)})
\right)
\end{equation}
and a soft edge mask
\begin{equation}
m_{uv}
=
\sigma\!\left(
f_{\mathrm{edge}}
([h_u^{(0)};h_v^{(0)};e_{uv}])
\right).
\end{equation}
Covalent edges are retained with unit weight, whereas the learned edge masks are applied to candidate intermolecular edges.
A second merged-graph MPNN refines the node representations using the masked graph.
The refined nodes are then pooled separately by molecular component,
\begin{equation}
r_s=\operatorname{Pool}(G_s),\qquad
r_1=\operatorname{Pool}(G_1),\qquad
r_2=\operatorname{Pool}(G_2),
\end{equation}
and the prediction is
\begin{equation}
\widehat y
=
f_{\mathrm{MMGNN}}
\left(
[r_s;r_1;r_2;e_{\mathrm{env}}]
\right).
\end{equation}

Following the interaction-selection motivation of MMGNN, the learned node and intermolecular-edge masks are regularized toward a sparse Bernoulli prior during training.
Our implementation should therefore be viewed as an \emph{MMGNN-style} three-component control: it preserves the merged-graph and learned interaction-selection mechanisms while matching the MPNN-based backbone family used by the other end-to-end controls.

\paragraph{Role of the interaction baselines.}
Triple-CIGIN and Triple-MMGNN provide complementary high-capacity controls for explicit downstream interaction modeling.
Triple-CIGIN introduces dense higher-order interaction conditioning while keeping the three molecular graphs separate, whereas Triple-MMGNN allows information exchange through a shared three-component graph with learned interaction selection.
Both are trained end-to-end using mixed-solvent supervision.
Their comparison with PRISM therefore tests whether explicit interaction capacity alone can provide the same generalization to unseen solvent combinations as relationally pretrained molecular representations combined with lightweight downstream composition.

\subsection{Hyperparameter Selection}
\label{app:impl:hyper}

Hyperparameters are selected using validation performance.
For PRISM pretraining, we sweep the embedding temperature $\tau$ and locality scale $\gamma_s$ over a small grid.
For downstream fine-tuning, we tune dropout, learning rate, batch size, fusion mode, and weight decay.

Table~\ref{tab:hp_single} reports reference hyperparameters for single-molecule fine-tuning.
Table~\ref{tab:hp_pair} reports reference hyperparameters for pure-solvent solubility fine-tuning.
Table~\ref{tab:hp_mixed} reports reference hyperparameters for mixed-solvent solubility fine-tuning.

\begin{table}[!ht]
\centering
\small
\caption{Reference hyperparameters for single-molecule fine-tuning.}
\label{tab:hp_single}
\setlength{\tabcolsep}{5pt}
\begin{tabular}{lcccccc}
\toprule
 & ESOL & FreeSolv & Lipo & BBBP & BACE & HIV \\
\midrule
Pretrain $\tau$ & 0.1 & 0.3 & 0.1 & 0.1 & 0.2 & 0.1 \\
Pretrain $\gamma_s$ & 0.75 & 0.75 & 0.25 & 0.25 & 0.75 & 0.25 \\
Fusion mode & Adaptive & Adaptive & Adaptive & Static & Static & Static \\
Dropout & 0.3 & 0.3 & 0.3 & 0.1 & 0.1 & 0.1 \\
Batch size & 64 & 64 & 64 & 64 & 16 & 16 \\
Weight decay & $10^{-5}$ & 0 & 0 & 0 & $10^{-5}$ & $10^{-5}$ \\
Epochs & 200 & 200 & 200 & 200 & 100 & 200 \\
\bottomrule
\end{tabular}
\end{table}

\begin{table}[!ht]
\centering
\small
\caption{Reference hyperparameters for pure-solvent solubility fine-tuning.}
\label{tab:hp_pair}
\setlength{\tabcolsep}{6pt}
\begin{tabular}{lc}
\toprule
Parameter & Value \\
\midrule
Pretrain $\tau$ & 0.1 \\
Pretrain $\gamma_s$ & 0.25 \\
Fusion mode & Static \\
Embedding dimension $D$ & 512 \\
Temperature embedding $d_T$ & 32 \\
Batch size & 64 \\
Validation ratio & 0.05 \\
Epochs & 200 \\
Weight decay & $10^{-5}$ \\
MLP head hidden sizes & $(1024, 128, 64)$ \\
MLP head dropout & 0.2 \\
Encoder LR / head LR & $10^{-4}$ / $10^{-3}$ \\
Early stopping patience & 50 \\
LR schedule patience / factor & 30 / 0.5 \\
\bottomrule
\end{tabular}
\end{table}

\begin{table}[!ht]
\centering
\small
\caption{Reference hyperparameters for mixed-solvent solubility fine-tuning.}
\label{tab:hp_mixed}
\setlength{\tabcolsep}{6pt}
\begin{tabular}{lc}
\toprule
Parameter & Value \\
\midrule
Embedding dimension $D$ & 512 \\
Temperature embedding $d_T$ & 32 \\
Fusion mode & Static weighted fusion \\
encoders & Original PRISM-pretrained branches, frozen \\
Fusion module & Initialized from pure-solvent checkpoint and trainable \\
Mixed head & Linear composition-weighted MLP \\
Trainable modules & Fusion module, temperature embedding, and prediction head \\
MLP head hidden sizes & $(512,512)$ \\
MLP head dropout & 0.1 \\
Batch size & 64 \\
Optimizer & Adam \\
Learning rate & $10^{-4}$ \\
Epochs & 100 \\
Checkpoint selection & Validation MSE \\
\bottomrule
\end{tabular}
\end{table}

\subsubsection{Sensitivity to pre-training hyperparameters ($\tau$ and $\gamma_s$)}
\label{app:sensitivity_tau_gamma}

We study sensitivity to embedding temperature $\tau$ and the locality scale constant $\gamma_s$.
We sweep $\tau\in\{0.1,0.2,0.3\}$ and $\gamma_s\in\{0.25,0.50,0.75\}$ and report downstream test performance.

\paragraph{Discussion (overall).}
Across tasks, performance varies smoothly over the grid without brittle behavior, suggesting that the method is reasonably robust. Best settings are mildly task-dependent. All pre-training and downstream hyperparameters are selected via validation set performance, and the chosen configurations for each task are fully reported.

\paragraph{MoleculeNet classification.}
Table~\ref{tab:tau_gamma_cls} indicates that BBBP prefers a sharper locality, while BACE benefits from smoother scaling.
This is consistent with task-dependent neighborhood granularity: some tasks reward emphasizing only very close prior neighbors, while others benefit from broader neighborhoods.

\paragraph{MoleculeNet regression.}
Table~\ref{tab:tau_gamma_reg} shows mild-to-moderate variation across the grid with no universally optimal choice.
In practice we adopt a simple default that performs consistently and avoid per-task over-tuning.

\paragraph{Pairwise solubility (pure-solvent).}
Table~\ref{tab:tau_gamma_pair} indicates that performance remains stable across reasonable settings.
In particular, $\tau{=}0.1,\gamma_s{=}0.25$ achieves the best Leeds performance, while SolProp is relatively insensitive.
This supports using a fixed pre-training setting for the main solubility comparisons.

\begin{table}[h]
\centering
\small
\setlength{\tabcolsep}{7pt}
\caption{\textbf{Hyperparameter sensitivity on classification tasks.}
Test ROC-AUC ($\uparrow$) under different pre-training settings of embedding temperature $\tau$ and locality scale $\gamma_s$.
HIV is omitted due to its much larger scale.}
\label{tab:tau_gamma_cls}
\begin{tabular}{cc|cc}
\toprule
$\tau$ & $\gamma_s$ & BBBP & BACE \\
\midrule
0.1 & 0.25 & \textbf{0.9459} & 0.8389 \\
0.1 & 0.50 & 0.8997 & 0.8548 \\
0.1 & 0.75 & 0.9006 & 0.8227 \\
\midrule
0.2 & 0.25 & 0.9017 & 0.8510 \\
0.2 & 0.50 & 0.9071 & 0.8480 \\
0.2 & 0.75 & 0.9085 & \textbf{0.8946} \\
\midrule
0.3 & 0.25 & 0.9036 & 0.8298 \\
0.3 & 0.50 & 0.9044 & 0.8451 \\
0.3 & 0.75 & 0.8994 & 0.8474 \\
\bottomrule
\end{tabular}
\vspace{-2mm}
\end{table}

\begin{table}[h]
\centering
\small
\setlength{\tabcolsep}{7pt}
\caption{\textbf{Hyperparameter sensitivity on regression tasks.}
Test RMSE ($\downarrow$) under different pre-training settings of $\tau$ and $\gamma_s$.}
\label{tab:tau_gamma_reg}
\begin{tabular}{cc|ccc}
\toprule
$\tau$ & $\gamma_s$ & ESOL & FreeSolv & Lipo \\
\midrule
0.1 & 0.25 & 0.874 & 1.835 & \textbf{0.538} \\
0.1 & 0.50 & 0.878 & 1.831 & 0.673 \\
0.1 & 0.75 & \textbf{0.772} & 1.812 & 0.678 \\
\midrule
0.2 & 0.25 & 0.863 & 1.820 & 0.681 \\
0.2 & 0.50 & 0.925 & 2.067 & 0.656 \\
0.2 & 0.75 & 0.848 & 1.828 & 0.679 \\
\midrule
0.3 & 0.25 & 0.850 & 1.789 & 0.688 \\
0.3 & 0.50 & 0.864 & 1.834 & 0.689 \\
0.3 & 0.75 & 0.860 & \textbf{1.770} & 0.685 \\
\bottomrule
\end{tabular}
\vspace{-2mm}
\end{table}

\begin{table}[h]
\centering
\small
\setlength{\tabcolsep}{6pt}
\caption{\textbf{Hyperparameter sensitivity on pure-solvent pairwise solubility.}
We report best validation RMSE, and test RMSE / $\%logS \pm 1$ ($\uparrow$) on Leeds and SolProp.
$\%logS \pm 1$ is shown in \%. Only configurations available in our logs are listed.}
\label{tab:tau_gamma_pair}
\begin{tabular}{cc|c|cc|cc}
\toprule
$\tau$ & $\gamma_s$ & best val RMSE $\downarrow$ &
Leeds RMSE $\downarrow$ & Leeds $\%logS \pm 1$ $\uparrow$ &
SolProp RMSE $\downarrow$ & SolProp $\%logS \pm 1$ $\uparrow$ \\
\midrule
0.1 & 0.50 & 0.1414 & 0.9340 & 73.65 & 0.7940 & 80.71 \\
0.1 & 0.25 & \textbf{0.1398} & \textbf{0.9099} & \textbf{75.66} & 0.7719 & 82.29 \\
0.2 & 0.25 & 0.1393 & 0.9352 & 74.81 & 0.7853 & 81.72 \\
0.2 & 0.75 & 0.1429 & 0.9179 & 74.65 & 0.7739 & \textbf{82.87} \\
0.3 & 0.25 & 0.1438 & 0.9331 & 74.03 & \textbf{0.7563} & 82.73 \\
\bottomrule
\end{tabular}
\vspace{-2mm}
\end{table}

\subsection{Compute Resources and Runtime}
\label{app:impl:compute}
For \(K\) prior spaces and batch size \(B\), the minibatch relational construction scales as \(O(KB^2)\) in pairwise comparisons.
Table~\ref{tab:cost_summary} summarizes the compute resources used in our experiments.
All reported times are wall-clock times per run.
Pretraining is the most expensive stage because it requires computing dense mini-batch pairwise neighborhood distributions.
Downstream fine-tuning is substantially cheaper.

\begin{table}[h]
\centering
\caption{Computational cost and hardware. Reported times are wall-clock per run.}
\label{tab:cost_summary}
\setlength{\tabcolsep}{6pt}
\small
\begin{tabular}{p{3.8cm}p{3.0cm}p{3.0cm}p{3.0cm}}
\toprule
\textbf{Stage} & \textbf{Hardware} & \textbf{Time per run} & \textbf{Notes} \\
\midrule
PRISM pretraining per branch & RTX 3090 & 7--8 hours & Batch size 512; $\sim$75\% peak GPU memory \\
Pure-solvent fine-tuning & RTX 3090 & $\sim$6 hours & End-to-end \\
MoleculeNet fine-tuning & RTX 4070 & $\leq$1 hour & Per task, per seed \\
Mixed-solvent training & RTX 4070 mobile & 3 hours & Frozen encoders; trainable fusion and head \\
\bottomrule
\end{tabular}
\end{table}

The complete pipeline is feasible on a single consumer-grade GPU.
The multi-branch design increases pretraining cost because each prior space has its own encoder, but the branches can be pretrained independently and in parallel.
At downstream time, the additional cost is moderate relative to the molecular encoder forward passes and remains practical for the datasets studied here.

{
At fixed batch size, PRISM scales linearly with the pretraining-corpus size, while increasing the number of priors increases total work approximately linearly because prior-specific branches are trained independently and can be parallelized.
In our profiling, D-MPNN encoding and backpropagation dominate runtime and memory, whereas the minibatch relational matrices add relatively little overhead.
Restricting relational matching to sampled candidates reduces matching cost but yields only a 2--4\% runtime reduction at the current scale, since all molecules still require graph encoding.
Larger-scale implementations may therefore benefit more from reducing or reusing molecular encoding than from approximating the relational matrix alone.
}

\clearpage
\section{Theoretical Analysis}
\label{app:theory}

This section provides theoretical support for the main objective-level claims used in the paper.

\subsection{Alignment Objectives in the PRISM}
\label{app:loss_definitions}

This appendix gives the formal definitions of the four alignment objectives compared.
Throughout, $\mathcal{B}$ denotes a mini-batch, $d_{ij}^{(s)} = d_s(\pi_i^{(s)}, \pi_j^{(s)})$ is the prior-space distance between molecules $i$ and $j$ under prior space $s$, and $r_i^{(s)} = G_{\theta_s}(x_i)$ is the corresponding branch embedding.
We write $\mathrm{sim}(\cdot,\cdot)$ for cosine similarity after $\ell_2$ normalization, and define the raw embedding-space distance as
\begin{equation}
    \delta_{ij}^{(s)}
    =
    1-\mathrm{sim}(r_i^{(s)},r_j^{(s)}).
\end{equation}

\subsection{Loss families}
\label{app:loss}

\subsubsection{Pointwise metric matching ($\ell_2$)}
\label{app:loss_l2}

Pointwise metric matching directly regresses pairwise embedding distances onto pairwise prior-space distances.
Because prior-space distances and embedding-space distances can have different numerical scales, both distances are rescaled within each mini-batch before applying the regression loss.
Let
\begin{equation}
    m_d^{(s)}
    =
    \mathrm{median}
    \left\{
        d_{ij}^{(s)}:
        i\neq j,\;
        i,j\in\mathcal{B}
    \right\},
    \qquad
    m_\delta^{(s)}
    =
    \mathrm{median}
    \left\{
        \delta_{ij}^{(s)}:
        i\neq j,\;
        i,j\in\mathcal{B}
    \right\}.
\end{equation}
The scaled distances are
\begin{equation}
    \tilde d_{ij}^{(s)}
    =
    \frac{d_{ij}^{(s)}}{m_d^{(s)}+\epsilon},
    \qquad
    \tilde\delta_{ij}^{(s)}
    =
    \frac{\delta_{ij}^{(s)}}{m_\delta^{(s)}+\epsilon}
    =
    \frac{1-\mathrm{sim}(r_i^{(s)},r_j^{(s)})}
    {m_\delta^{(s)}+\epsilon},
\end{equation}
where $\epsilon$ is a small numerical constant.
The loss is
\begin{equation}
\mathcal{L}^{(s)}_{\ell_2}
=
\frac{1}{|\mathcal{B}|(|\mathcal{B}|-1)}
\sum_{i\in\mathcal{B}}
\sum_{\substack{j\in\mathcal{B}\\ j\neq i}}
\Big(
\tilde\delta_{ij}^{(s)} - \tilde d_{ij}^{(s)}
\Big)^2 .
\label{eq:l2_loss}
\end{equation}
This objective is the most rigid in the alignment family: it requires the embedding to reproduce descriptor distances pointwise, including their magnitudes.
When the prior space does not exhaust the information in the molecular graph, this rigidity over-constrains the encoder.

\subsubsection{Rank preservation}
\label{app:loss_rank}

Rank-based alignment preserves only the ordinal structure of prior-space neighborhoods, discarding distance magnitude.
For a fixed anchor $i$, each unordered neighbor pair $\{j,k\}\subset\mathcal{B}\setminus\{i\}$ contributes one comparison.
To avoid imposing an artificial order when the prior space gives equal distances, we define the set of strict unordered comparisons as
\begin{equation}
    \mathcal{T}_i^{(s)}
    =
    \left\{
    (j,k):
    j,k\in\mathcal{B}\setminus\{i\},\;
    j<k,\;
    d_{ij}^{(s)}\neq d_{ik}^{(s)}
    \right\}.
\end{equation}
Anchors with $|\mathcal{T}_i^{(s)}|=0$ are skipped; this case is rare in continuous descriptor spaces and can occur only when no strict prior ordering is available for that anchor.

The non-differentiable rank disagreement rate is
\begin{equation}
\mathcal{L}^{(s)}_{\mathrm{rank},0}
=
\frac{1}{|\mathcal{I}^{(s)}|}
\sum_{i\in\mathcal{I}^{(s)}}
\frac{1}{|\mathcal{T}_i^{(s)}|}
\sum_{(j,k)\in\mathcal{T}_i^{(s)}}
\mathbf{1}\!\left[
\mathrm{sgn}(d_{ij}^{(s)} - d_{ik}^{(s)})
\neq
\mathrm{sgn}(\delta_{ij}^{(s)} - \delta_{ik}^{(s)})
\right],
\label{eq:rank_loss_indicator}
\end{equation}
where
\[
    \mathcal{I}^{(s)}
    =
    \{i\in\mathcal{B}:|\mathcal{T}_i^{(s)}|>0\}.
\]
When there are no tied prior-space distances, this is equivalent to averaging over
$\binom{|\mathcal{B}|-1}{2}$ unordered neighbor pairs per anchor.
The loss equals zero when the embedding ranking agrees with the prior-space ranking for every unordered comparison.
Thus, unlike an ordered-pair double sum, Eq.~\eqref{eq:rank_loss_indicator} cannot produce negative loss values under perfect agreement.

For optimization, we use a smooth logistic surrogate.
For each $(j,k)\in\mathcal{T}_i^{(s)}$, define the prior-order label
\begin{equation}
    y_{ijk}^{(s)}
    =
    \mathrm{sgn}\!\left(d_{ik}^{(s)}-d_{ij}^{(s)}\right)
    \in\{-1,+1\}.
\end{equation}
Here $y_{ijk}^{(s)}=+1$ means that $j$ is closer than $k$ to anchor $i$ in prior space.
The corresponding embedding-space score difference is
\begin{equation}
    s_{ijk}^{(s)}
    =
    \delta_{ik}^{(s)}-\delta_{ij}^{(s)}.
\end{equation}
Thus, $y_{ijk}^{(s)}s_{ijk}^{(s)}>0$ means that the embedding-space ordering agrees with the prior-space ordering.
The differentiable rank loss is
\begin{equation}
\mathcal{L}^{(s)}_{\mathrm{rank}}
=
\frac{1}{|\mathcal{I}^{(s)}|}
\sum_{i\in\mathcal{I}^{(s)}}
\frac{1}{|\mathcal{T}_i^{(s)}|}
\sum_{(j,k)\in\mathcal{T}_i^{(s)}}
\log\left(
1+\exp\left[-y_{ijk}^{(s)}s_{ijk}^{(s)}/\tau_r\right]
\right),
\label{eq:rank_loss_smooth}
\end{equation}
where $\tau_r>0$ is a rank-temperature parameter.

\subsubsection{Neighborhood-graph alignment (kNN + InfoNCE)}
\label{app:loss_nng}

Neighborhood-graph alignment first uses prior-space kNN to determine positive and negative samples, then applies a contrastive objective.
For each anchor $i\in\mathcal{B}$, we define the positive set
\begin{equation}
\mathcal{P}_i^{(s)}
=
\{\,j\in\mathcal{B}\setminus\{i\}: j\in\mathcal{N}_{k_{\mathrm{NN}}}^{(s)}(i)\,\},
\label{eq:nng_positive_set}
\end{equation}
i.e., the top-$k_{\mathrm{NN}}$ prior-space neighbors of $i$ within the batch, and use all remaining in-batch molecules as negatives.
We then apply the contrastive form used by MolCLR~\citep{molclr}:
\begin{equation}
\mathcal{L}^{(s)}_{\mathrm{NNG}}
=
\frac{1}{|\mathcal{B}|}
\sum_{i\in\mathcal{B}}
\frac{-1}{|\mathcal{P}_i^{(s)}|}
\sum_{j\in\mathcal{P}_i^{(s)}}
\log
\frac{\exp\!\left(\mathrm{sim}(r_i^{(s)},r_j^{(s)})/\tau\right)}
{\sum_{\ell\in\mathcal{B}\setminus\{i\}}\exp\!\left(\mathrm{sim}(r_i^{(s)},r_\ell^{(s)})/\tau\right)} ,
\label{eq:nng_loss}
\end{equation}
where $\ell$ is a dummy summation index in the denominator, and $\tau$ is the contrastive temperature.
The objective treats the kNN graph as a binary structure---each pair is either ``positive'' or ``negative''---and discards the gradation of similarity within the neighbor set.
When $|\mathcal{P}_i^{(s)}|=1$, Eq.~\eqref{eq:nng_loss} reduces to standard InfoNCE with a single positive.

\subsubsection{Distribution matching (default)}
\label{app:loss_kl}

The soft-neighborhood KL instantiation, used as the default in PRISM, defines anchor-wise stochastic neighborhoods on both sides and aligns them via KL divergence:
\begin{equation}
R_i^{(s)}(j)
=
\frac{\exp(-\beta_s d_{ij}^{(s)})}
{\sum_{\ell\in\mathcal{B}\setminus\{i\}}\exp(-\beta_s d_{i\ell}^{(s)})},
\qquad
P_i^{(s)}(j)
=
\frac{\exp(\mathrm{sim}(r_i^{(s)},r_j^{(s)})/\tau)}
{\sum_{\ell\in\mathcal{B}\setminus\{i\}}\exp(\mathrm{sim}(r_i^{(s)},r_\ell^{(s)})/\tau)} .
\label{eq:soft_neighborhoods_app}
\end{equation}
\begin{equation}
\mathcal{L}^{(s)}_{\mathrm{KL}}
=
\frac{1}{|\mathcal{B}|}
\sum_{i\in\mathcal{B}}
\sum_{j\in\mathcal{B}\setminus\{i\}}
R_i^{(s)}(j)
\log\frac{R_i^{(s)}(j)}{P_i^{(s)}(j)} .
\label{eq:kl_loss_app}
\end{equation}
Unlike the hard NNG instantiation, the soft form does not require a binary positive/negative partition: every in-batch molecule contributes to the anchor's neighborhood with a graded weight determined by its prior-space distance.
This preserves both neighbor identity (which molecules are close) and similarity magnitude (how close they are).

\paragraph{Conditional asymmetry and symmetric SNE alternative.}
\label{app:symmetric_sne}

The distributions in Eq.~(52) are conditional on the anchor molecule and
are therefore generally asymmetric: $R_i^{(s)}(j)$ need not equal
$R_j^{(s)}(i)$, even though the underlying distance
$d_{ij}^{(s)}=d_{ji}^{(s)}$ is symmetric.
This asymmetry arises from anchor-specific normalization and retains the
local neighborhood context of each molecule.

To test whether this choice is important, we additionally consider a
symmetric SNE-style formulation. Following symmetric SNE, we define
\[
\bar{R}_{ij}^{(s)}
=
\frac{R_i^{(s)}(j)+R_j^{(s)}(i)}
     {2|\mathcal{B}|},
\qquad
\bar{P}_{ij}^{(s)}
=
\frac{P_i^{(s)}(j)+P_j^{(s)}(i)}
     {2|\mathcal{B}|},
\]
and minimize
\[
\mathcal{L}_{\mathrm{sym}}^{(s)}
=
\sum_{i\neq j}
\bar{R}_{ij}^{(s)}
\log
\frac{\bar{R}_{ij}^{(s)}}{\bar{P}_{ij}^{(s)}}.
\]

{
Both formulations transfer useful relational information, but the
conditional objective performs consistently better in our matched
comparison. We interpret this result as evidence that retaining
anchor-specific neighborhood context is useful for molecular
representation learning, whereas symmetrization averages the two
conditional neighborhood views.
}

\begin{table}[H]
\centering
\caption{
Comparison of conditional and symmetric SNE-style alignment under matched
settings. Results are mean $\pm$ std over three seeds. RMSE ($\downarrow$).
}
\label{tab:symmetric_sne}
\begin{tabular}{lccc}
\toprule
Objective & Lipo & FreeSolv & ESOL \\
\midrule
Symmetric SNE
& $0.592 \pm 0.009$
& $2.146 \pm 0.079$
& $0.803 \pm 0.032$ \\
PRISM conditional KL
& $0.538 \pm 0.018$
& $1.770 \pm 0.093$
& $0.772 \pm 0.002$ \\
\bottomrule
\end{tabular}
\end{table}

\paragraph{Connection to InfoNCE.}
\label{app:theory:infonce}

We next relate the PRISM stochastic-neighborhood alignment objective to InfoNCE.
For an anchor $i$ in prior space $s$, PRISM aligns the prior-induced soft neighborhood $R_i^{(s)}$ to the embedding-induced soft neighborhood $P_i^{(s)}$ via KL divergence.
The KL divergence decomposes as
\begin{equation}
    \mathrm{KL}(R_i^{(s)} \,\|\, P_i^{(s)})
    =
    H(R_i^{(s)}, P_i^{(s)})
    -
    H(R_i^{(s)}),
\label{eq:kl_decomp}
\end{equation}
where $H(p, q) = -\sum_j p_j \log q_j$ and $H(p) = -\sum_j p_j \log p_j$.
Since $H(R_i^{(s)})$ does not depend on encoder parameters, minimizing the KL divergence is equivalent to minimizing the soft cross-entropy from $R_i^{(s)}$ to $P_i^{(s)}$.

In the hard-neighborhood special case where the prior distribution assigns all mass to one positive $j_i^+$, $H(R_i^{(s)})=0$ and Eq.~\eqref{eq:kl_decomp} reduces to $-\log P_i^{(s)}(j_i^+)$.
With a softmax over scaled cosine similarities,
\begin{equation}
    -\log P_i^{(s)}(j_i^+)
    =
    -\log
    \frac{
        \exp(\mathrm{sim}(r_i^{(s)}, r_{j_i^+}^{(s)})/\tau)
    }{
        \sum_{\ell\neq i}
        \exp(\mathrm{sim}(r_i^{(s)}, r_\ell^{(s)})/\tau)
    },
\label{eq:hard_target_infonce}
\end{equation}
which is the InfoNCE loss with $j_i^+$ as the positive and the remaining in-batch molecules as negatives.

Thus, InfoNCE corresponds to a hard-target specialization of PRISM's soft-neighborhood objective, whereas PRISM uses the soft distribution form natively to reflect graded chemical similarity.
The difference is that PRISM does not require selecting a single positive example.
Instead, the prior space provides a soft neighborhood distribution over all in-batch molecules, allowing multiple chemically related molecules to contribute with graded weights and avoiding a binary positive/negative partition when molecular similarity is continuous.

\subsubsection{TDL as a Rank-Projected Version of PRISM-KL}
\label{app:theory:tdl_rank_projection}

We next relate rank-style TDL/ListMLE objectives to PRISM-KL.
Fix an anchor molecule $i$ and let $\mathcal{C}_i=\mathcal{B}\setminus\{i\}$ with $|\mathcal{C}_i|=M$.
For prior space $s$, PRISM defines a prior-induced soft neighborhood distribution
\[
R_i^{(s)}(j)
=
\frac{\exp(-\beta_s d_{ij}^{(s)})}
{\sum_{\ell\in\mathcal{C}_i}\exp(-\beta_s d_{i\ell}^{(s)})},
\qquad j\in\mathcal{C}_i,
\]
and an embedding-induced distribution
\[
P_i^{(s)}(j)
=
\frac{\exp(u_{ij}^{(s)})}
{\sum_{\ell\in\mathcal{C}_i}\exp(u_{i\ell}^{(s)})},
\qquad
u_{ij}^{(s)}=\mathrm{sim}(r_i^{(s)},r_j^{(s)})/\tau .
\]
PRISM minimizes
\[
\mathcal{L}^{(s)}_{\mathrm{PRISM}}(i)
=
\mathrm{KL}\!\left(R_i^{(s)}\middle\|P_i^{(s)}\right).
\]
If the alignment is exact, $P_i^{(s)}=R_i^{(s)}$, then for any $j,k\in\mathcal{C}_i$,
\begin{equation}
\frac{P_i^{(s)}(j)}{P_i^{(s)}(k)}
=
\frac{R_i^{(s)}(j)}{R_i^{(s)}(k)}
=
\exp\!\left[-\beta_s\left(d_{ij}^{(s)}-d_{ik}^{(s)}\right)\right].
\end{equation}
Since the softmax normalizer cancels in the probability ratio, we have
\begin{equation}
    u_{ij}^{(s)}-u_{ik}^{(s)}
    =
    -\beta_s\left(d_{ij}^{(s)}-d_{ik}^{(s)}\right).
\end{equation}
Thus, PRISM-KL not only preserves the prior-induced ordering
\begin{equation}
    d_{ij}^{(s)}<d_{ik}^{(s)}
    \quad\Longrightarrow\quad
    u_{ij}^{(s)}>u_{ik}^{(s)},
\end{equation}
but also preserves the magnitude of the distance gap through the log-probability ratio.

TDL keeps only the rank projection of this condition.
Let $\pi_1,\ldots,\pi_M$ be the candidates sorted by increasing prior distance,
\[
d_{i\pi_1}^{(s)}<d_{i\pi_2}^{(s)}<\cdots<d_{i\pi_M}^{(s)},
\]
assuming a strict order for simplicity.
For each position $t$, define the suffix set
\[
S_t=\{\pi_t,\pi_{t+1},\ldots,\pi_M\}
=
\{m\in\mathcal{C}_i: d_{im}^{(s)}\ge d_{i\pi_t}^{(s)}\}.
\]
On this suffix set, define the truncated embedding distribution
\[
P_i^{(t)}(m)
=
\frac{\exp(u_{im}^{(s)})}
{\sum_{\ell\in S_t}\exp(u_{i\ell}^{(s)})},
\qquad m\in S_t,
\]
and a hard target $Q_i^{(t)}(m)=\mathbf{1}[m=\pi_t]$.
Applying hard-target KL on every suffix gives
\begin{equation}
\sum_{t=1}^{M}
\mathrm{KL}\!\left(Q_i^{(t)}\middle\|P_i^{(t)}\right)
=
\sum_{t=1}^{M}
-\log
\frac{\exp(u_{i\pi_t}^{(s)})}
{\sum_{\ell=t}^{M}\exp(u_{i\pi_\ell}^{(s)})}.
\end{equation}
Re-indexing $m=\pi_t$, this becomes
\begin{equation}
\mathcal{L}_{\mathrm{TDL}}(i)
=
\sum_{m\in\mathcal{C}_i}
-\log
\frac{\exp(u_{im}^{(s)})}
{\sum_{\ell\in\mathcal{C}_i}
\mathbf{1}[d_{i\ell}^{(s)}\ge d_{im}^{(s)}]\exp(u_{i\ell}^{(s)})},
\label{eq:tdl_listmle}
\end{equation}
which is the Plackett--Luce/ListMLE form of TDL.
Therefore, TDL can be viewed as first projecting the PRISM target distribution $R_i^{(s)}$ to its rank ordering and then applying sequential hard-target KL over suffix sets.
It matches $\mathrm{rank}(R_i^{(s)})$, whereas PRISM matches the full soft neighborhood distribution $R_i^{(s)}$ itself.

\paragraph{Controlled empirical comparison.}
To complement the analysis above, we compare the TDL-style objective
with PRISM's conditional KL objective while holding the prior
configuration, D-MPNN backbone, ZINC250k pretraining corpus, downstream
model, and evaluation protocol fixed.
This is an objective-level comparison rather than a full reproduction
of the original method.
Table~\ref{tab:tdl_kl_comparison} reports the results.

\begin{table}[t]
\centering
\small
\setlength{\tabcolsep}{8pt}
\caption{\textbf{Controlled comparison of TDL-style and conditional KL
alignment.}
Only the alignment objective is changed.
Results are mean $\pm$ standard deviation over three seeds.
Lower RMSE is better.}
\label{tab:tdl_kl_comparison}
\begin{tabular}{lccc}
\toprule
Objective & Lipo RMSE $\downarrow$
          & FreeSolv RMSE $\downarrow$
          & ESOL RMSE $\downarrow$ \\
\midrule
TDL-style alignment
& $0.622 \pm 0.011$
& $2.221 \pm 0.028$
& $0.819 \pm 0.003$ \\
PRISM (conditional KL)
& $\mathbf{0.538 \pm 0.018}$
& $\mathbf{1.770 \pm 0.093}$
& $\mathbf{0.772 \pm 0.002}$ \\
\bottomrule
\end{tabular}
\end{table}

\paragraph{Discussion.}
Conditional KL achieves lower downstream RMSE on all three tasks.
Together with the analysis above, these results support preserving
graded prior-induced similarities rather than their ordering alone
under this matched comparison, providing additional empirical
justification for PRISM's default alignment objective.

\subsubsection{Shared-Encoder Multi-Prior KL as Early-Fused Target Matching}
\label{app:theory:early_fusion}

We show that a shared-encoder multi-prior objective, where all prior spaces are optimized against the same embedding-induced distribution $P_i$, is mathematically equivalent to GraphMSL-style early fusion of target similarity distributions.

Fix an anchor molecule $i$, and let $\mathcal{C}_i=\mathcal{B}\setminus\{i\}$ be the in-batch candidate set.
Suppose that $K$ prior spaces define $K$ target neighborhood distributions
\[
R_i^{(1)},R_i^{(2)},\ldots,R_i^{(K)},
\qquad
\sum_{j\in\mathcal{C}_i}R_i^{(s)}(j)=1
\quad
\text{for all }s.
\]
In the shared-encoder control, all prior losses are optimized against the same embedding-induced distribution
\[
P_i(j)
=
\frac{\exp(u_{ij})}
{\sum_{\ell\in\mathcal{C}_i}\exp(u_{i\ell})},
\]
where $u_{ij}$ is the shared embedding logit.
The unweighted shared-$P_i$ multi-prior objective is
\begin{equation}
\mathcal{L}_{\mathrm{shared}}(i)
=
\sum_{s=1}^{K}
\mathrm{KL}\!\left(R_i^{(s)}\middle\|P_i\right).
\label{eq:shared_multi_prior}
\end{equation}
Define the normalized early-fused target distribution as the average of the prior-induced distributions:
\begin{equation}
\bar R_i(j)
=
\frac{1}{K}\sum_{s=1}^{K}R_i^{(s)}(j).
\end{equation}
Since each $R_i^{(s)}$ is a probability distribution, $\bar R_i$ is also a probability distribution.

Expanding the shared objective gives
\begin{align}
\mathcal{L}_{\mathrm{shared}}(i)
&=
\sum_{s=1}^{K}
\sum_{j\in\mathcal{C}_i}
R_i^{(s)}(j)
\log
\frac{R_i^{(s)}(j)}{P_i(j)}
\nonumber \\
&=
\sum_{s=1}^{K}
\sum_j
R_i^{(s)}(j)\log R_i^{(s)}(j)
-
\sum_{s=1}^{K}
\sum_j
R_i^{(s)}(j)\log P_i(j).
\end{align}
Using
\[
\sum_{s=1}^{K}R_i^{(s)}(j)=K\bar R_i(j),
\]
the second term becomes
\[
-
K\sum_j
\bar R_i(j)\log P_i(j).
\]
Therefore,
\begin{equation}
\mathcal{L}_{\mathrm{shared}}(i)
=
\sum_{s=1}^{K}
\sum_j
R_i^{(s)}(j)\log R_i^{(s)}(j)
-
K\sum_j
\bar R_i(j)\log P_i(j).
\label{eq:shared_expanded}
\end{equation}

Now consider the GraphMSL-style early-fusion objective.
In this form, the multiple target distributions are first fused into a single target distribution
\[
T_i(j)=\bar R_i(j),
\]
and the model learns one embedding-induced distribution
\[
Q_i(j)=P_i(j).
\]
The early-fusion cross-entropy objective is
\begin{equation}
\mathcal{L}_{\mathrm{early}}(i)
=
-\sum_j T_i(j)\log Q_i(j)
=
-\sum_j \bar R_i(j)\log P_i(j).
\end{equation}
Comparing this with Eq.~\eqref{eq:shared_expanded}, we obtain
\begin{equation}
\mathcal{L}_{\mathrm{shared}}(i)
=
K\mathcal{L}_{\mathrm{early}}(i)
+
\underbrace{
\sum_{s=1}^{K}
\sum_j
R_i^{(s)}(j)\log R_i^{(s)}(j)
}_{\text{constant w.r.t. }P_i}.
\end{equation}
Thus, the shared-$P_i$ objective and the GraphMSL-style early-fusion objective differ only by a scalar factor $K$ and a term independent of the encoder parameters.

Equivalently, the same relation can be written in KL form:
\begin{equation}
\boxed{
\sum_{s=1}^{K}
\mathrm{KL}\!\left(R_i^{(s)}\middle\|P_i\right)
=
K\,\mathrm{KL}\!\left(\bar R_i\middle\|P_i\right)
+
\sum_{s=1}^{K}
\mathrm{KL}\!\left(R_i^{(s)}\middle\|\bar R_i\right)
}.
\label{eq:shared_early_fusion_decomposition}
\end{equation}
The second term is independent of $P_i$, and hence independent of the encoder parameters.
The factor $K$ only rescales the fused-target KL objective and does not change its minimizer.
Hence,
\begin{equation}
\arg\min_\theta
\sum_{s=1}^{K}
\mathrm{KL}\!\left(R_i^{(s)}\middle\|P_i^\theta\right)
=
\arg\min_\theta
\mathrm{KL}\!\left(\bar R_i\middle\|P_i^\theta\right).
\end{equation}

We can also verify the same equivalence by comparing gradients.
For one prior space $s$,
\[
\mathrm{KL}\!\left(R_i^{(s)}\middle\|P_i\right)
=
\sum_{\ell}R_i^{(s)}(\ell)\log R_i^{(s)}(\ell)
-
\sum_{\ell}R_i^{(s)}(\ell)\log P_i(\ell).
\]
Using the softmax identity
\[
\frac{\partial \log P_i(\ell)}{\partial u_{ij}}
=
\mathbf{1}[\ell=j]-P_i(j),
\]
we obtain
\begin{equation}
\frac{\partial}{\partial u_{ij}}
\mathrm{KL}\!\left(R_i^{(s)}\middle\|P_i\right)
=
P_i(j)-R_i^{(s)}(j).
\end{equation}
Summing over all prior spaces gives
\begin{equation}
\frac{\partial \mathcal{L}_{\mathrm{shared}}(i)}
{\partial u_{ij}}
=
\sum_{s=1}^{K}
\left(P_i(j)-R_i^{(s)}(j)\right)
=
K\left(P_i(j)-\bar R_i(j)\right).
\end{equation}
For the GraphMSL-style early-fusion objective,
\[
\mathcal{L}_{\mathrm{early}}(i)
=
-\sum_{\ell}\bar R_i(\ell)\log P_i(\ell),
\]
the gradient is
\begin{equation}
\frac{\partial \mathcal{L}_{\mathrm{early}}(i)}
{\partial u_{ij}}
=
P_i(j)-\bar R_i(j).
\end{equation}
Therefore,
\begin{equation}
\boxed{
\frac{\partial \mathcal{L}_{\mathrm{shared}}(i)}
{\partial u_{ij}}
=
K
\frac{\partial \mathcal{L}_{\mathrm{early}}(i)}
{\partial u_{ij}}
}.
\end{equation}
Thus, the unweighted shared-$P_i$ multi-prior objective and the GraphMSL-style early-fusion objective induce the same optimization direction, differing only by the scalar factor $K$.
This scalar changes the loss scale but not the target relational geometry being learned.

By contrast, multi-branch PRISM does not optimize all prior spaces against the same $P_i$.
Instead, each prior space $s$ has its own embedding-induced distribution $P_i^{(s)}$, and PRISM optimizes
\[
\mathcal{L}_{\mathrm{PRISM}}(i)
=
\sum_{s=1}^{K}
\mathrm{KL}\!\left(R_i^{(s)}\middle\|P_i^{(s)}\right).
\]
For each branch, the gradient is
\[
\frac{\partial}{\partial u_{ij}^{(s)}}
\mathrm{KL}\!\left(R_i^{(s)}\middle\|P_i^{(s)}\right)
=
P_i^{(s)}(j)-R_i^{(s)}(j).
\]
These are separate prior-specific gradient signals.
They cannot be collapsed into one averaged signal $P_i(j)-\bar R_i(j)$, because the learner distributions $P_i^{(s)}$ are different across prior branches.
Hence, GraphMSL-style early fusion or a shared-$P_i$ multi-prior objective preserves only the fused relational geometry $\bar R_i$, whereas multi-branch PRISM preserves the prior-specific geometries $\{R_i^{(s)}\}_{s=1}^{K}$ before downstream fusion.

\subsection{Label Energy as a Prediction-Theoretic Upper Bound}
\label{app:theory:energy}

\begin{proposition}[Label energy upper-bounds prior-space $k$NN risk]
\label{prop:energy_knn_risk}
Fix a prior space $s$ and its neighbor sets $\mathcal{N}^{(s)}_k(i)$.

\textbf{Regression.}
For real-valued labels $y_i\in\mathbb{R}$, define the prior-space $k$NN regressor
\begin{equation}
    \widehat{y}_i
    =
    \frac{1}{k}
    \sum_{j\in\mathcal{N}^{(s)}_k(i)}
    y_j.
\end{equation}
Then
\begin{equation}
    \frac{1}{|Q|}
    \sum_{i\in Q}
    (y_i - \widehat{y}_i)^2
    \leq
    \mathcal{E}^{(s)}_k(y),
    \qquad
    \mathrm{RMSE}(\widehat{y})
    \leq
    \sqrt{\mathcal{E}^{(s)}_k(y)}.
\label{eq:energy_reg_bound}
\end{equation}

\textbf{Binary classification.}
For binary labels $y_i\in\{0,1\}$ and $k=1$, define the prior-space nearest-neighbor classifier $\widehat{y}_i = y_{j^\star(i)}$ with $j^\star(i)\in\mathcal{N}^{(s)}_1(i)$.
Then
\begin{equation}
    \frac{1}{|Q|}
    \sum_{i\in Q}
    \mathbf{1}\{\widehat{y}_i \neq y_i\}
    =
    \mathcal{E}^{(s)}_1(y).
\label{eq:energy_1nn_error}
\end{equation}
For $k$NN majority vote with odd $k$, the misclassification rate satisfies
\begin{equation}
    \frac{1}{|Q|}
    \sum_{i\in Q}
    \mathbf{1}\{\widehat{y}_i \neq y_i\}
    \leq
    2\mathcal{E}^{(s)}_k(y).
\label{eq:energy_majority_bound}
\end{equation}
\end{proposition}

\begin{proof}
\textit{Regression.}
By definition,
\begin{equation}
    y_i - \widehat{y}_i
    =
    \frac{1}{k}
    \sum_{j\in\mathcal{N}^{(s)}_k(i)}
    (y_i - y_j).
\end{equation}
By Jensen's inequality applied to the squared function,
\begin{equation}
    (y_i - \widehat{y}_i)^2
    \leq
    \frac{1}{k}
    \sum_{j\in\mathcal{N}^{(s)}_k(i)}
    (y_i - y_j)^2.
\end{equation}
Averaging over $i\in Q$ gives the regression bound.
Taking the square root yields the RMSE bound.

\textit{Binary classification, $k=1$.}
For $y_i, y_{j^\star(i)} \in \{0, 1\}$,
\begin{equation}
    \mathbf{1}\{y_i \neq y_{j^\star(i)}\}
    =
    (y_i - y_{j^\star(i)})^2.
\end{equation}
Averaging over $i\in Q$ yields Eq.~\eqref{eq:energy_1nn_error}.

\textit{Binary classification, $k$NN majority vote.}
If the majority vote disagrees with $y_i$, then at least half of the $k$ neighbors must have labels different from $y_i$:
\begin{equation}
    \mathbf{1}\{\widehat{y}_i \neq y_i\}
    \leq
    2 \cdot \frac{1}{k}
    \sum_{j\in\mathcal{N}^{(s)}_k(i)}
    \mathbf{1}\{y_j \neq y_i\}.
\end{equation}
Using $\mathbf{1}\{y_j\neq y_i\} = (y_i - y_j)^2$ for binary labels and averaging over $i\in Q$ yields Eq.~\eqref{eq:energy_majority_bound}.
\end{proof}

\paragraph{Interpretation.}
Proposition~\ref{prop:energy_knn_risk} shows that low label energy implies the existence of a low-error local predictor operating directly in the prior space.
Label energy is therefore a prediction-theoretic diagnostic of prior--task compatibility.
For regression, it directly controls the RMSE of a prior-space $k$NN regressor.
For binary classification, it controls nearest-neighbor disagreement in the prior space.
Although our classification metric is ROC-AUC rather than 0--1 error, low local disagreement is a useful indicator that the prior geometry separates labels in a way that can support downstream classification.

\subsection{No-Worse Approximation on a Finite Set}
\label{app:theory:noworse}

We finally give a finite-set justification for preserving prior information in learned representations.
The purpose is not to prove that PRISM always improves generalization, but to clarify why retaining recoverable prior information should not reduce empirical predictive power when the downstream predictor class is sufficiently expressive.

Let $\mathcal{D} = \{(x_i, y_i)\}_{i=1}^n$ be a finite task dataset.
For prior space $s$, let $\pi_i^{(s)}$ denote the prior feature and $r_i^{(s)}$ the learned representation produced by the corresponding branch.

\begin{proposition}[No-worse approximation on a finite set]
\label{prop:no_worse_finite}
Assume there exists a recovery map $f_{\mathrm{rec}}^{(s)}$ such that
\begin{equation}
    \pi_i^{(s)}
    =
    f_{\mathrm{rec}}^{(s)}(r_i^{(s)})
    \quad
    \text{for all } i = 1,\ldots,n.
\label{eq:prior_recovery_assumption}
\end{equation}
Let $\mathcal{F}$ be a class of predictors on prior features and $\mathcal{H}$ a class of predictors on learned representations.
Define the empirical squared errors
\begin{equation}
    \widehat{e}_\pi
    =
    \inf_{f\in\mathcal{F}}
    \frac{1}{n}\sum_{i=1}^n
    (y_i - f(\pi_i^{(s)}))^2,
    \qquad
    \widehat{e}_r
    =
    \inf_{h\in\mathcal{H}}
    \frac{1}{n}\sum_{i=1}^n
    (y_i - h(r_i^{(s)}))^2.
\end{equation}
If $\{f \circ f_{\mathrm{rec}}^{(s)} : f \in \mathcal{F}\} \subseteq \mathcal{H}$, then $\widehat{e}_r \leq \widehat{e}_\pi$.
\end{proposition}

\begin{proof}
For any $f\in\mathcal{F}$, let $h_f(r) = f(f_{\mathrm{rec}}^{(s)}(r))$.
By the predictor-class assumption, $h_f\in\mathcal{H}$.
By the recovery assumption, $h_f(r_i^{(s)}) = f(\pi_i^{(s)})$ for all $i$.
Therefore the empirical error of $h_f$ on $\mathcal{D}$ equals the empirical error of $f$ on $\mathcal{D}$.
Taking infima on both sides yields $\widehat{e}_r \leq \widehat{e}_\pi$.
\end{proof}

\paragraph{Interpretation.}
Proposition~\ref{prop:no_worse_finite} states that if prior features are recoverable from learned representations on a finite dataset, then any predictor that could have been built from prior features can also be represented as a predictor on the learned representations, provided the downstream predictor class is sufficiently expressive.
Therefore, preserving prior information need not make the representation empirically worse than using prior features directly.
This is a finite-set approximation argument, not a generalization theorem.
It depends on the recoverability assumption and on the expressiveness of the downstream predictor class.

\clearpage
\section{Additional Results}
\label{app:extra_results}

This section provides additional results that complement the main tables in Section~4.
We include the extended MoleculeNet comparison, pure-solvent OOD slice analysis, and temperature-dependent behavior analysis.
We avoid repeating main-text tables unless additional decomposition or diagnostic information is provided.
Mixed-solvent architectural ablations are reported separately in Appendix~\ref{app:extra_ablation:mixed_head}, because they test alternative prediction heads rather than additional main-result settings.

\subsection{Extended MoleculeNet Comparison}
\label{app:extra_results:molnet_full}

The main text reports a compact subset of representative baselines.
Table~\ref{tab:molnet_full} provides an extended MoleculeNet comparison, including traditional feature-based models, supervised GNNs, 2D pretrained molecular models, geometry-aware pretrained models, and PRISM controls.
We include HIV where matched results are available in our runs; external baselines without matched HIV results under our protocol are marked as not reported rather than imputed.

We emphasize three details about this comparison.
First, the default PRISM model uses only 2D molecular graphs and chemically interpretable 2D prior spaces.
Geometry-aware baselines are included for completeness but are grouped separately.
Second, the reported \texttt{GROVER\_large} values are taken from the standardized KANO benchmark rather than directly from the original GROVER paper.
This avoids mixing incompatible evaluation protocols and addresses known sensitivity of MoleculeNet scaffold splits.
Third, HIV is included as a classification task, but not every external baseline has a matched HIV value under the same evaluation protocol; we therefore report available HIV values explicitly and avoid filling missing entries with incompatible numbers.

\begin{table*}[h]
\centering
\small
\setlength{\tabcolsep}{4.0pt}
\caption{
\textbf{Extended MoleculeNet comparison on classification tasks (ROC-AUC \%, $\uparrow$) and regression tasks (RMSE, $\downarrow$).}
\textbf{Bold} indicates the best result among reported values in each column.
PRISM uses only 2D molecular graphs; geometry-aware baselines are reported for completeness and grouped separately.
Dashes indicate that a matched value under our protocol is not reported.
}
\label{tab:molnet_full}
\resizebox{0.99\textwidth}{!}{%
\begin{tabular}{l|ccc|ccc}
\toprule
& \multicolumn{3}{c|}{\textbf{Classification} (ROC-AUC \%, $\uparrow$)}
& \multicolumn{3}{c}{\textbf{Regression} (RMSE, $\downarrow$)} \\
\midrule
Dataset & BBBP & BACE & HIV & ESOL & FreeSolv & Lipo \\
\# Molecules & 2,039 & 1,513 & 41,127 & 1,128 & 642 & 4,200 \\
\midrule
\multicolumn{7}{l}{\textit{\textbf{Traditional ML} (fingerprints + VSA)}} \\
MLP
& 85.1 $\pm$ 1.6
& 85.2 $\pm$ 0.8
& 70.0 $\pm$ 1.7
& 1.461 $\pm$ 0.060
& 4.942 $\pm$ 0.084
& 0.878 $\pm$ 0.016 \\
XGBoost
& 84.7 $\pm$ 0.0
& 88.5 $\pm$ 0.0
& 74.7 $\pm$ 0.8
& 0.925 $\pm$ 0.000
& 3.398 $\pm$ 0.000
& 0.741 $\pm$ 0.000 \\
\midrule
\multicolumn{7}{l}{\textit{\textbf{Supervised GNNs} (without pretraining)}} \\
GIN
& 65.8 $\pm$ 4.5
& 70.1 $\pm$ 5.4
& 75.3 $\pm$ 1.9
& 1.452 $\pm$ 0.020
& 2.765 $\pm$ 0.180
& 0.850 $\pm$ 0.071 \\
D-MPNN
& 91.9 $\pm$ 3.0
& 85.2 $\pm$ 0.6
& 77.1 $\pm$ 0.5
& 1.050 $\pm$ 0.008
& 1.673 $\pm$ 0.082
& 0.683 $\pm$ 0.016 \\
Attentive FP
& 64.3 $\pm$ 1.8
& 78.4 $\pm$ 0.0
& 75.8 $\pm$ 1.5
& 0.877 $\pm$ 0.029
& 2.073 $\pm$ 0.183
& 0.721 $\pm$ 0.001 \\
D-MPNN w.\ Prior
& 90.6 $\pm$ 0.1
& 87.7 $\pm$ 0.2
& 73.0 $\pm$ 0.2
& 0.798 $\pm$ 0.003
& 2.665 $\pm$ 0.084
& 0.645 $\pm$ 0.004 \\
\midrule
\multicolumn{7}{l}{\textit{\textbf{Pretrained models (2D-only)}}} \\
PretrainGNN
& 68.7 $\pm$ 1.3
& 84.5 $\pm$ 0.7
& 79.5 $\pm$ 0.7
& 1.100 $\pm$ 0.006
& 2.764 $\pm$ 0.002
& 0.739 $\pm$ 0.003 \\
GROVER$_{\mathrm{large}}$
& 86.8 $\pm$ 2.2
& 82.4 $\pm$ 3.6
& 68.2 $\pm$ 1.1
& 1.423 $\pm$ 0.288
& 2.947 $\pm$ 0.615
& 0.823 $\pm$ 0.010 \\
MolCLR
& 73.3 $\pm$ 1.0
& 82.8 $\pm$ 0.7
& 77.4 $\pm$ 0.6
& 1.113 $\pm$ 0.023
& 2.301 $\pm$ 0.247
& 0.789 $\pm$ 0.009 \\
MoAMa
& 72.8 $\pm$ 0.9
& 80.1 $\pm$ 0.5
& 79.3 $\pm$ 0.6
& 1.394 $\pm$ 0.120
& 2.700 $\pm$ 0.110
& 0.772 $\pm$ 0.030 \\
MoleculeSDE$_{\mathrm{GIN}}$
& 67.4 $\pm$ 1.2
& 78.1 $\pm$ 4.1
& 76.5 $\pm$ 0.8
& 2.080 $\pm$ 0.190
& 3.640 $\pm$ 0.170
& 0.610 $\pm$ 0.000 \\
\midrule
\multicolumn{7}{l}{\textit{\textbf{Pretrained models (3D / geometry-aware)}}} \\
3D-Infomax
& 69.1 $\pm$ 1.1
& 79.4 $\pm$ 1.9
& 75.9 $\pm$ 0.59
& 0.894 $\pm$ 0.028
& 2.337 $\pm$ 0.227
& 0.739 $\pm$ 0.009 \\
GEM
& 88.8 $\pm$ 0.4
& 87.9 $\pm$ 1.1
& \textbf{81.3 $\pm$ 0.3}
& 0.813 $\pm$ 0.028
& 1.748 $\pm$ 0.114
& 0.674 $\pm$ 0.022 \\
Uni-Mol
& 72.9 $\pm$ 0.6
& 85.7 $\pm$ 0.2
& 79.6 $\pm$ 0.5
& 0.788 $\pm$ 0.029
& \textbf{1.620 $\pm$ 0.035}
& 0.603 $\pm$ 0.010 \\
\midrule
PRISM$_{\mathrm{random}}$
& 93.0 $\pm$ 0.6
& 87.8 $\pm$ 1.7
& 75.1 $\pm$ 0.5
& 0.848 $\pm$ 0.025
& 2.222 $\pm$ 0.054
& 0.642 $\pm$ 0.014 \\
D-MPNN (reconstruction)
& 94.1 $\pm$ 0.6
& 71.1 $\pm$ 6.7
& 73.4 $\pm$ 0.9
& 0.829 $\pm$ 0.076
& 2.308 $\pm$ 0.075
& 0.577 $\pm$ 0.019 \\
\textbf{PRISM}
& \textbf{94.6 $\pm$ 0.3}
& \textbf{89.4 $\pm$ 0.4}
& 77.7 $\pm$ 0.6
& \textbf{0.772 $\pm$ 0.002}
& 1.770 $\pm$ 0.093
& \textbf{0.538 $\pm$ 0.018} \\
\bottomrule
\end{tabular}%
}
\end{table*}

\paragraph{Discussion.}
PRISM achieves the best reported performance on BBBP, BACE, ESOL, and Lipophilicity, and remains competitive on HIV and FreeSolv.
The comparison with PRISM$_{\mathrm{random}}$ isolates the effect of PRISM initialization under the same four-branch architecture and downstream heads.
The weaker relative performance on HIV and FreeSolv reflects two different limitations.
HIV has weaker alignment with the default VSA/fingerprint priors, although Appendix~\ref{app:extra_ablation:dynamic_selection} shows that selecting a Pharm2D prior using label energy improves HIV performance.
FreeSolv shows that low local label energy alone does not guarantee an advantage over a strong supervised baseline when task size, baseline strength, and missing task-specific factors such as 3D hydration geometry matter.

\subsection{Pure-Solvent OOD Slice Analysis}
\label{app:extra_results:slice}

Aggregate external-test metrics do not distinguish between different sources of distribution shift.
We therefore decompose the Leeds and SolProp results by solvent identity and solvent-conditioned temperature range.
All test solutes are absent from the processed BigSolDB dataset.
We use this complete dataset, including its training and validation partitions, as the reference for defining the slices.
Solute and solvent identities are matched using canonical isomeric SMILES.

A test solvent is classified as seen if its canonical SMILES appears in the BigSolDB reference set, and unseen otherwise.
All three Leeds solvents---acetone, benzene, and ethanol---are seen under this definition.
SolProp contains both seen and unseen solvents.

For each seen solvent, we define the reference temperature range by the minimum and maximum temperatures observed for that solvent in BigSolDB.
A test measurement is classified as temperature in-range if its temperature lies within this inclusive interval, and as temperature extrapolation otherwise.
Temperature slices are defined only for seen solvents.
All Leeds measurements are temperature in-range, so its overall, seen-solvent, and temperature-in-range slices contain the same records.

Table~\ref{tab:ood_slices_unseen_solute} reports the resulting comparisons.
All 1,294 Leeds records and 5,230 SolProp records are retained.
Repeated measurements at the same solute--solvent--temperature conditions remain separate records.
The reported metrics summarize the fixed prediction sets used for this analysis; seed-level uncertainty is not reported.

\begin{table*}[t]
\centering
\caption{\textbf{OOD slice analysis on unseen-solute tests.}
Slices are defined relative to the complete processed BigSolDB dataset using canonical isomeric SMILES.
Temperature slices use the solvent-specific temperature range and apply only to seen solvents.
For Leeds, the seen-solvent and temperature-in-range slices coincide with the overall set.
Metrics are RMSE ($\downarrow$) and pct-in-1 ($\uparrow$), the percentage of predictions within one $\log S$ unit of the measured value.
Differences are PRISM minus Fastprop; differences in pct-in-1 are reported in percentage points (pp).}
\label{tab:ood_slices_unseen_solute}
\small
\setlength{\tabcolsep}{6pt}
\resizebox{0.99\textwidth}{!}{%
\begin{tabular}{llrcccccc}
\toprule
Dataset & Slice & $n$
& \multicolumn{2}{c}{Fastprop}
& \multicolumn{2}{c}{PRISM}
& \multicolumn{2}{c}{$\Delta$ (PRISM--Fastprop)} \\
& & &
RMSE $\downarrow$ & pct-in-1 (\%) $\uparrow$ &
RMSE $\downarrow$ & pct-in-1 (\%) $\uparrow$ &
RMSE & pct-in-1 (pp) \\
\midrule
\multicolumn{9}{l}{\textbf{Leeds}} \\
\multicolumn{9}{l}{\emph{All solutes are unseen; all solvents are seen; all temperatures are in-range.}} \\
Leeds & Overall
& 1,294 & 0.954 & 73.57 & 0.909 & 75.43 & $-0.045$ & $+1.85$ \\
\midrule
\multicolumn{9}{l}{\textbf{SolProp}} \\
SolProp & Overall
& 5,230 & 0.833 & 78.07 & 0.773 & 82.28 & $-0.060$ & $+4.21$ \\
SolProp & Solvent: Seen
& 4,988 & 0.835 & 77.83 & 0.772 & 82.40 & $-0.062$ & $+4.57$ \\
SolProp & Solvent: Unseen
& 242 & 0.795 & 83.06 & 0.781 & 79.75 & $-0.014$ & $-3.31$ \\
SolProp & Temp.\ (solvent-conditioned): In-range
& 4,791 & 0.840 & 77.46 & 0.767 & 82.84 & $-0.073$ & $+5.39$ \\
SolProp & Temp.\ (solvent-conditioned): Extrapolation
& 197 & 0.693 & 86.80 & 0.894 & 71.57 & $+0.201$ & $-15.23$ \\
\bottomrule
\end{tabular}%
}
\end{table*}

\paragraph{Discussion.}
PRISM achieves lower overall RMSE and higher pct-in-1 on both external datasets.
On SolProp, it improves both metrics for seen solvents and for temperatures within the solvent-specific reference range.
The results are mixed for unseen solvents: PRISM reduces RMSE from 0.795 to 0.781, but pct-in-1 decreases from 83.06\% to 79.75\%.
Temperature extrapolation remains a limitation.
On the 197 measurements outside the corresponding solvent-specific temperature ranges, PRISM has higher RMSE than Fastprop (0.894 versus 0.693) and lower pct-in-1 (71.57\% versus 86.80\%).
These results support improved aggregate external-test performance while showing that the gains do not extend uniformly across solvent and temperature shifts.

\subsection{Temperature-Dependent Behavior and
{Normalized} EMD Analysis}
\label{app:temperature_trends}

Solubility prediction is not only a pointwise regression problem.
For temperature-dependent measurements, a useful model should also
reproduce the empirical distribution of temperature sensitivity.
We therefore compare the distribution of finite-difference temperature gradients
\begin{equation}
    \frac{d\log S}{dT}
\end{equation}
computed from experimental measurements and model predictions.
This analysis supports the main-text claim that PRISM better matches
the empirical temperature-gradient distribution, with
{normalized EMD} 0.01 compared with 0.02 for Fastprop.

\paragraph{Gradient construction.}
{
For each solute--solvent pair with at least two measurements, we compute
numerical derivatives using \texttt{numpy.gradient}, supplying the
recorded temperatures as the coordinate array.
Let $\mathbf{T}=(T_1,\ldots,T_n)$,
$\mathbf{y}=(\log S_1,\ldots,\log S_n)$, and
$\widehat{\mathbf{y}}=(\widehat{\log S}_1,\ldots,\widehat{\log S}_n)$
denote the temperatures, experimental values, and model predictions
at the same measurement conditions.
The experimental gradients are
}
\begin{equation}
    {
    \mathbf{g}^{\mathrm{exp}}
    =
    \operatorname{gradient}(\mathbf{y},\mathbf{T}).
    }
\label{eq:temp_gradient_exp}
\end{equation}
{
For each model, we apply the same operator to its predictions:
}
\begin{equation}
    {
    \mathbf{g}^{\mathrm{model}}
    =
    \operatorname{gradient}(\widehat{\mathbf{y}},\mathbf{T}).
    }
\label{eq:temp_gradient_model}
\end{equation}
{
Here, $\operatorname{gradient}$ uses centered finite differences at
interior points and first-order one-sided differences at the boundaries,
with the supplied temperature coordinates accounting for nonuniform spacing.
We retain entries only when both the experimental and predicted
gradients are finite.
}

\paragraph{Distribution construction and
{normalized} EMD.}
We aggregate gradients across all eligible solute--solvent temperature
series in the evaluation set.
{
For the CDFs and metric computation, we construct histograms using
$M=200$ equally spaced bins over the fixed gradient interval
$[-0.05,0.05]$.
Gradients outside this interval are excluded from the respective
histograms, and each histogram is normalized by its retained count
to obtain a probability mass function.
Panel~(e) uses 120 bins over the same interval to visualize the counts.
}
Let $p\in\mathbb{R}^{M}$ be the empirical gradient distribution and
$q\in\mathbb{R}^{M}$ be the model-gradient distribution.

For one-dimensional histograms with ordered bins, Earth Mover's Distance
can be computed from cumulative distribution functions.
Let
\begin{equation}
    F_p(b)
    =
    \sum_{m=1}^{b}p_m,
    \qquad
    F_q(b)
    =
    \sum_{m=1}^{b}q_m.
\end{equation}
{
We report normalized EMD, defined as the mean absolute difference
between the two CDF arrays:
}
\begin{equation}
    {
    \mathrm{EMD}_{\mathrm{norm}}(p,q)
    =
    \frac{1}{M}
    \sum_{b=1}^{M}
    \left|F_p(b)-F_q(b)\right|.
    }
\label{eq:emd}
\end{equation}
{
Because both CDF arrays are nondecreasing, this expression is exactly
the value returned by applying \texttt{wasserstein\_distance} to the
two arrays.
For equal-width bins with interval length $L=0.1$ and bin width
$\Delta=L/M$, it is also equivalent to the histogram EMD
$\sum_b |F_p(b)-F_q(b)|\Delta$ divided by $L$.
The reported metric is therefore dimensionless.
Lower normalized EMD indicates closer agreement between the experimental
and predicted gradient distributions within the analyzed interval.
}

\begin{figure*}[h]
  \centering
  \includegraphics[width=\textwidth]{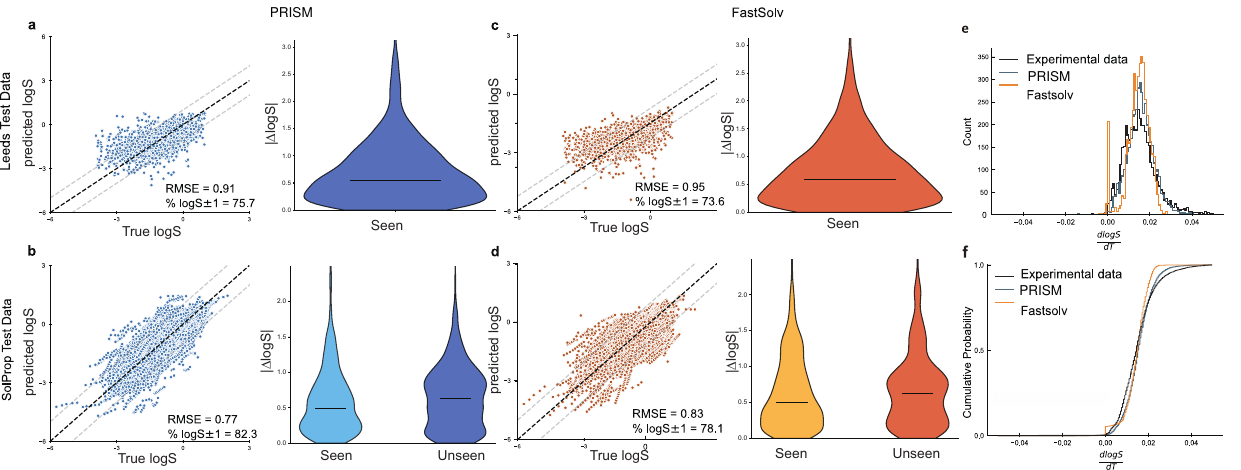}
  \caption{
  \textbf{Pure-solvent external-test behavior and temperature-gradient analysis.}
  Panels (a,b) show PRISM predictions on Leeds and SolProp, together
  with error distributions for seen/unseen solvent slices.
  Panels (c,d) show the corresponding Fastprop predictions.
  Panels (e,f) compare the empirical and model-induced finite-difference
  temperature-gradient distributions.
  PRISM achieves lower {normalized EMD} to the empirical
  $d\log S/dT$ distribution than Fastprop (0.01 vs.\ 0.02), despite
  not using an explicit temperature-gradient loss.
  }
  \label{fig:pure_solvent_temperature}
\end{figure*}

\paragraph{Comparison with Fastprop.}
Fastprop explicitly incorporates a temperature-gradient loss on
$d\log S/dT$, whereas PRISM does not impose such a thermodynamic
derivative constraint.
Despite this, PRISM obtains a lower {normalized EMD}
between the predicted and experimental gradient distributions.
This suggests that prior-induced molecular representations,
combined with solute--solvent relation modeling, help capture
temperature-dependent solubility trends without directly optimizing
a gradient-matching objective.

\paragraph{Caveat.}
This distribution-level result should not be interpreted as solving
all forms of temperature extrapolation.
As shown in Appendix~\ref{app:extra_results:slice}, the small
solvent-conditioned temperature-extrapolation subset in SolProp
remains difficult.
The {normalized EMD} analysis supports a weaker
and more specific claim: across the evaluated external-test
temperature intervals, PRISM reproduces the aggregate distribution
of temperature sensitivity more faithfully than the baseline
{within the analyzed gradient range}.

\subsection{Independent Pure/Binary/Ternary Mixture Evaluation}
\label{app:cenci_mixture}

To evaluate whether PRISM representations remain useful beyond the binary-solvent benchmark in Section~\ref{sec:experiments}, we additionally use the independent dataset of Cenci et al.~\citep{cenci2024solubility}.
The dataset contains solubility measurements for one active pharmaceutical ingredient (API) in 14 organic solvents and includes pure solvents, binary mixtures, and ternary mixtures at multiple temperatures.
Following the original calibration--validation protocol, 176 pure- and binary-solvent measurements are used for calibration, while the 288-sample validation set contains pure, binary, and ternary solvent systems under new compositions and temperatures.

Because this dataset contains only one solute and a relatively small calibration set, we use it as a representation-reuse test rather than a new end-to-end molecular-learning benchmark.
We freeze the PRISM-pretrained encoders, construct composition-weighted solvent representations, and fit a low-capacity quadratic Ridge regressor on the calibration data.
We compare this model with the UNIFAC--PLS approach used by Cenci et al., and additionally evaluate a hybrid model combining UNIFAC and PRISM information.

\begin{table}[h]
\centering
\small
\setlength{\tabcolsep}{8pt}
\caption{
\textbf{Independent pure-, binary-, and ternary-solvent evaluation on the Cenci et al.~dataset.}
The original calibration--validation split is retained.
Lower log-RMSE and higher $R^2$ indicate better performance.
}
\label{tab:cenci_mixture}
\begin{tabular}{lccc}
\toprule
Model
& Overall log-RMSE $\downarrow$
& Overall $R^2$ $\uparrow$
& Ternary log-RMSE $\downarrow$ \\
\midrule
UNIFAC--PLS
& 0.3061
& 0.9042
& 0.0921 \\

PRISM--Ridge
& 0.2738
& 0.9217
& 0.1954 \\

UNIFAC + PRISM
& \textbf{0.2502}
& \textbf{0.9409}
& \textbf{0.0904} \\
\bottomrule
\end{tabular}
\end{table}

PRISM--Ridge improves the overall validation performance relative to UNIFAC--PLS, reducing log-RMSE from 0.3061 to 0.2738 and increasing $R^2$ from 0.9042 to 0.9217.
However, PRISM alone is weaker on the small ternary subset, where UNIFAC--PLS achieves a lower log-RMSE.
Combining UNIFAC and PRISM yields the best overall performance and also slightly improves the ternary result.
These results suggest that PRISM representations can be reused beyond binary mixtures and provide information complementary to a specialized physicochemical mixture model, while not replacing such domain-specific modeling in every regime.

\subsection{Unlabeled Target-Domain Adaptation}
\label{app:target_adaptation}

We investigate whether additional relational pretraining on unlabeled
target-domain molecules improves pure-solvent external transfer.
Let $\mathcal{Z}$ denote the ZINC250k pretraining corpus and
$\mathcal{T}$ the set of unique solute and solvent molecules collected
from BigSolDB, Leeds, and SolProp.
We canonicalize molecular SMILES and remove duplicates when constructing
the joint pool $\mathcal{U}=\mathcal{Z}\cup\mathcal{T}$.
Only molecular structures and their computed prior features are used
during relational pretraining; Leeds and SolProp solubility labels
are not used for pretraining or supervised fine-tuning.
These experiments therefore evaluate transductive adaptation separately
from the default inductive protocol used in the main results.

\paragraph{Adaptation protocols.}
The default protocol uses ZINC250k molecules as both anchors and
candidate neighbors.
We first evaluate target-only adaptation, where both are drawn from
$\mathcal{T}$.
We then retain target-domain molecules as anchors while expanding
the candidate-neighbor pool to $\mathcal{U}$, evaluating this setting
at 100 and 300 epochs.
Finally, we evaluate joint source--target pretraining, where both
anchors and candidate neighbors are drawn from $\mathcal{U}$,
for 100 epochs.
In each case, relational pretraining is followed by the same
supervised BigSolDB fine-tuning and external evaluation on Leeds
and SolProp.

\begin{table*}[t]
\centering
\small
\setlength{\tabcolsep}{5pt}
\caption{\textbf{Unlabeled target-domain adaptation for pure-solvent
external transfer.}
``Target'' denotes $\mathcal{T}$ and ``joint'' denotes
$\mathcal{U}=\mathcal{Z}\cup\mathcal{T}$.
Epoch counts, where shown, refer to the relational-pretraining stage.
Results are mean $\pm$ standard deviation over three seeds.
The metric pct-in-1 is the percentage of predictions within one
log-solubility unit of the experimental value.}
\label{tab:target_adaptation}
\resizebox{\textwidth}{!}{%
\begin{tabular}{lcccc}
\toprule
Protocol
& Leeds RMSE $\downarrow$
& SolProp RMSE $\downarrow$
& Leeds pct-in-1 (\%) $\uparrow$
& SolProp pct-in-1 (\%) $\uparrow$ \\
\midrule
ZINC-only pretraining (100 epochs)
& $0.910 \pm 0.002$
& $0.770 \pm 0.030$
& $75.7 \pm 0.9$
& $82.3 \pm 0.8$ \\
Target-only adaptation
& $0.920 \pm 0.002$
& $0.740 \pm 0.041$
& $74.1 \pm 0.7$
& $82.5 \pm 0.6$ \\
Target anchors, joint neighbors (100 epochs)
& $0.910 \pm 0.002$
& $0.775 \pm 0.040$
& $75.6 \pm 0.6$
& $82.4 \pm 0.5$ \\
Target anchors, joint neighbors (300 epochs)
& $0.941 \pm 0.002$
& $0.796 \pm 0.022$
& $73.1 \pm 0.6$
& $79.1 \pm 0.4$ \\
Joint anchors and neighbors (100 epochs)
& $0.908 \pm 0.002$
& $0.766 \pm 0.030$
& $75.7 \pm 0.9$
& $82.5 \pm 0.6$ \\
\bottomrule
\end{tabular}%
}
\end{table*}

\paragraph{Discussion.}
The effect of additional adaptation depends on the protocol and
evaluation dataset.
Target-only adaptation improves SolProp RMSE but worsens Leeds
performance.
Expanding the candidate-neighbor pool restores Leeds RMSE to the
default level at 100 epochs, without improving SolProp RMSE.
Extending this target-anchor protocol to 300 epochs worsens performance
on both datasets, indicating that longer adaptation is not necessarily
beneficial in this setting.

Joint source--target pretraining yields small numerical improvements
over the default protocol: RMSE decreases from 0.910 to 0.908 on Leeds
and from 0.770 to 0.766 on SolProp.
Leeds pct-in-1 remains unchanged, while SolProp pct-in-1 increases
from 82.3\% to 82.5\%.
Thus, although joint pretraining provides modest gains in this
comparison, additional access to unlabeled target-domain molecules
does not produce consistent improvements across the evaluated
adaptation protocols.
We retain the default inductive protocol for the main experiments,
as it achieves comparable performance without using external-test
molecules during pretraining.

\clearpage
\section{Additional Ablation Studies}
\label{app:extra_ablation}

This section provides additional ablation studies that complement the main-text results without repeating tables already reported in Section~4.
The main paper reports the mixed-solvent random and solvent-pair-disjoint results, the Gaussian-random-prior control on the primary MoleculeNet tasks, the solubility pretraining ablation, and the KL/$\ell_2$/rank alignment-objective comparison.
Here we include additional controls that are not fully shown in the main text:
(i) the effect of PRISM initialization on MoleculeNet under the same PRISM architecture;
(ii) the single-shared-encoder baseline on MoleculeNet;
(iii) leave-one-prior-space-out ablations;
(iv) task-specific versus instance-wise weighted fusion;
(v) mixed-solvent head ablations under the frozen-encoder protocol;
(vi) extended prior-quality and dynamic prior-selection analyses on additional bioactivity tasks;
(vii) compatibility with 3D encoders and 3D prior spaces; and
(viii) the effect of pretraining-data scale.

\paragraph{Protocol note on pretraining hyperparameters.}
The sensitivity analyses in Appendix~\ref{app:sensitivity_tau_gamma} are used as a small validation-driven selection grid for PRISM pretraining hyperparameters.
Thus, the main reported PRISM results use the task-specific settings listed in Table~\ref{tab:hp_single}, rather than a single universal pretraining checkpoint for all tasks.
The statement that $\tau=0.1,\gamma_s=0.25$ is a simple robust default should therefore be interpreted as a practical baseline setting that performs reasonably across several tasks, not as the exact setting used for every reported main result.
All PRISM controls and ablations for a given task use the corresponding task-level setting unless explicitly stated otherwise.
This avoids attributing gains to downstream heads alone, but it also means that comparisons against external baselines should be read as comparisons under each method's reported model-selection protocol rather than as a strict equal-pretraining-grid comparison.

\paragraph{Protocol note on mixed-solvent baselines.}
The mixed-solvent PRISM experiments use original PRISM-pretrained encoders and initialize the fusion module from a pure-solvent checkpoint.
The four end-to-end MPNN baselines in the main mixed-solvent table are trained on the mixed-solvent dataset under the reported mixed-solvent protocols and do not use the same ZINC-plus-BigSolDB pretraining path.
Therefore, the mixed-solvent comparison is best interpreted as a representation-reuse probe: it tests whether PRISM-pretrained molecular representations can be reused under frozen encoders and lightweight composition-weighted pooling.
It is not a strict equal-data-budget comparison between pretraining strategies.
This distinction is important because the fusion initialization provides PRISM with additional task-relevant solute--solvent experience from BigSolDB.

\subsection{PRISM Initialization Control on MoleculeNet}
\label{app:extra_ablation:molnet_pretrain}

We first isolate the effect of PRISM pretraining on MoleculeNet.
Both models in Table~\ref{tab:abl_pretrain_molnet} use the same four-branch PRISM architecture and the same downstream heads.
The only difference is initialization: PRISM$_{\mathrm{random}}$ is trained from random initialization, whereas PRISM uses PRISM-pretrained branches.

\begin{table}[h]
\caption{\textbf{Effect of PRISM pretraining on MoleculeNet.}
Same four-branch PRISM architecture with identical heads trained from random initialization
(PRISM$_{\mathrm{random}}$) versus PRISM-pretrained initialization (PRISM).}
\label{tab:abl_pretrain_molnet}
\centering
\small
\setlength{\tabcolsep}{5pt}
\resizebox{0.99\textwidth}{!}{%
\begin{tabular}{l|ccc|ccc}
\toprule
& \multicolumn{3}{c|}{\textbf{Classification} (ROC-AUC \%, $\uparrow$)}
& \multicolumn{3}{c}{\textbf{Regression} (RMSE, $\downarrow$)} \\
\midrule
Init & BBBP & BACE & HIV & ESOL & FreeSolv & Lipo \\
\midrule
PRISM$_{\mathrm{random}}$
& $93.0\pm0.6$
& $87.8\pm1.7$
& $75.1\pm0.5$
& $0.848\pm0.025$
& $2.222\pm0.054$
& $0.642\pm0.014$ \\
PRISM
& $\mathbf{94.6\pm0.3}$
& $\mathbf{89.4\pm0.4}$
& $\mathbf{77.7\pm0.6}$
& $\mathbf{0.772\pm0.002}$
& $\mathbf{1.770\pm0.093}$
& $\mathbf{0.538\pm0.018}$ \\
\midrule
$\Delta$ (PRISM $-$ random)
& $+1.6$
& $+1.6$
& $+2.6$
& $-0.076$
& $-0.452$
& $-0.104$ \\
\bottomrule
\end{tabular}
}
\end{table}

\paragraph{Discussion.}
PRISM pretraining improves the same PRISM architecture on all six MoleculeNet tasks.
This rules out the explanation that the gains come only from increasing model capacity through a multi-branch architecture.

\subsection{Compression diagnostic via PRISM-Distilled}
\label{app:distillation}

\paragraph{Construction.}
PRISM-Distilled is a single-encoder variant trained from PRISM under the same alignment family (Section~\ref{sec:prism_family}).
The teacher is the multi-prior PRISM model with four pretrained branch encoders.
For each molecule $x_i$, we extract its concatenated multi-prior representation
\[
H_i = [h_i^{(s_1)};\, h_i^{(s_2)};\, h_i^{(s_3)};\, h_i^{(s_4)}] \in \mathbb{R}^{K \times D}
\]
and treat $H_i$ as a new ``prior representation'' that defines a teacher-induced relational structure over the molecular population.
We then train a fresh single MPNN student encoder $G_{\theta}$ from random initialization to align with this teacher relation under the same KL alignment objective in Eq.~\eqref{eq:align_loss}, with locality scale $\beta$ rescaled to the median pairwise distance of $\{H_i\}$ in the mini-batch.
The student therefore inherits the relational geometry of the teacher (which molecules are close to which) rather than absolute coordinates, scales, or branch-specific information.

\paragraph{Pretraining and downstream protocol.}
Distillation is run on the same 250K ZINC subset used for PRISM pretraining, with batch size 512, embedding temperature $\tau=0.1$, and AdamW optimizer.
For all downstream evaluations, we replace the PRISM encoder $E_\Theta$ with the distilled student $G_\theta$ and reuse the same task-specific heads and training protocols described in the main text (Section~\ref{sec:task_heads} and Appendix~\ref{app:impl}).
We test three downstream regimes: pure-solvent external transfer (Table~\ref{tab:pure_distill_results}), mixed-solvent prediction under both frozen and fine-tuned encoder settings (Table~\ref{tab:mixed_distill_results}), and single-molecule downstream tasks with encoder fine-tuning (Table~\ref{tab:single_molecule_distill_results}).
All results are the mean $\pm$ std over three seeds.

\paragraph{Results.}
On pure-solvent external transfer (Table~\ref{tab:pure_distill_results}), distillation raises Leeds RMSE from $0.91$ to $0.95$ and SolProp RMSE from $0.77$ to $0.97$, while $\%|\Delta\log S|\leq 1$ drops by under $6$ points on both datasets.
On mixed-solvent prediction (Table~\ref{tab:mixed_distill_results}), the frozen distilled encoder reaches $0.123 \pm 0.013$ on the random split and $0.457 \pm 0.110$ on the pair-disjoint split, compared with $0.095 \pm 0.008$ and $0.440 \pm 0.010$ for PRISM; the mean-error gap is small but the run-to-run standard deviation on the pair-disjoint split widens by roughly an order of magnitude.
Fine-tuning the distilled encoder does not close this gap.
On single-molecule downstream tasks (Table~\ref{tab:single_molecule_distill_results}), distillation degradation is substantially larger: FreeSolv RMSE rises from $1.770$ to $3.987$, BACE ROC-AUC drops from $89.4$ to $75.08$, and BBBP from $94.6$ to $85.80$.
The contrast across the three settings supports the reading in Section~\ref{sec:ablation}: solubility tasks survive relational compression because they depend on neighborhood structure that the student inherits, while single-molecule endpoints rely more heavily on coordinate-level or branch-specific signals that purely relational distillation does not preserve.

\begin{table}[h]
\centering
\caption{
\textbf{Pure-solvent external transfer with distilled vs.\ normal encoder.}
RMSE is lower better, $\%|\Delta \log S|\leq 1$ is higher better.
}
\label{tab:pure_distill_results}
\begin{tabular}{lcccc}
\toprule
Dataset 
& \multicolumn{2}{c}{Distilled encoder} 
& \multicolumn{2}{c}{Normal encoder} \\
\cmidrule(lr){2-3} \cmidrule(lr){4-5}
& RMSE $\downarrow$ 
& $\%|\Delta \log S|\leq 1$ $\uparrow$
& RMSE $\downarrow$ 
& $\%|\Delta \log S|\leq 1$ $\uparrow$ \\
\midrule
Leeds 
& $0.95 \pm 0.0136$ 
& $73.93 \pm 0.62$ 
& $0.91 \pm 0.002$ 
& $75.7 \pm 0.9$ \\
SolProp 
& $0.97 \pm 0.0235$ 
& $77.24 \pm 1.11$ 
& $0.77 \pm 0.030$ 
& $82.3 \pm 0.8$ \\
\bottomrule
\end{tabular}
\end{table}

\begin{table}[h]
\centering
\caption{
\textbf{Mixed-solvent prediction with distilled vs.\ normal encoder, under random and pair-disjoint splits.}
RMSE is lower better, $\%|\Delta \log S|\leq 1$ is higher better.
}
\label{tab:mixed_distill_results}
\begin{tabular}{llcc}
\toprule
Encoder setting & Split 
& RMSE $\downarrow$ 
& $\%|\Delta \log S|\leq 1$ $\uparrow$ \\
\midrule
Frozen distilled encoder 
& Random 
& $0.123 \pm 0.013$ 
& $99.95 \pm 0.02$ \\
Frozen normal encoder 
& Random 
& $0.095 \pm 0.008$ 
& $100.0 \pm 0.0$ \\
Fine-tuned distilled encoder 
& Random 
& $0.1321 \pm 0.015$ 
& $99.96 \pm 0.03$ \\
\midrule
Frozen distilled encoder 
& Pair-disjoint 
& $0.457 \pm 0.110$ 
& $95.59 \pm 3.75$ \\
Frozen normal encoder 
& Pair-disjoint 
& $0.440 \pm 0.010$ 
& $96.6 \pm 0.9$ \\
Fine-tuned distilled encoder 
& Pair-disjoint 
& $0.462 \pm 0.026$ 
& $96.41 \pm 3.52$ \\
\bottomrule
\end{tabular}
\end{table}

\begin{table}[h]
\centering
\caption{
\textbf{Single-molecule downstream results with distilled vs.\ normal encoder, both fine-tuned.}
RMSE is lower better for regression, ROC-AUC is higher better for classification.
}
\label{tab:single_molecule_distill_results}
\begin{tabular}{lccc}
\toprule
Dataset & Metric & Distilled encoder & Normal encoder \\
\midrule
ESOL 
& RMSE $\downarrow$ 
& $0.9072 \pm 0.012$ 
& $0.772 \pm 0.002$ \\
FreeSolv 
& RMSE $\downarrow$ 
& $3.987 \pm 0.106$ 
& $1.770 \pm 0.09$ \\
Lipo 
& RMSE $\downarrow$ 
& $0.888 \pm 0.035$ 
& $0.538 \pm 0.02$ \\
BACE 
& ROC-AUC $\uparrow$ 
& $75.08 \pm 5.2$ 
& $89.4 \pm 0.4$ \\
BBBP 
& ROC-AUC $\uparrow$ 
& $85.80 \pm 3.2$ 
& $94.6 \pm 0.3$ \\
\bottomrule
\end{tabular}
\end{table}

\subsection{Single-Shared-Encoder Baseline on MoleculeNet}
\label{app:extra_ablation:shared_encoder}

PRISM pretrains one encoder branch per prior space.
A natural alternative is to use a single encoder and optimize it with the sum of all prior-alignment losses:
\begin{equation}
    \mathcal{L}_{\mathrm{shared}}
    =
    \sum_{s\in\mathcal{S}}
    \mathcal{L}^{(s)}_{\mathrm{align}}.
\end{equation}
This baseline tests whether separate prior-specialized branches are necessary, or whether one embedding space can absorb all prior geometries simultaneously.

Table~\ref{tab:abl_shared_encoder} compares the single-shared-encoder baseline with the multi-branch PRISM design.

\begin{table}[h]
\centering
\small
\setlength{\tabcolsep}{6pt}
\caption{\textbf{Single shared encoder versus multi-branch PRISM on MoleculeNet.}
The shared-encoder baseline uses one D-MPNN trained with the sum of all prior-alignment losses.
Classification metrics are ROC-AUC, and regression metrics are RMSE.}
\label{tab:abl_shared_encoder}
\resizebox{0.99\textwidth}{!}{%
\begin{tabular}{lccccc}
\toprule
Model & BBBP $\uparrow$ & BACE $\uparrow$ & ESOL $\downarrow$ & FreeSolv $\downarrow$ & Lipo $\downarrow$ \\
\midrule
Single shared encoder
& $0.851\pm0.032$
& $0.860\pm0.020$
& $0.9209\pm0.027$
& $2.040\pm0.197$
& $0.740\pm0.011$ \\
PRISM
& $\mathbf{0.946\pm0.003}$
& $\mathbf{0.894\pm0.004}$
& $\mathbf{0.772\pm0.002}$
& $\mathbf{1.770\pm0.093}$
& $\mathbf{0.538\pm0.018}$ \\
\bottomrule
\end{tabular}}
\end{table}

\paragraph{Discussion.}
The shared-encoder baseline is consistently weaker than the multi-branch design.
This suggests that heterogeneous prior geometries are difficult to absorb into a single embedding space through a simple summed alignment loss.
Separate branches allow each encoder to specialize in one chemically meaningful relational view before downstream fusion.

\subsection{Leave-One-Prior-Space-Out Ablations}
\label{app:extra_ablation:loso_prior}

We evaluate whether each prior branch contributes to downstream performance.
For each ablation, we remove one prior branch while keeping the same training and evaluation protocol.
Tables~\ref{tab:abl_spaces_cls} and~\ref{tab:abl_spaces_reg} report leave-one-prior-space-out results on MoleculeNet classification and regression tasks.
Table~\ref{tab:abl_spaces_solv} reports the corresponding ablations on pure-solvent and mixed-solvent solubility tasks.

\begin{table}[h]
\centering
\small
\setlength{\tabcolsep}{7pt}
\caption{\textbf{Leave-one-prior-space-out ablations on MoleculeNet classification.}
AUC is reported in the standard $[0,1]$ scale.}
\label{tab:abl_spaces_cls}
\begin{tabular}{lccc}
\toprule
Spaces used & BBBP & BACE & HIV \\
\midrule
all & $0.9456\pm0.0031$ & $0.8936\pm0.0047$ & $0.7768\pm0.0092$ \\
w/o FP & $0.9316\pm0.0031$ & $0.8846\pm0.0019$ & $0.7140\pm0.0107$ \\
w/o PEOE & $0.9380\pm0.0052$ & $0.8821\pm0.0013$ & $0.7249\pm0.0079$ \\
w/o SlogP & $0.9256\pm0.0023$ & $0.8839\pm0.0019$ & $0.7191\pm0.0045$ \\
w/o SMR & $0.9283\pm0.0085$ & $0.8768\pm0.0096$ & $0.7274\pm0.0091$ \\
\bottomrule
\end{tabular}
\vspace{-2mm}
\end{table}

\begin{table}[h]
\centering
\small
\setlength{\tabcolsep}{7pt}
\caption{\textbf{Leave-one-prior-space-out and single-prior ablations on MoleculeNet regression.}
Metrics are RMSE ($\downarrow$).}
\label{tab:abl_spaces_reg}
\begin{tabular}{lccc}
\toprule
Spaces used & ESOL & FreeSolv & Lipo \\
\midrule
all & $0.7722\pm0.0275$ & $1.7702\pm0.0444$ & $0.5380\pm0.0156$ \\
w/o FP & $0.8129\pm0.0175$ & $2.5011\pm0.2005$ & $0.5739\pm0.0099$ \\
w/o PEOE & $0.7826\pm0.0274$ & $2.2335\pm0.1533$ & $0.5563\pm0.0082$ \\
w/o SlogP & $0.8080\pm0.0046$ & $2.4074\pm0.1035$ & $0.5525\pm0.0007$ \\
w/o SMR & $0.7780\pm0.0085$ & $2.4312\pm0.1630$ & $0.5733\pm0.0106$ \\
\midrule
SMR only & $0.8247\pm0.0295$ & $2.4812\pm0.1673$ & $0.5736\pm0.0041$ \\
SlogP only & $0.7791\pm0.0079$ & $2.6472\pm0.4197$ & $0.5671\pm0.0071$ \\
PEOE only & $0.8311\pm0.0256$ & $2.4134\pm0.0292$ & $0.5764\pm0.0159$ \\
FP only & $0.7706\pm0.0112$ & $2.5058\pm0.2614$ & $0.5984\pm0.0026$ \\
\bottomrule
\end{tabular}
\vspace{-2mm}
\end{table}

\begin{table*}[h]
\centering
\small
\setlength{\tabcolsep}{5pt}
\caption{\textbf{Leave-one-prior-space-out ablations on solubility tasks.}
Metrics are RMSE ($\downarrow$) and $\%|\Delta\log S|\leq1$ (\%, $\uparrow$).
Pure-solvent results are evaluated on Leeds and SolProp.
Mixed-solvent results are evaluated under the solvent-pair-disjoint OOD split using the same frozen-encoder linear-mixing protocol as the main text.}
\label{tab:abl_spaces_solv}
\resizebox{0.99\textwidth}{!}{%
\begin{tabular}{lcccccc}
\toprule
& \multicolumn{3}{c}{RMSE ($\downarrow$)}
& \multicolumn{3}{c}{$\%|\Delta\log S|\leq1$ (\%, $\uparrow$)} \\
\cmidrule(lr){2-4}\cmidrule(lr){5-7}
Spaces used
& Leeds & SolProp & Mixed OOD
& Leeds & SolProp & Mixed OOD \\
\midrule
all
& $0.910\pm0.002$ & $0.770\pm0.030$ & $0.440\pm0.008$
& $75.7\pm0.8$ & $82.3\pm0.8$ & $96.6\pm0.9$ \\
w/o SMR
& $0.9660\pm0.002$ & $0.8040\pm0.008$ & $0.501\pm0.009$
& $73.2\pm0.3$ & $80.8\pm0.5$ & $89.4\pm1.5$ \\
w/o SlogP
& $0.9586\pm0.002$ & $0.8120\pm0.031$ & $0.472\pm0.010$
& $73.2\pm0.6$ & $78.2\pm2.1$ & $90.6\pm1.0$ \\
w/o PEOE
& $0.9580\pm0.002$ & $0.8065\pm0.027$ & $0.488\pm0.007$
& $73.2\pm0.8$ & $80.0\pm0.1$ & $91.7\pm0.7$ \\
w/o FP
& $0.9720\pm0.004$ & $0.8203\pm0.102$ & $0.637\pm0.024$
& $72.5\pm1.3$ & $79.3\pm0.9$ & $87.1\pm2.7$ \\
\bottomrule
\end{tabular}}
\vspace{-2mm}
\end{table*}

\paragraph{Discussion.}
Removing any prior branch generally weakens performance.
On MoleculeNet classification, the effect is especially clear on HIV, where removing any branch substantially lowers ROC-AUC.
For regression tasks, single-prior variants are generally weaker than the full multi-prior model on FreeSolv and Lipophilicity.
ESOL is a small exception: the FP-only variant obtains $0.7706\pm0.0112$, slightly below the all-space result $0.7722\pm0.0275$.
Because this difference is small and within the reported variability, we interpret it as evidence that ESOL can be well supported by a strong structural prior, rather than as evidence against multi-prior complementarity overall.
The solubility ablations show a clearer multi-prior pattern: removing individual prior spaces weakens external-transfer performance on pure-solvent solubility and degrades the frozen-encoder-branch mixed-solvent OOD setting under the same linear-mixing protocol used in the main text.
This supports the use of multiple chemically motivated prior spaces for relation-centric solubility prediction.

\paragraph{Note on ESOL reporting across tables.}
The ESOL numbers in Tables~\ref{tab:molnet_full}, \ref{tab:abl_spaces_reg}, and \ref{tab:fusion_full} are produced under different diagnostic protocols.
Table~\ref{tab:molnet_full} reports the main selected PRISM result.
Table~\ref{tab:abl_spaces_reg} reports leave-one-prior and single-prior ablations, where the ``all'' row repeats the corresponding ablation protocol rather than the exact main-table run.
Table~\ref{tab:fusion_full} isolates the fusion mechanism under a fixed fusion-mode comparison.
These values are therefore not intended to be numerically identical, although they support the same qualitative conclusion that PRISM is strong on ESOL.

\subsection{Fusion Strategy: Task-specific versus Instance-Wise Weights}
\label{app:extra_ablation:fusion}

PRISM supports both task-specific and instance-wise weighted fusion.
The task-specific variant uses a single learnable weight vector for all molecules in a task, whereas the instance-wise variant computes sample-dependent weights from the branch representations.
Table~\ref{tab:fusion_full} compares these two fusion strategies for classification, single-molecule regression, and pure-solvent solubility tasks.

We do not include mixed-solvent prediction in this fusion-mode comparison because the main mixed-solvent protocol freezes the pretrained PRISM encoders but keeps the pure-solvent-initialized fusion module trainable.
In that setting, the goal is not to compare static versus molecule-specific fusion modes, but to test whether frozen prior-specialized encoders can be compositionally reused through lightweight downstream fusion and composition-weighted pooling.

\begin{table}[h]
\centering
\small
\setlength{\tabcolsep}{6pt}
\caption{\textbf{Task-specific versus instance-wise weighted fusion.}
Classification reports ROC-AUC (\%, $\uparrow$); single-molecule regression reports RMSE ($\downarrow$);
pure-solvent solubility reports $\%|\Delta\log S|\leq1$ (\%, $\uparrow$) and RMSE ($\downarrow$).
The pure-solvent RMSE values follow the same external-transfer protocol as the main pure-solvent table.}
\label{tab:fusion_full}
\begin{tabular}{lcc}
\toprule
\textbf{Setting} & \textbf{Task-specific weights} & \textbf{Instance-wise weights} \\
\midrule
BBBP (AUC) & $94.6\pm0.3$ & $72.28\pm1.5$ \\
BACE (AUC) & $87.9\pm0.4$ & $80.91\pm0.4$ \\
HIV (AUC)  & $77.7\pm0.6$ & $73.2\pm0.8$ \\
\midrule
ESOL (RMSE) & $0.929\pm0.003$ & $0.782\pm0.002$ \\
FreeSolv (RMSE) & $2.069\pm0.007$ & $1.770\pm0.093$ \\
Lipo (RMSE) & $0.664\pm0.034$ & $0.538\pm0.018$ \\
\midrule
Leeds ($\%|\Delta\log S|\leq1$) & $75.7\pm0.8$ & $68.62\pm0.70$ \\
SolProp ($\%|\Delta\log S|\leq1$) & $82.3\pm0.8$ & $72.84\pm1.2$ \\
Leeds (RMSE) & $0.91\pm0.002$ & $1.070\pm0.003$ \\
SolProp (RMSE) & $0.77\pm0.030$ & $0.903\pm0.022$ \\
\bottomrule
\end{tabular}
\vspace{-2mm}
\end{table}

\paragraph{Discussion.}
Molecule-specific fusion helps single-molecule regression, where the dominant physicochemical factor can vary substantially across molecules.
In contrast, task-specific fusion is more stable for classification and solute--solvent prediction.
This suggests that relation-level or conditional prediction tasks can make sample-wise gating more prone to overfitting, especially when the target depends on interactions among multiple molecules rather than a single molecular property.

\subsection{Mixed-Solvent Head Ablation}
\label{app:extra_ablation:mixed_head}

The main mixed-solvent experiments use frozen PRISM encoders with a trainable pure-solvent-initialized fusion module and a simple composition-weighted head.
All variants use the same frozen PRISM encoders and train the fusion module together with the mixed-solvent prediction head.
The simple composition-weighted head performs best and has the lowest variance, suggesting that the main mixed-solvent gain comes from reusable PRISM-pretrained encoders and a symmetry-respecting mixture representation rather than from a more expressive mixture-specific head.

The three heads differ in how composition, temperature, and solvent representations are combined.
The weighted-sum head first forms the main-text solvent-mixture representation
\begin{equation}
    z_{\mathrm{mix}}=\alpha_1 z_1+\alpha_2z_2,
    \qquad
    \alpha_1+\alpha_2=1,
\end{equation}
and predicts from $[z_{\mathrm{sol}};z_{\mathrm{mix}};t]$.
This representation is permutation-invariant to swapping the two solvent identities together with their mole fractions.
The direct-concat head instead avoids forming an explicit mixture embedding: it maps the scalar composition and temperature into embeddings,
\begin{equation}
    e_\alpha=h_\alpha(\alpha_1),
    \qquad
    t=h_T(T),
\end{equation}
and predicts directly from the concatenated representation
\begin{equation}
    [z_{\mathrm{sol}};\,z_1;\,z_2;\,e_\alpha;\,t].
\end{equation}
Unlike the weighted-sum head, this direct concatenation is not permutation-invariant by construction, so the MLP must learn the solvent-order symmetry from data.
The attention interaction head explicitly decomposes the mixture representation into an ideal additive term and a non-ideal interaction residual.
It first forms the same ideal composition-weighted solvent embedding as the main-text head,
\begin{equation}
    z_{\mathrm{ideal}}
    =
    \alpha_1 z_1 + \alpha_2 z_2,
    \qquad \alpha_1+\alpha_2=1 .
\end{equation}
It then uses a cross-attention module over the two solvent representations to estimate a non-ideal solvent--solvent interaction term:
\begin{equation}
    z_{\mathrm{nonideal}}
    =
    \alpha_1\alpha_2\,
    \mathrm{CrossAttn}(z_1,z_2).
\end{equation}
The multiplicative factor $\alpha_1\alpha_2$ makes the residual vanish at pure-solvent endpoints and emphasizes it near mixed compositions, where non-ideal solvent--solvent effects are most relevant.
The final solvent-mixture representation is
\begin{equation}
    z_{\mathrm{mix}}
    =
    z_{\mathrm{ideal}} + z_{\mathrm{nonideal}},
\end{equation}
and the prediction head uses $[z_{\mathrm{sol}};z_{\mathrm{mix}};t]$ as input.
Compared with the linear weighted-sum head, this variant gives the model additional capacity to learn composition-dependent deviations from ideal additive mixing.

Table~\ref{tab:mixed_head_ablation} reports the mixed-solvent head ablation results.

\begin{table}[h]
\centering
\small
\setlength{\tabcolsep}{7pt}
\caption{\textbf{Mixed-solvent head ablation under the solvent-pair-disjoint OOD split.}
All variants use the same frozen PRISM encoders and train the fusion module together with the mixed-solvent prediction head.
The simple composition-weighted head performs best and has the lowest variance.}
\label{tab:mixed_head_ablation}
\begin{tabular}{lcc}
\toprule
Head design & RMSE $\downarrow$ & $\%|\Delta\log S|\leq1$ $\uparrow$ \\
\midrule
Direct concat head
& $0.550\pm0.031$ & $91.3\pm2.8$ \\
Attention interaction head
& $0.445\pm0.020$ & $94.3\pm0.8$ \\
Linear weighted-sum head
& $\mathbf{0.440\pm0.008}$ & $\mathbf{96.6\pm0.9}$ \\
\bottomrule
\end{tabular}
\vspace{-2mm}
\end{table}

\paragraph{Discussion.}
The direct-concat head underperforms the composition-weighted sum.
This result should be interpreted carefully: the gap may reflect both the value of forming a composition-weighted mixture embedding and the fact that the weighted-sum head enforces the natural permutation symmetry of binary solvent mixtures.
Thus, the experiment does not by itself prove that separated solvent embeddings are intrinsically worse.
Rather, it shows that, under the frozen-encoder protocol, a symmetry-respecting linear composition head is a stronger and more stable way to reuse PRISM solvent representations.
The attention interaction head is more expressive and achieves a similar RMSE to the weighted-sum head, but it has lower $\%|\Delta\log S|\leq1$ and larger variance.
The simple weighted-sum head is therefore preferred because it is both more stable and more aligned with the representation-transfer hypothesis tested in the main text.
Rather than relying on a task-specific interaction module to relearn solvent-pair effects, this head asks whether frozen PRISM-pretrained encoders, combined with lightweight trainable fusion, can be compositionally reused under an interpretable linear operation.

\subsection{Solvent Aggregation, Prior Fusion, and Warm-Start Initialization}
\label{app:mixed_aggregation_fusion}

We further examine how solvent aggregation, prior-branch fusion,
and warm-start initialization affect mixed-solvent OOD performance.
The preceding head comparison evaluates alternative prediction heads;
here, we additionally vary the fusion mechanism and its initialization
to distinguish their contributions.

\paragraph{Experimental setup.}
All configurations retain the frozen PRISM-pretrained molecular
encoders and train the downstream fusion or aggregation components
together with the mixed-solvent prediction head.
We consider two ways to aggregate solvent representations:
composition-weighted summation and concatenation.
We separately compare learned prior-branch fusion
(\emph{Fusion Head}) with concatenation of the prior-branch
representations (\emph{Branch Concatenation}).
Thus, solvent concatenation combines representations of different
solvent molecules, whereas branch concatenation combines
representations from different prior-specialized encoders.

For configurations with a fusion head, we evaluate both
pure-solvent warm-start initialization and random initialization.
Warm start initializes the fusion module from the pure-solvent
task before mixed-solvent training.
Removing warm start changes this initialization while retaining
the frozen PRISM-pretrained encoders.
It therefore tests the contribution of supervised pure-solvent
fusion initialization, rather than removing PRISM pretraining.
Results are reported as mean and standard deviation across five
solvent-pair-disjoint splits.

\begin{table*}[t]
\centering
\caption{
\textbf{Ablation of solvent aggregation, prior-branch fusion,
and warm-start initialization under solvent-pair-disjoint OOD.}
All configurations use frozen PRISM-pretrained encoders.
Warm start refers to initialization of the fusion module from
the pure-solvent task.
Results are mean $\pm$ standard deviation across five splits.
The first row is the default PRISM configuration.
The metric pct-in-1 is the percentage of predictions with
absolute error at most one log-solubility unit.
}
\label{tab:mixed_aggregation_fusion}
\small
\setlength{\tabcolsep}{7pt}
\resizebox{0.99\textwidth}{!}{%
\begin{tabular}{lllcc}
\toprule
Solvent aggregation
& Prior aggregation
& Warm start
& RMSE $\downarrow$
& pct-in-1 (\%) $\uparrow$ \\
\midrule
Weighted sum
& Fusion head
& Yes
& $\mathbf{0.438 \pm 0.030}$
& $\mathbf{96.6 \pm 0.9}$ \\

Weighted sum
& Fusion head
& No
& $0.442 \pm 0.080$
& $96.1 \pm 1.4$ \\

Weighted sum
& Branch concatenation
& No
& $0.525 \pm 0.255$
& $93.7 \pm 6.4$ \\
\midrule
Concatenation
& Fusion head
& Yes
& $0.550 \pm 0.051$
& $91.3 \pm 3.1$ \\

Concatenation
& Fusion head
& No
& $0.557 \pm 0.183$
& $91.2 \pm 7.1$ \\

Concatenation
& Branch concatenation
& No
& $0.570 \pm 0.311$
& $91.7 \pm 8.3$ \\
\bottomrule
\end{tabular}%
}
\end{table*}

\paragraph{Effects of aggregation and fusion.}
Table~\ref{tab:mixed_aggregation_fusion} supports the use of
composition-weighted solvent aggregation and learned prior fusion.
With warm start disabled, replacing learned prior fusion with
branch concatenation under weighted-sum solvent aggregation
increases mean RMSE from $0.442$ to $0.525$ and its standard
deviation from $0.080$ to $0.255$.
Under the same initialization setting, retaining learned prior
fusion but replacing weighted-sum solvent aggregation with
concatenation increases mean RMSE from $0.442$ to $0.557$.
These comparisons hold warm-start status fixed and therefore
separate the architectural effects from pure-solvent initialization.

The solvent-aggregation comparison should nevertheless be
interpreted with the symmetry distinction in mind.
Composition-weighted summation is invariant to jointly swapping
solvent identities and their mole fractions, whereas direct
concatenation does not enforce this invariance by construction.
The observed advantage therefore supports the weighted-sum
design as a whole, rather than isolating composition weighting
from permutation symmetry.

\paragraph{Effect of warm-start initialization.}
Warm start changes mean RMSE only slightly but reduces the
observed variation across splits under both solvent aggregation
schemes.
With weighted-sum aggregation, RMSE changes from
$0.442 \pm 0.080$ without warm start to
$0.438 \pm 0.030$ with warm start.
With concatenation, it changes from
$0.557 \pm 0.183$ to $0.550 \pm 0.051$.
The standard deviation of pct-in-1 also decreases in both cases.
Thus, the main benefit observed for warm start in this experiment
is improved consistency across splits, with a smaller change
in average predictive accuracy.

The default configuration achieves the lowest mean RMSE and
highest mean pct-in-1 among the evaluated configurations.
Its similar mean RMSE without warm start suggests that
pure-solvent fusion initialization alone does not explain
its OOD performance.
These results support composition-weighted aggregation and
learned prior fusion for reusing frozen PRISM representations,
with warm start providing an additional reduction in observed
split-to-split variability.

\subsection{Extended Gaussian Random Prior Control on Classification Tasks}
\label{app:extra_ablation:gaussian}

The main text reports the Gaussian-random-prior control on the three solubility-related MoleculeNet regression tasks.
Here we provide the corresponding classification-side control, covering MoleculeNet classification benchmarks and two additional bioactivity datasets used in the prior-generality analysis.
Each molecule is assigned a 10-dimensional i.i.d.\ Gaussian vector sampled from $\mathcal{N}(0,1)$, independent of molecular structure and labels.
This random vector defines a structureless prior space for PRISM pretraining.

Table~\ref{tab:abl_gaussian_random_classification} reports the results.
Across all classification tasks, Gaussian-prior pretraining degrades performance relative to the non-pretrained D-MPNN baseline, confirming that PRISM depends on prior quality.
However, performance does not abruptly collapse, suggesting that the D-MPNN backbone and downstream supervised training can still recover part of the molecular signal even under an uninformative relational target.

\begin{table}[h]
\centering
\small
\setlength{\tabcolsep}{6pt}
\caption{\textbf{Gaussian random prior control on classification and bioactivity tasks.}
Classification performance is reported as ROC-AUC (\%, $\uparrow$).
The random prior degrades performance relative to supervised D-MPNN but does not cause abrupt collapse.}
\label{tab:abl_gaussian_random_classification}
\begin{tabular}{lccc}
\toprule
Dataset 
& Energy (Gaussian) 
& PRISM + Gaussian prior 
& D-MPNN w/o pretraining \\
\midrule
BBBP 
& 0.356 
& $83.07\pm3.04$ 
& $91.9\pm3.0$ \\
BACE 
& 0.498 
& $71.17\pm9.70$ 
& $85.2\pm0.6$ \\
HIV 
& 0.067 
& $72.29\pm1.98$ 
& $77.1\pm0.5$ \\
CYP2D6 
& 0.025 
& $66.34\pm0.75$ 
& $72.7\pm0.37$ \\
KOR 
& 0.406 
& $74.42\pm0.22$ 
& $82.0\pm0.18$ \\
\bottomrule
\end{tabular}
\end{table}

\paragraph{Discussion.}
The Gaussian prior is intentionally uninformative.
The additional CYP2D6 and KOR results show the same trend as the main-text Gaussian-prior control: performance degrades compared with the supervised D-MPNN baseline, but the model does not collapse.
This supports the interpretation that PRISM depends on prior quality, and that poor priors can be identified by label-energy diagnostics.

\subsection{Dynamic Prior Selection via Label Energy Ratio}
\label{app:extra_ablation:dynamic_selection}

A ratio $\rho_s < 1$ indicates that prior $s$ induces smoother label neighborhoods than a random prior and is selected for PRISM pretraining.
We test this task-level prior-selection rule using Pharm2D, fingerprint, and VSA priors on MoleculeNet classification benchmarks and two additional bioactivity datasets, CYP2D6 and KOR.

Table~\ref{tab:abl_dynamic_prior_selection} summarizes the selected priors and downstream results.

\begin{table*}[h]
\centering
\small
\setlength{\tabcolsep}{5pt}
\caption{\textbf{Dynamic prior selection using label-energy ratio.}
A ratio below one indicates that the candidate prior is more label-smooth than a Gaussian random prior.
Classification performance is reported as ROC-AUC (\%, $\uparrow$).}
\label{tab:abl_dynamic_prior_selection}
\resizebox{0.99\textwidth}{!}{%
\begin{tabular}{lcccp{3.0cm}ccc}
\toprule
Dataset
& $\rho_{\mathrm{Pharm2D}}$
& $\rho_{\mathrm{FP}}$
& $\rho_{\mathrm{VSA}}$
& Selected priors
& PRISM with selected priors
& PRISM with all priors
& D-MPNN w/o pretraining \\
\midrule
BBBP
& 0.166
& 0.096
& 0.109
& Pharm2D + FP + VSA
& $95.8\pm0.6$
& $95.8\pm0.6$
& $91.9\pm3.0$ \\
BACE
& 0.264
& 0.109
& 0.118
& Pharm2D + FP + VSA
& $89.3\pm0.9$
& $89.3\pm0.9$
& $85.2\pm0.6$ \\
HIV
& 0.73
& 2.87
& 2.63
& Pharm2D
& $78.66\pm2.00$
& $76.35\pm1.90$
& $77.1\pm0.50$ \\
CYP2D6
& 0.96
& 1.05
& 1.06
& Pharm2D
& $73.31\pm1.02$
& $42.08\pm9.35$
& $72.7\pm0.37$ \\
KOR
& 0.48
& 0.47
& 0.56
& Pharm2D + FP + VSA
& $88.91\pm0.29$
& $88.91\pm0.29$
& $82.0\pm0.18$ \\
\bottomrule
\end{tabular}}
\end{table*}

\paragraph{Discussion.}
This experiment reframes prior choice as a measurable task-dependent selection problem.
For BBBP, BACE, and KOR, all three candidate priors satisfy $\rho_s<1$, so Pharm2D, FP, and VSA are jointly selected.
For HIV and CYP2D6, only Pharm2D satisfies the selection criterion, while FP and VSA have $\rho_s>1$ and are excluded.
Selected-prior PRISM consistently outperforms the non-pretrained D-MPNN baseline across the five classification tasks.

Comparing selected-prior PRISM with PRISM trained over all candidate priors further shows why prior selection matters.
When all candidate priors are label-compatible, as in BBBP, BACE, and KOR, selected-prior and all-prior PRISM coincide.
When label-incompatible priors are forcibly included, as in CYP2D6 and HIV, downstream performance can degrade.
This pattern shows that prior quality is not only quantifiable but also actionable: the label-energy ratio identifies useful prior spaces before pretraining and filters out priors that may harm representation learning.

\subsection{Compatibility with 3D Backbones and 3D Priors}
\label{app:extra_ablation:schnet}

The default PRISM implementation uses a 2D D-MPNN backbone and 2D prior spaces.
To test whether PRISM is tied to 2D graph encoders, we replace the backbone with SchNet, a standard 3D geometry-aware molecular encoder.
RDKit-generated conformers are used as input.
We consider three settings:
(i) randomly initialized SchNet;
(ii) PRISM-pretrained SchNet using Morgan FP as the prior space; and
(iii) PRISM-pretrained SchNet using a compact 3D descriptor prior.
For both PRISM settings, pretraining uses 50k randomly sampled ZINC molecules.

Table~\ref{tab:abl_schnet} reports this 3D-backbone compatibility experiment.

\begin{table}[h]
\centering
\small
\setlength{\tabcolsep}{5pt}
\caption{\textbf{PRISM with a 3D SchNet backbone.}
Classification metrics are ROC-AUC, and regression metrics are RMSE.
The 3D prior uses RDKit shape descriptors covering global shape, anisotropy, compactness, spatial extent, and principal-axis mass distribution.}
\label{tab:abl_schnet}
\resizebox{0.99\textwidth}{!}{%
\begin{tabular}{lccccc}
\toprule
Model & BBBP $\uparrow$ & BACE $\uparrow$ & ESOL $\downarrow$ & FreeSolv $\downarrow$ & Lipo $\downarrow$ \\
\midrule
Random SchNet
& $0.830\pm0.033$
& $0.807\pm0.022$
& $0.716\pm0.045$
& $1.318\pm0.361$
& $0.775\pm0.032$ \\
PRISM + SchNet (FP prior)
& $0.835\pm0.048$
& $\mathbf{0.830\pm0.043}$
& $\mathbf{0.643\pm0.032}$
& $\mathbf{1.062\pm0.150}$
& $\mathbf{0.686\pm0.038}$ \\
PRISM + SchNet (3D descriptor prior)
& $\mathbf{0.844\pm0.022}$
& $0.807\pm0.070$
& $0.704\pm0.079$
& $1.756\pm0.262$
& $0.801\pm0.010$ \\
\bottomrule
\end{tabular}}
\end{table}

\paragraph{Discussion.}
These preliminary results suggest that PRISM is not tied to a 2D graph backbone.
With SchNet as a 3D geometry-aware encoder, PRISM pretraining with the FP prior improves BACE, ESOL, FreeSolv, and Lipo over random SchNet, while the compact 3D descriptor prior improves BBBP.
The 3D descriptor prior is not uniformly stronger than the FP prior, likely because it summarizes only coarse global shape and does not encode detailed local chemistry.
The goal of this ablation is compatibility rather than state-of-the-art 3D performance.
The default results in this paper remain centered on 2D graph-based property prediction, but the PRISM objective itself is encoder-agnostic: it only requires a prior-induced relational structure and an embedding-induced relational structure to align.

\subsection{Effect of Pretraining Data Size}
\label{app:extra_ablation:scale}

Finally, we evaluate whether PRISM benefits from more unlabeled pretraining data.
Table~\ref{tab:pretrain_size_reg} compares pretraining on 25k, 50k, and 250k molecules before fine-tuning on regression tasks.

\begin{table}[h]
\centering
\small
\setlength{\tabcolsep}{6pt}
\caption{\textbf{Effect of pretraining data size on regression tasks.}
Metrics are RMSE ($\downarrow$, mean $\pm$ std).}
\label{tab:pretrain_size_reg}
\begin{tabular}{lccc}
\toprule
Pretrain size & ESOL & FreeSolv & Lipo \\
\midrule
25k
& $1.011\pm0.031$
& $2.029\pm0.113$
& $0.698\pm0.011$ \\
50k
& $0.933\pm0.009$
& $2.138\pm0.025$
& $0.664\pm0.003$ \\
250k
& $\mathbf{0.772\pm0.002}$
& $\mathbf{1.770\pm0.093}$
& $\mathbf{0.538\pm0.018}$ \\
\bottomrule
\end{tabular}
\end{table}

\paragraph{Discussion.}
Performance improves substantially when scaling from 25k or 50k molecules to 250k molecules.
This suggests that PRISM benefits from a larger unlabeled corpus because the encoder observes a broader range of prior-induced neighborhoods during pretraining.
The result is consistent with the role of PRISM as a relational pretraining objective rather than a task-specific regularizer.

\clearpage
\section{Interpretability and Diagnostic Analyses}
\label{app:interp}

This section provides additional analyses for interpreting PRISM.
We focus on three complementary questions.
First, we test whether fusion weights are directionally faithful to perturbation-based importance.
Second, we visualize whether PRISM pretraining preserves prior-induced molecular geometry.
Third, we test whether and where prior information remains recoverable from frozen PRISM-pretrained embeddings.
The final probe is interpreted as a diagnostic rather than as a full proof of recoverability for every prior space.

\subsection{Faithfulness Tests for Fusion Weights}
\label{app:interp:faithfulness}

PRISM fusion weights are intended to represent soft preferences over prior spaces.
Because the prior spaces are chemically related and partially redundant, these weights should not be interpreted as hard causal assignments.
Nevertheless, a meaningful fusion mechanism should be directionally consistent with perturbation-based importance tests.
We therefore evaluate faithfulness separately for task-specific and molecule-specific fusion.

\paragraph{Task-specific fusion.}
For tasks with task-specific weights, leave-one-prior-space-out evaluation provides a direct faithfulness test.
If a prior branch receives a larger learned weight, removing that branch should generally cause a larger performance drop.
Across Leeds/SolProp and BBBP/BACE/HIV, the rank agreement between learned weights and leave-one-space-out importance is consistently positive, with Spearman $\rho=0.80$ across the evaluated settings.
This suggests that task-specific fusion weights capture meaningful task-level prior-space importance.

\paragraph{Molecule-specific fusion.}
For molecule-specific fusion, leave-one-space-out is less direct because the weights vary across molecules.
We therefore perform a suppression test on ESOL, FreeSolv, and Lipophilicity.
For each test molecule, we compare the degradation caused by suppressing the top-weighted branch with the degradation caused by suppressing the bottom-weighted branch.
Suppressing the top-weighted branch more often causes larger degradation, with probabilities of 0.64, 0.76, and 0.64 on ESOL, FreeSolv, and Lipophilicity, respectively.

Table~\ref{tab:fusion_faithfulness} summarizes both faithfulness tests.

\begin{table}[h]
\centering
\small
\setlength{\tabcolsep}{6pt}
\caption{\textbf{Faithfulness tests for PRISM fusion weights.}
Task-specific weights are evaluated by rank agreement with leave-one-space-out importance.
Instance-wise weights are evaluated by top-weight versus bottom-weight suppression.}
\label{tab:fusion_faithfulness}
\begin{tabular}{lll}
\toprule
Fusion setting & Faithfulness test & Result \\
\midrule
Task-specific weights
& Weight--LOSO importance rank agreement
& Spearman $\rho=0.80$ \\
Instance-wise, ESOL
& Top-vs-bottom suppression probability
& 0.64 \\
Instance-wise, FreeSolv
& Top-vs-bottom suppression probability
& 0.76 \\
Instance-wise, Lipo
& Top-vs-bottom suppression probability
& 0.64 \\
\bottomrule
\end{tabular}
\end{table}

\paragraph{Interpretation.}
These results provide directional evidence that the learned weights are meaningful rather than arbitrary.
For task-specific fusion, higher-weighted spaces tend to be more important under removal.
For molecule-specific fusion, the branch receiving the largest weight for a molecule is more likely to be the branch whose suppression causes the larger error increase.
However, the probabilities are not close to one, which is expected because the four default prior spaces are partially interdependent.
Thus, fusion weights should be interpreted as soft preference reweighting over complementary chemical views, not as exclusive hard routing or causal attribution.

\subsection{Training Diagnostics: Geometry Preservation}
\label{app:interp:tsne}

We next visualize whether PRISM pretraining preserves the geometry induced by prior spaces.
Figure~\ref{fig:train_diag}a compares two-dimensional t-SNE visualizations of the original VSA prior spaces and the corresponding PRISM embedding spaces.
We show three VSA priors: SMR\_VSA, SlogP\_VSA, and PEOE\_VSA.

The purpose of this diagnostic is not to use t-SNE as a quantitative metric.
Instead, it provides a qualitative sanity check: molecules that form local neighborhoods in the prior-feature space should remain locally organized in the learned embedding space.
The highlighted anchors show representative local regions used for the anchor-wise distribution analysis in Appendix~\ref{app:interp:anchor_match}.

\begin{figure*}[h]
  \centering
  \includegraphics[width=0.99\textwidth]{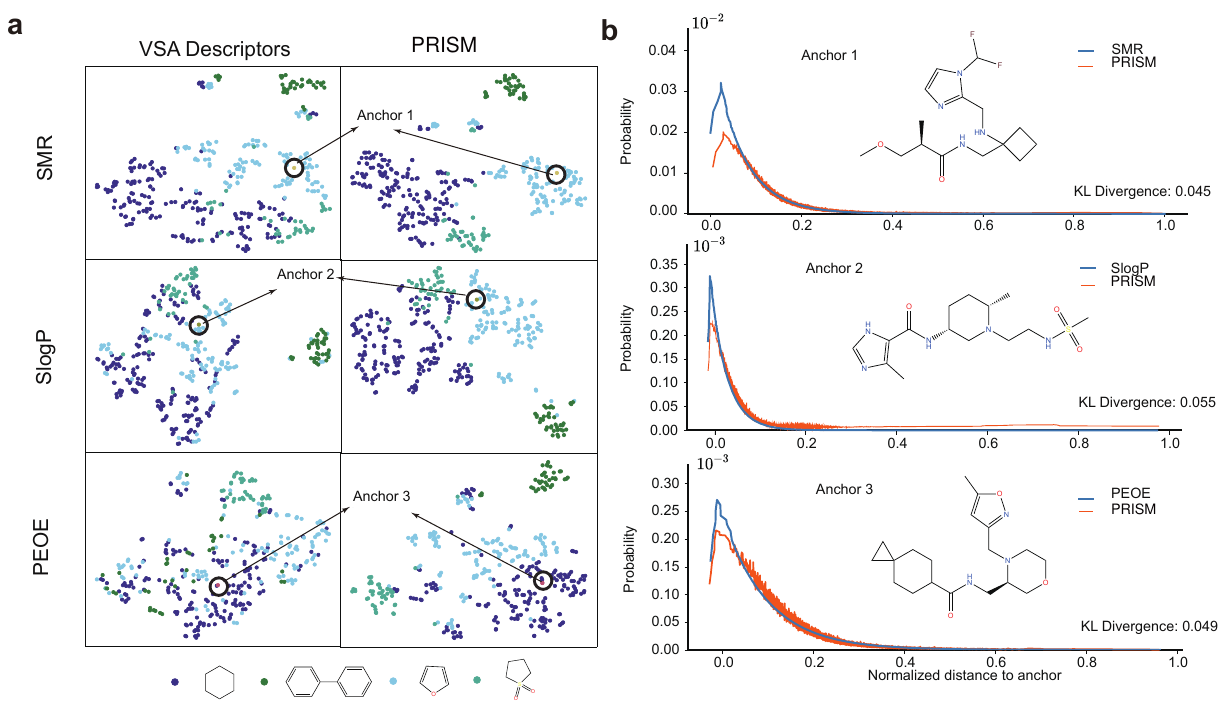}
  \caption{
  \textbf{Training diagnostics for PRISM prior-space alignment.}
  \textbf{(a)} Two-dimensional t-SNE visualizations comparing the prior-feature geometry and the PRISM embedding geometry for SMR\_VSA, SlogP\_VSA, and PEOE\_VSA.
  Representative anchors are highlighted to illustrate local-neighborhood preservation.
  \textbf{(b)} Anchor-wise neighborhood distribution matching.
  For each anchor, we compare the distance-to-anchor distribution induced by the prior space and by the learned embedding.
  The small KL divergences indicate that PRISM pretraining preserves prior-induced stochastic neighborhoods.
  The KL values in panel (b) are absolute anchor-wise KL divergences, whereas Table~\ref{tab:loss_ablation} reports normalized KL ratios for checkpoint-level comparison.
  }
  \label{fig:train_diag}
\end{figure*}

\paragraph{Interpretation.}
The learned embedding spaces qualitatively preserve the local organization of the corresponding prior spaces.
This supports the basic mechanism of PRISM: the encoder is not merely trained with an auxiliary scalar loss, but is shaped to reproduce prior-induced relational structure.

\subsection{Anchor-wise Neighborhood Distribution Matching}
\label{app:interp:anchor_match}

Figure~\ref{fig:train_diag}b provides a more direct diagnostic of the PRISM objective.
For each highlighted anchor, we compare the normalized distance-to-anchor distribution in the prior space with the corresponding distribution in the learned embedding space.
The three examples correspond to SMR\_VSA, SlogP\_VSA, and PEOE\_VSA.

The resulting absolute anchor-wise KL divergences are small:
\begin{equation}
    \mathrm{KL}_{\mathrm{SMR}} = 0.045,
    \qquad
    \mathrm{KL}_{\mathrm{SlogP}} = 0.055,
    \qquad
    \mathrm{KL}_{\mathrm{PEOE}} = 0.049.
\end{equation}
These are absolute KL values for representative anchor-wise neighborhood distributions, not the normalized KL ratios used in the checkpoint-level objective comparison in Table~\ref{tab:loss_ablation}.
They indicate that the learned embedding neighborhoods closely track the prior-induced neighborhoods for representative anchors.

\paragraph{Interpretation.}
This diagnostic directly visualizes the training target.
PRISM aligns distributions over neighbors, not individual descriptor values.
The close overlap between prior-space and embedding-space distance distributions therefore supports that the pretraining objective enforces distribution-level neighborhood matching.
This also helps explain why PRISM can transfer descriptor-induced relational information without feeding descriptors as downstream model inputs.

\subsection{Prior Recoverability from Frozen Embeddings}
\label{app:prior_recoverability}

Beyond neighborhood matching, we test prior recoverability by
freezing the pre-trained encoder and training a lightweight
two-layer MLP decoder to map embeddings back to prior features
(VSA coordinates or fingerprint bits).
We conduct this probe on ZINC250k under an 8:1:1 scaffold split.
Table~\ref{tab:prior_probe} reports the test errors.

\begin{table}[t]
\centering
\small
\setlength{\tabcolsep}{6pt}
\caption{\textbf{Prior-recoverability probe on ZINC250k
(8:1:1 scaffold split).}
We freeze the pre-trained encoder and train a two-layer MLP
decoder to reconstruct prior features from embeddings.
For VSA priors, we report MSE; for fingerprints, we report
$1-\mathrm{Tanimoto}$ (lower is better).
Metrics are computed on the test split.}
\label{tab:prior_probe}
\begin{tabular}{lc}
\toprule
\textbf{Prior space} & \textbf{Probe error} $\downarrow$ \\
\midrule
SMR\_VSA (MSE)   & 0.476 \\
SlogP\_VSA (MSE) & 0.767 \\
PEOE\_VSA (MSE)  & 1.450 \\
FP ($1-\mathrm{Tanimoto}$) & 0.377 \\
\bottomrule
\end{tabular}
\vspace{-2mm}
\end{table}

\paragraph{Discussion.}
Table~\ref{tab:prior_probe} provides quantitative evidence
{of prior recoverability in PRISM embeddings}.
The learned embeddings
{support reconstruction of the injected priors
using a lightweight decoder}.
{This probe complements the neighborhood-matching
diagnostics by evaluating prior-relevant information accessible
from individual molecular embeddings while keeping the encoder
fixed.}

\clearpage
\section{Limitations and Broader Impacts}
\label{app:limitations_impact}

\paragraph{Limitations.}
PRISM depends on the quality of the chosen prior spaces: when a prior induces label-irrelevant neighborhoods, relational supervision can degrade downstream performance (Gaussian-prior control, Appendix~\ref{app:extra_ablation:gaussian}).
This dependence is well-aligned with our use-inspired scope: chemistry provides mature prior families, and label energy turns prior selection from an empirical question into a diagnosable one.
Compatibility with 3D priors and 3D-aware backbones is verified in Appendix~\ref{app:extra_ablation:schnet}; we keep the main results centered on 2D inputs because 3D pretraining brings orthogonal design choices (notably conformer generation) outside our scope.
On extreme temperature extrapolation, Fastprop~\cite{attia_data-driven_2025} remains stronger because it explicitly adds a $d\log S/dT$ gradient term as an auxiliary loss; we deliberately omit such a thermodynamic constraint to keep the empirical results attributable to PRISM's prior-induced relational supervision alone, since thermodynamic regularization is itself a separate form of domain-knowledge injection that would confound the comparison.
The two are orthogonal and can be combined in future work, but disentangling them in the present study is what allows us to characterize what relational supervision contributes on its own. 
The mixed-solvent evidence is based on one curated binary-solvent benchmark; extending the evaluation to additional mixture datasets and cross-dataset mixture transfer remains an important next step.

\paragraph{Broader impacts.}
The intended positive impact is to improve molecular property prediction under distribution shifts, which may support more efficient solvent selection, molecular screening, and experimental prioritization.
Potential negative impacts are indirect: improved molecular prediction tools could be misused in harmful chemical design workflows if combined with inappropriate objectives or unsafe candidate-generation pipelines.
We release only representation-learning and prediction code, not generative design tools, and encourage users to follow standard safety, licensing, and domain-review practices when applying the method.

\paragraph{Assets and licenses.}
We use publicly available molecular datasets and open-source cheminformatics and machine-learning libraries.
Table~\ref{tab:assets_licenses} summarizes the main external assets, their roles in this work, and the license or access terms reported by their original sources.
The PRISM implementation, training and evaluation scripts, and usage instructions are publicly available at \url{https://github.com/cxy0206/PRISM}. The repository includes a code license and documentation for data preparation.

\begin{table}[h]
\centering
\small
\caption{\textbf{External assets used in this work.}}
\label{tab:assets_licenses}
\setlength{\tabcolsep}{5pt}
\renewcommand{\arraystretch}{1.12}
\begin{tabularx}{0.98\textwidth}{p{2.6cm}YY}
\toprule
Asset & Role & License / access terms \\
\midrule
RDKit
& Molecular parsing, descriptors, fingerprints, and featurization
& BSD-style open-source license \\

PyTorch
& Neural-network implementation and training
& BSD-style open-source license \\

MoleculeNet / DeepChem
& Single-molecule property benchmarks and scaffold-split evaluation
& Open-source benchmark resources; DeepChem/MoleculeNet code is released under the MIT license \\

BigSolDB
& Training source for pure-solvent solubility transfer
& Creative Commons Attribution 4.0 International (CC BY 4.0) \\

SolProp
& External pure-solvent solubility test set
& Creative Commons Attribution 4.0 International (CC BY 4.0) \\

Leeds
& External pure-solvent solubility test set
& Open Zenodo record; license marked as ``Other (Open)'' \\

Bao et al.\ binary-solvent dataset / CheMixHub drug-solubility task
& Mixed-solvent solubility benchmark
& CC BY 4.0 for the CheMixHub drug-solubility task; original source cited in the paper \\

ZINC250k subset
& Unlabeled molecular pretraining corpus
& Public molecular dataset for research use; source and preprocessing details will be provided in the released repository \\
\bottomrule
\end{tabularx}
\vspace{-2mm}
\end{table}


\clearpage
\section*{NeurIPS Paper Checklist}

\begin{enumerate}

\item {\bf Claims}
    \item[] Question: Do the main claims made in the abstract and introduction accurately reflect the paper's contributions and scope?
    \item[] Answer: \answerYes{}.
    \item[] Justification: The abstract and introduction state the paper's main contributions: PRISM as a label-free molecular pretraining framework that turns expert-defined chemical prior spaces into relational teachers; prior-specific encoder branches aligned through soft neighborhood distribution matching; solubility-centered OOD evaluation across scaffold, solute, solvent, and solvent-pair shifts; and label energy as a diagnostic for prior--task compatibility. The claims are supported by the main experiments in Section~\ref{sec:experiments}, the ablations in Section~\ref{sec:ablation}, and the additional analyses in Appendix~\ref{app:extra_results}, Appendix~\ref{app:extra_ablation}, and Appendix~\ref{app:interp}. Scope and limitations are further qualified in Appendix~\ref{app:limitations_impact}.

\item {\bf Limitations}
    \item[] Question: Does the paper discuss the limitations of the work performed by the authors?
    \item[] Answer: \answerYes{}.
    \item[] Justification: The paper discusses limitations related to prior--task mismatch, uninformative-prior degradation, the low-cost prior-bank scope, the 2D-centered default implementation, missing 3D factors for tasks such as FreeSolv, MoleculeNet split sensitivity, three-seed evaluation, the single mixed-solvent benchmark, and extreme temperature extrapolation. These limitations are discussed in Section~\ref{sec:conclusion}, Appendix~\ref{app:datasets:metrics}, Appendix~\ref{app:extra_results:slice}, Appendix~\ref{app:extra_ablation:gaussian}, Appendix~\ref{app:extra_ablation:schnet}, and Appendix~\ref{app:limitations_impact}.

\item {\bf Theory assumptions and proofs}
    \item[] Question: For each theoretical result, does the paper provide the full set of assumptions and a complete (and correct) proof?
    \item[] Answer: \answerYes{}.
    \item[] Justification: The theoretical analysis is provided in Appendix~\ref{app:theory}. The appendix formally defines the alignment objectives, relates PRISM-KL to InfoNCE and rank/ListMLE-style objectives, proves the label-energy $k$NN risk bound, and gives a finite-set no-worse approximation result under an explicit recoverability assumption. The assumptions, statements, and proofs are stated in the appendix and cross-referenced from the main text.

\item {\bf Experimental result reproducibility}
    \item[] Question: Does the paper fully disclose all the information needed to reproduce the main experimental results of the paper to the extent that it affects the main claims and/or conclusions of the paper (regardless of whether the code and data are provided or not)?
    \item[] Answer: \answerYes{}.
    \item[] Justification: The paper specifies datasets, preprocessing assumptions, split protocols, leakage controls, unlabeled pretraining protocol, evaluation metrics, model architectures, prior spaces, pretraining objective, downstream heads, hyperparameters, checkpoint selection, repeated-run protocol, and compute resources in Appendix~\ref{app:datasets} and Appendix~\ref{app:impl}. Implementation code and instructions for pretraining, fine-tuning, and evaluation are publicly available at \url{https://github.com/cxy0206/PRISM}.

\item {\bf Open access to data and code}
    \item[] Question: Does the paper provide open access to the data and code, with sufficient instructions to faithfully reproduce the main experimental results, as described in supplemental material?
    \item[] Answer: \answerYes{}.
    \item[] Justification: The PRISM implementation, training and evaluation scripts, and usage instructions are publicly available at \url{https://github.com/cxy0206/PRISM}. Data access, preprocessing, split construction, leakage controls, and the unlabeled pretraining protocol are described in Appendix~C, with dataset citations and license or access information summarized in Appendix~J.

\item {\bf Experimental setting/details}
    \item[] Question: Does the paper specify all the training and test details (e.g., data splits, hyperparameters, how they were chosen, type of optimizer) necessary to understand the results?
    \item[] Answer: \answerYes{}.
    \item[] Justification: Section~\ref{sec:experiments} summarizes the experimental settings, and Appendix~\ref{app:datasets} and Appendix~\ref{app:impl} provide detailed dataset descriptions, split construction, leakage checks, metrics, molecular featurization, model heads, optimizer settings, hyperparameters, checkpoint selection, and compute resources. Additional ablation-specific protocols are described in Appendix~\ref{app:extra_ablation}.

\item {\bf Experiment statistical significance}
    \item[] Question: Does the paper report error bars suitably and correctly defined or other appropriate information about the statistical significance of the experiments?
    \item[] Answer: \answerYes{}.
    \item[] Justification: Unless otherwise specified, the main experimental results report mean $\pm$ standard deviation over three seeds. Appendix~\ref{app:datasets:metrics} states that split-dependent evaluations regenerate the corresponding split for each seed, so the reported standard deviations capture both initialization and split resampling when applicable. Table~5 and the ablation in Appendix~H.7 use five seeds.

\item {\bf Experiments compute resources}
    \item[] Question: For each experiment, does the paper provide sufficient information on the computer resources (type of compute workers, memory, time of execution) needed to reproduce the experiments?
    \item[] Answer: \answerYes{}.
    \item[] Justification: Appendix~\ref{app:impl:compute} reports the hardware, wall-clock time, batch size, and memory usage for PRISM pretraining and downstream training. The paper notes that prior branches can be pretrained independently, gives the $O(KB^2)$ minibatch relational-construction cost, and states that the reported pipeline is feasible on consumer-grade GPUs.
    
\item {\bf Code of ethics}
    \item[] Question: Does the research conducted in the paper conform, in every respect, with the NeurIPS Code of Ethics \url{https://neurips.cc/public/EthicsGuidelines}?
    \item[] Answer: \answerYes{}.
    \item[] Justification: The work uses public molecular datasets and standard molecular property prediction benchmarks, does not involve human subjects, and does not release high-risk generative models. Appendix~\ref{app:limitations_impact} discusses potential misuse and responsible-use considerations. The supplementary repository will be anonymized at submission.

\item {\bf Broader impacts}
    \item[] Question: Does the paper discuss both potential positive societal impacts and negative societal impacts of the work performed?
    \item[] Answer: \answerYes{}.
    \item[] Justification: Appendix~\ref{app:limitations_impact} discusses positive impacts for molecular property prediction under distribution shifts, including solvent selection, molecular screening, and experimental prioritization. It also discusses indirect dual-use risks if molecular prediction tools are embedded into unsafe chemical design workflows, and states that the released assets are representation-learning and prediction tools rather than generative design tools.
    
\item {\bf Safeguards}
    \item[] Question: Does the paper describe safeguards that have been put in place for responsible release of data or models that have a high risk for misuse (e.g., pre-trained language models, image generators, or scraped datasets)?
    \item[] Answer: \answerNA{}.
    \item[] Justification: The paper does not release high-risk generative models, scraped internet datasets, or models intended to synthesize harmful content. The planned release consists of property-prediction code, processed split files, configuration files, and reproduction instructions. The paper also encourages standard safety, licensing, and domain-review practices when applying the method.

\item {\bf Licenses for existing assets}
    \item[] Question: Are the creators or original owners of assets (e.g., code, data, models), used in the paper, properly credited and are the license and terms of use explicitly mentioned and properly respected?
    \item[] Answer: \answerYes{}.
    \item[] Justification: The paper cites the datasets, benchmarks, software libraries, and molecular featurization tools used in the experiments. Appendix~\ref{app:limitations_impact} summarizes the main external assets, including RDKit, PyTorch, MoleculeNet/DeepChem, BigSolDB, SolProp, Leeds, the Bao et al. mixed-solvent dataset/CheMixHub task, and ZINC250k, together with their roles and license or access terms. The anonymized repository will include license files and attribution notes for reused assets.

\item {\bf New assets}
    \item[] Question: Are new assets introduced in the paper well documented and is the documentation provided alongside the assets?
    \item[] Answer: \answerYes{}.
    \item[] Justification: The released research assets comprise the PRISM implementation and training and evaluation scripts. The public repository at \url{https://github.com/cxy0206/PRISM} provides usage instructions, environment requirements, and a code license.

\item {\bf Crowdsourcing and research with human subjects}
    \item[] Question: For crowdsourcing experiments and research with human subjects, does the paper include the full text of instructions given to participants and screenshots, if applicable, as well as details about compensation (if any)? 
    \item[] Answer: \answerNA{}.
    \item[] Justification: The paper does not involve crowdsourcing, user studies, surveys, or research with human subjects.

\item {\bf Institutional review board (IRB) approvals or equivalent for research with human subjects}
    \item[] Question: Does the paper describe potential risks incurred by study participants, whether such risks were disclosed to the subjects, and whether Institutional Review Board (IRB) approvals (or an equivalent approval/review based on the requirements of your country or institution) were obtained?
    \item[] Answer: \answerNA{}.
    \item[] Justification: The paper does not involve crowdsourcing, user studies, surveys, or research with human subjects, so IRB approval is not applicable.

\item {\bf Declaration of LLM usage}
    \item[] Question: Does the paper describe the usage of LLMs if it is an important, original, or non-standard component of the core methods in this research? Note that if the LLM is used only for writing, editing, or formatting purposes and does \emph{not} impact the core methodology, scientific rigor, or originality of the research, declaration is not required.
    \item[] Answer: \answerNA{}.
    \item[] Justification: The core method, experiments, and scientific claims do not use LLMs as an important, original, or non-standard methodological component. Any use of writing, editing, or formatting tools, if applicable, is limited to drafting or polishing and does not affect the method, experiments, results, or originality of the research.

\end{enumerate}

\end{document}